\documentclass[10pt,letterpaper]{paper}

\usepackage[all]{hypcap}

\usepackage[authoryear, round]{natbib}

\usepackage[utf8]{inputenc} 
\usepackage[T1]{fontenc}    
\usepackage{url}            
\usepackage{booktabs}       
\usepackage{amsfonts}       
\usepackage{nicefrac}       
\usepackage{microtype}      
\usepackage{multirow} 
\usepackage[textsize=tiny]{todonotes}
\usepackage{algorithm}
\usepackage{amssymb}
\usepackage{cleveref}
\usepackage{dsfont}
\usepackage{nicefrac}
\usepackage{inconsolata}
\usepackage{xcolor}
\colorlet{revisiongreen}{black}
\DeclareRobustCommand{\oprdaddition}[1]{\textcolor{revisiongreen}{#1}}
\usepackage{amsmath}
\usepackage{amssymb}
\usepackage{booktabs}
\usepackage{multirow}
\usepackage{makecell}
\usepackage{caption}
\usepackage{subcaption}
\usepackage{bbm}
\usepackage{wrapfig}
\usepackage{mathtools,xparse}
\usepackage{pifont}
\usepackage{enumitem}
\usepackage{array}
\usepackage{scalerel,xparse}
\usepackage{colortbl}
\usepackage{tablefootnote}
\usepackage{tocloft}  
\usepackage{siunitx}

\usepackage{hyperref}  
\newcommand{\eqautoref}[1]{\hyperref[#1]{Eq.~\ref*{#1}}}
\newcommand{\eqsautoref}[2]{\hyperref[#1]{Eqs.~\ref*{#1}} and~\hyperref[#2]{\ref*{#2}}}
\newcommand{\secautoref}[1]{\hyperref[#1]{Section~\ref*{#1}}}
\newcommand{\appautoref}[1]{\hyperref[#1]{Appendix~\ref*{#1}}}
\newcommand{\figautoref}[1]{\hyperref[#1]{Figure~\ref*{#1}}}
\newcommand{\figsautoref}[2]{\hyperref[#1]{Figures~\ref*{#1}} and~\hyperref[#2]{\ref*{#2}}}

\usepackage{dblfloatfix} 
\usepackage{adjustbox}
\usepackage{algpseudocode}
\usepackage{setspace}
\usepackage[most,skins,theorems]{tcolorbox}
\usepackage{graphicx}

\newboolean{showsection}
\setboolean{showsection}{true}
\makeatletter
\@namedef{ver@everyshi.sty}{}
\makeatother

\definecolor{custom_green}{rgb}{0.0, 0.5, 0.0}
\definecolor{custom_red}{rgb}{1.0, 0.01, 0.24}
\definecolor{custom_blue}{HTML}{C9DAF7}
\definecolor{custom_purple}{HTML}{D9D1E9}
\definecolor{title_blue}{HTML}{204899} 
\definecolor{cite_blue}{HTML}{044dc1}  
\definecolor{cite_purple}{HTML}{7406a7}  
\definecolor{microsoft_red}{HTML}{ec4e21}
\definecolor{qualitativeTokenInk}{HTML}{17212B}
\definecolor{qualitativeAlignmentRed}{HTML}{CE4942}
\definecolor{qualitativeStudentBlue}{HTML}{A9CFEA}
\definecolor{qualitativeShiftGold}{HTML}{F7D58A}
\newcommand{\opposedtoken}[1]{\begingroup\setlength{\fboxsep}{0.7pt}\setlength{\fboxrule}{1.0pt}\smash{\fcolorbox{black}{qualitativeAlignmentRed}{\textcolor{white}{\texttt{\textbf{#1}}}}}\endgroup}
\DeclareRobustCommand{\captionopposedtoken}[1]{\begingroup\setlength{\fboxsep}{1.0pt}\setlength{\fboxrule}{1.0pt}\smash{\fcolorbox{black}{qualitativeAlignmentRed}{\textcolor{white}{\texttt{\textbf{#1}}}}}\endgroup}
\newcommand{\studenttoken}[1]{\begingroup\setlength{\fboxsep}{1.3pt}\setlength{\fboxrule}{0.4pt}\smash{\fcolorbox{qualitativeTokenInk}{qualitativeStudentBlue}{\textcolor{qualitativeTokenInk}{\texttt{\textbf{#1}}}}}\endgroup}
\newcommand{\shifttoken}[1]{\begingroup\setlength{\fboxsep}{1.3pt}\setlength{\fboxrule}{0.4pt}\smash{\fcolorbox{qualitativeTokenInk}{qualitativeShiftGold}{\textcolor{qualitativeTokenInk}{\texttt{\textbf{#1}}}}}\endgroup}

\hypersetup{
    colorlinks = true,
    citecolor = {cite_blue},
    linkcolor = {cite_purple},
    urlcolor = {cite_blue},
}

\definecolor{blanchedalmond}{rgb}{1.0, 0.92, 0.8}
\definecolor{carmine}{rgb}{0.59, 0.0, 0.09}
\definecolor{lightblue}{rgb}{0.22,0.45,0.70}%

\newtheorem{theorem}{Theorem}[section]

\newtheorem{proposition}[theorem]{Proposition}

\renewcommand{\mathbf}{\boldsymbol}

\makeatletter
\def\Ddots{\mathinner{\mkern1mu\raise\p@
\vbox{\kern7\p@\hbox{.}}\mkern2mu
\raise4\p@\hbox{.}\mkern2mu\raise7\p@\hbox{.}\mkern1mu}}
\makeatother

\numberwithin{equation}{section}

\definecolor{amaranth}{rgb}{0.9, 0.17, 0.31}
\definecolor{antiquebrass}{rgb}{0.8, 0.58, 0.46}
\definecolor{antiquefuchsia}{rgb}{0.57, 0.36, 0.51}
\definecolor{chromeyellow}{rgb}{0.31, 0.47, 0.26}

\newcommand{\1}{\mathds 1}

\usepackage{amsmath,amsfonts,bm}

\def\eqref#1{equation~\ref{#1}}

\def\1{\bm{1}}

\DeclareMathAlphabet{\mathsfit}{\encodingdefault}{\sfdefault}{m}{sl}
\SetMathAlphabet{\mathsfit}{bold}{\encodingdefault}{\sfdefault}{bx}{n}

\tcbset{
    positive/.style={
        colback=custom_blue!5!white, 
        colframe=custom_blue!60!black, 
        coltitle=custom_blue!50!black, 
        colbacktitle=custom_blue!20!white, 
        fonttitle=\bfseries\centering, 
        title=Positive Steering,
        rounded corners,
        boxrule=0.5mm,
        enhanced,
        before skip=5mm, 
        after skip=5mm,  
    },
    neutral/.style={
        colback=gray!5!white, 
        colframe=black!50!white, 
        coltitle=black!50, 
        colbacktitle=black!10!white, 
        fonttitle=\bfseries\centering, 
        title=Neutral Response,
        rounded corners,
        boxrule=0.5mm,
        enhanced,
        before skip=5mm,
        after skip=5mm,
    },
    negative/.style={
        colback=custom_orange!5!white, 
        colframe=custom_orange!60!black, 
        coltitle=custom_orange!50!black, 
        colbacktitle=custom_orange!20!white, 
        fonttitle=\bfseries\centering, 
        title=Negative Steering,
        rounded corners,
        boxrule=0.5mm,
        enhanced,
        before skip=5mm,
        after skip=5mm,
    },
    dataexamplebox/.style={
        colframe=ccbgc_2,
        colback=ccbgc,
        coltitle=black,
        fonttitle=\small\bfseries,
        colbacktitle=ccbgc_2,
        title=Example: Myopic Alice/Bob,
        rounded corners,
        boxrule=0.5mm,
        arc=1mm,
        width=\columnwidth,
        top=1mm,
        bottom=1mm,
        left=1mm,
        right=1mm,
        fontupper=\footnotesize
    }
}

\newcommand\pythonstyle{\lstset{
basicstyle=\ttfamily\footnotesize,
language=Python,
morekeywords={self, clip, exp, mse_loss, uniform_sample, concatenate, logsumexp},              %
keywordstyle=\color{deepblue},
emph={MyClass,__init__},          %
emphstyle=\color{deepred},    %
stringstyle=\color{deepgreen},
frame=single,                         %
showstringspaces=false
}}

\lstnewenvironment{python}[1][]
{
\pythonstyle
\lstset{#1}
}
{}

\newcommand\pythoninline[1]{{\pythonstyle\lstinline!#1!}}

\makeatletter
\def\mathcolor#1#{\@mathcolor{#1}}
\def\@mathcolor#1#2#3{%
  \protect\leavevmode
  \begingroup
    \color#1{#2}#3%
  \endgroup
}
\makeatother

\tcbset{
  aibox/.style={
    width=474.18663pt,
    top=7pt,
    bottom=5pt,
    colback=blue!6!white,
    colframe=black,
    colbacktitle=black,
    enhanced,
    center,
    attach boxed title to top left={yshift=-0.1in,xshift=0.15in},
    boxed title style={boxrule=0pt,colframe=white,},
  }
}
\newtcolorbox{AIbox}[2][]{aibox,title=#2,#1}

\Crefformat{equation}{#2Eq.\;(#1)#3}

\Crefformat{figure}{#2Figure #1#3}
\Crefformat{assumption}{#2Assumption #1#3}
\Crefname{assumption}{Assumption}{Assumptions}

\usepackage{crossreftools}
\makeatletter
\renewcommand\footnoterule{%
  \kern 15\p@
  \hrule \@width 2in \kern 2.6\p@ 
  \vspace{4pt}
}
\makeatother

\definecolor{oprdBoxBg}{RGB}{247,249,252}
\definecolor{oprdBoxFrame}{RGB}{203,213,225}

\newtcolorbox{oprdmethodbox}{
    colback=oprdBoxBg,
    colframe=oprdBoxFrame,
    boxrule=0.45pt,
    arc=2pt,
    boxsep=0pt,
    left=6pt,
    right=6pt,
    top=3pt,
    bottom=8pt,
    before skip=4pt,
    after skip=4pt,
    before upper={
        \setlength{\abovedisplayskip}{4pt}
        \setlength{\belowdisplayskip}{4pt}
        \setlength{\abovedisplayshortskip}{4pt}
        \setlength{\belowdisplayshortskip}{4pt}
    }
}

\newtcolorbox{oprdmethodbox2}{
    colback=oprdBoxBg,
    colframe=oprdBoxFrame,
    boxrule=0.45pt,
    arc=2pt,
    boxsep=0pt,
    left=6pt,
    right=6pt,
    top=5pt,
    bottom=3pt,
    before skip=4pt,
    after skip=4pt,
    before upper={
        \setlength{\abovedisplayskip}{4pt}
        \setlength{\belowdisplayskip}{4pt}
        \setlength{\abovedisplayshortskip}{4pt}
        \setlength{\belowdisplayshortskip}{4pt}
    }
}

\definecolor{ampbg}{RGB}{255,235,210}
\definecolor{ampfg}{RGB}{190,105,35}

\definecolor{unchbg}{RGB}{220,238,255}
\definecolor{unchfg}{RGB}{55,115,175}

\reportnumber{} %

\title{
Eliciting Weak-to-Strong Generalization with On-Policy Reverse Distillation
}

\author[1,*]{Youngrok Park}
\author[1,*,\textdagger]{Sangmin Bae}
\author[1]{Hojung Jung}
\author[2]{Jongwoo Ko}
\author[3]{Yunseon Choi}
\author[2]{\authorcr Young Jin Kim}
\author[2]{Pashmina Cameron}
\author[4,5,6]{Aaron Courville}
\author[1]{Se-Young Yun}

\makeatletter
\renewcommand{\AB@affilsepx}{,\protect\hspace{0.5em}\protect\Affilfont}
\makeatother

\affil[1]{KAIST AI}
\affil[2]{Microsoft}
\affil[3]{University of Toronto}
\affil[4]{Mila}
\affil[5]{Université de Montréal}
\affil[6]{CIFAR AI Chair}

\correspondingauthor{%
Correspondence to:
{
\email{\{yr-park,\,bsmn0223,\,yunseyoung\}@kaist.ac.kr}}.%
}

\usepackage{eso-pic}

\newcommand{\firstpagelogo}{%
  \AddToShipoutPictureFG*{%
    \AtTextUpperLeft{%
      \raisebox{24pt}[0pt][0pt]{%
      \includegraphics[height=6mm]{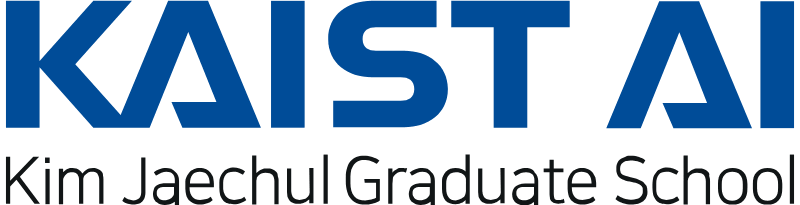}%
        \hspace{5mm}%
        \includegraphics[height=6mm]{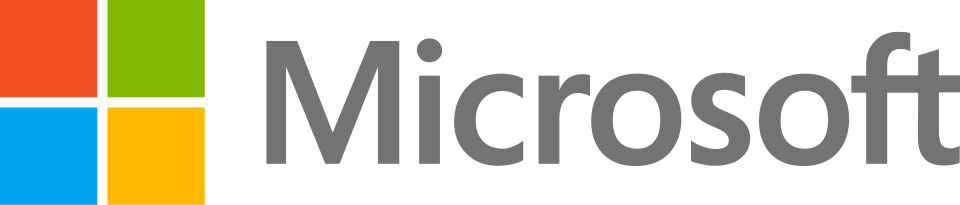}%
        \hspace{5mm}%
        \includegraphics[height=9mm]{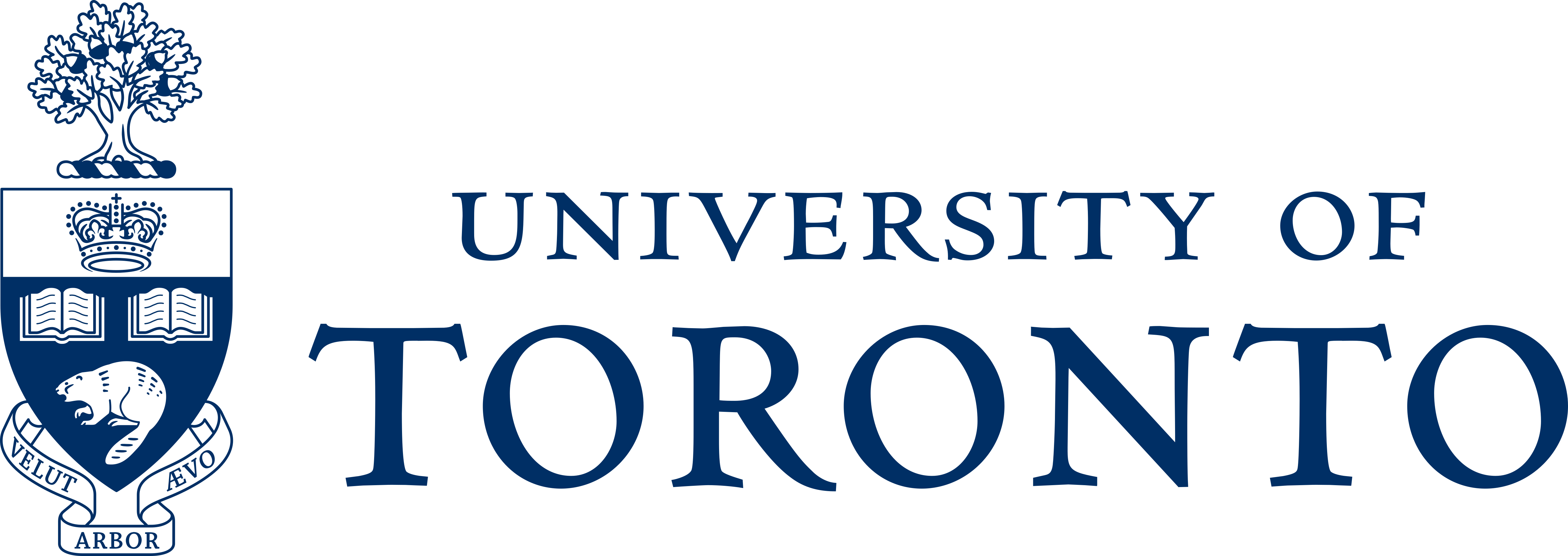}%
        \hspace{4mm}%
        \includegraphics[height=5mm]{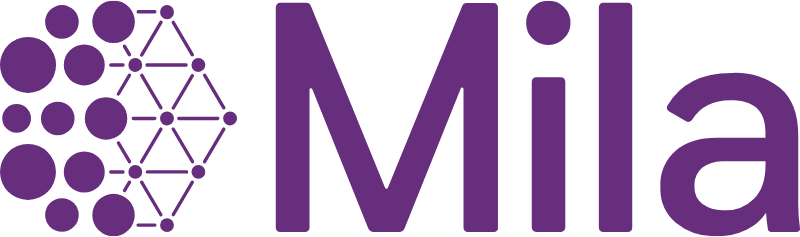}%
      }%
    }%
  }%
}

\begin{abstract}

\textbf{Abstract:}
Weak-to-strong generalization asks whether stronger models can learn from weaker supervisors and surpass them. This is particularly important for successive model generations and multi-domain consolidation, where frontier-scale post-training from scratch can be prohibitively expensive. Yet conventional distillation treats the weak teacher as an optimization target, potentially imposing its capacity ceiling on the student. We introduce On-Policy Reverse Distillation (OPRD), which evaluates the teacher’s policy shift relative to its reference policy on student rollouts and amplifies the component of the student’s verifier-driven policy gradient along that direction. By rescaling only verifier-supported updates, OPRD preserves the stationary points of policy optimization while accelerating learning beyond the teacher. Our evaluations cover mathematical reasoning, logical reasoning, code generation, and visual tool use. In both successive model transfer and multi-teacher distillation, OPRD achieves higher performance with fewer student updates than existing RL and distillation approaches. Response-style analysis shows that OPRD students remain closer to verifier-only RL models than to their weak teachers, suggesting that guidance accelerates rather than redirects student optimization. Strong-to-weak results further show that OPRD combines verifier-driven optimization with teacher guidance regardless of capacity ordering.
\end{abstract}

\begin{document}

\firstpagelogo
\maketitle
  
\section{Introduction}
\label{sec:intro}

\begin{figure*}[!h]
\centering

\begingroup

\begin{minipage}[t]{0.435\linewidth}
\centering

\vspace{0pt}
\makebox[\linewidth][c]{%
  \hspace*{6pt}%
  \includegraphics[
    width=0.95\linewidth
]{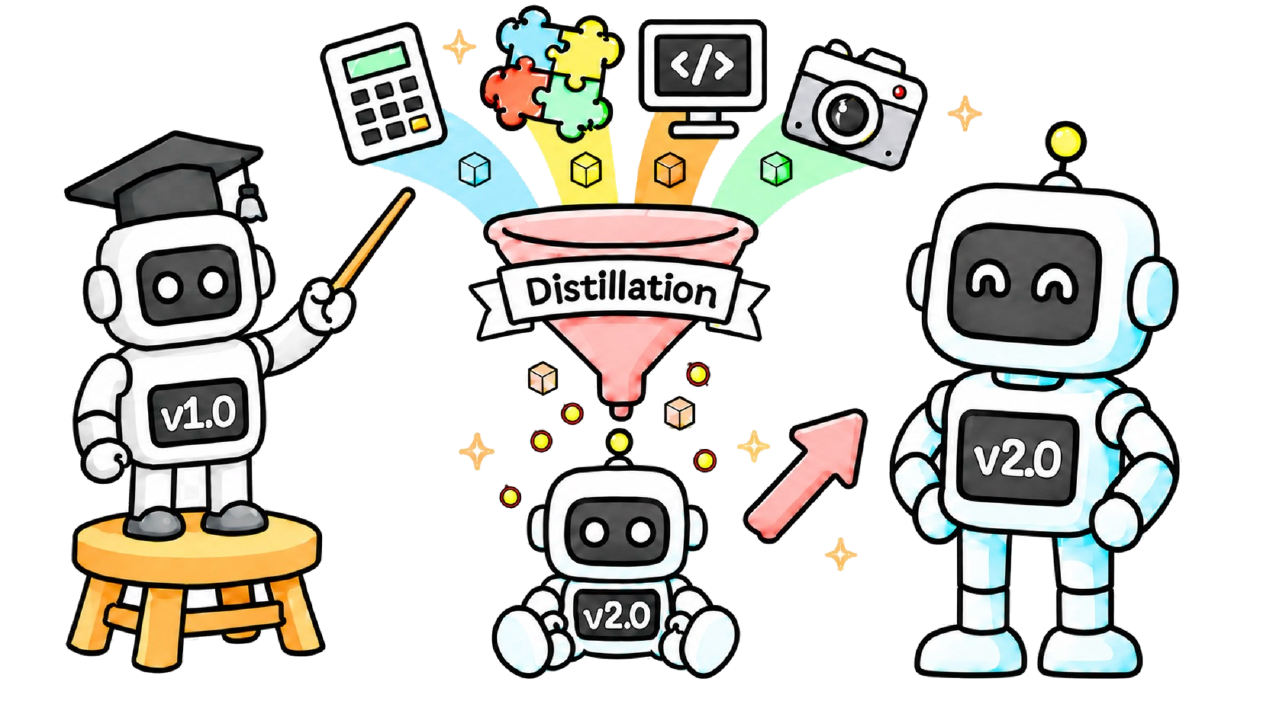}
  \hspace*{-6pt}%
}

\vspace{5pt}
\includegraphics[
  width=0.95\linewidth
]{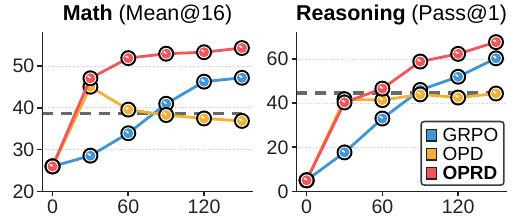}

\captionsetup[subfigure]{
    justification=centering,
    singlelinecheck=false
}
\subcaption{Successive model transfer}
\label{fig:overview_a}

\end{minipage}
\hfill%
\begin{minipage}[t]{0.555\linewidth}
\vspace{-2pt}
\centering

\makebox[\linewidth][c]{%
  \hspace*{6pt}%
  \includegraphics[
    width=0.88\linewidth
  ]{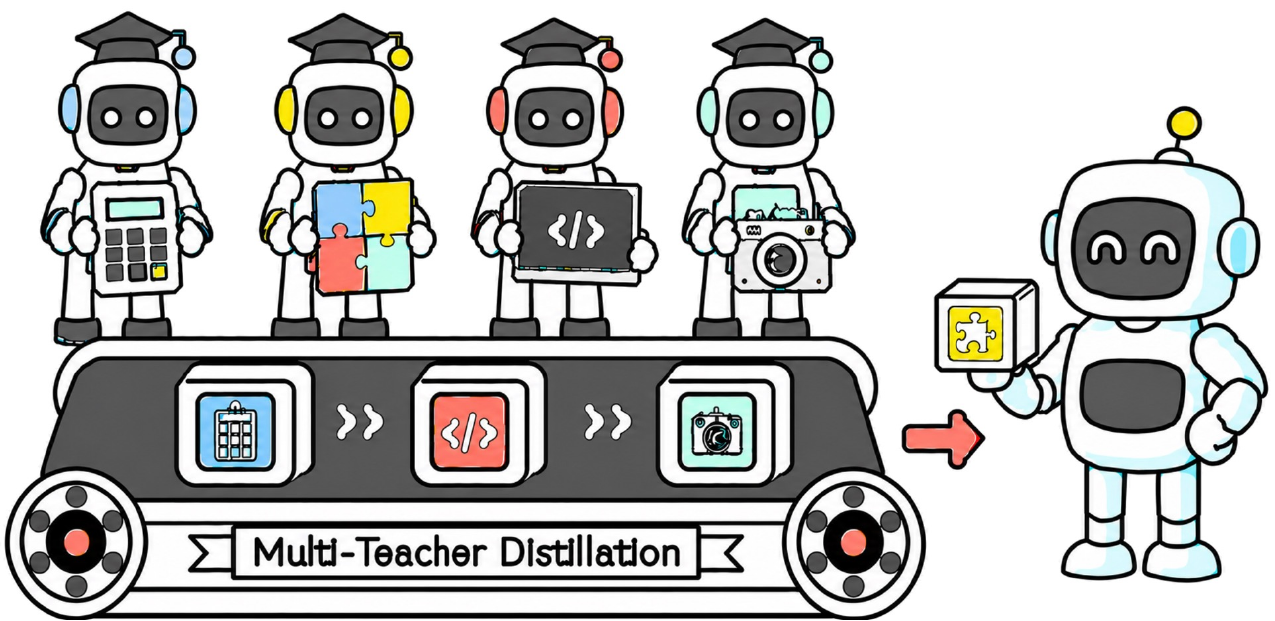}%
  \hspace*{-6pt}%
}

\vspace{1.2pt}

\includegraphics[
  width=1.05\linewidth
]{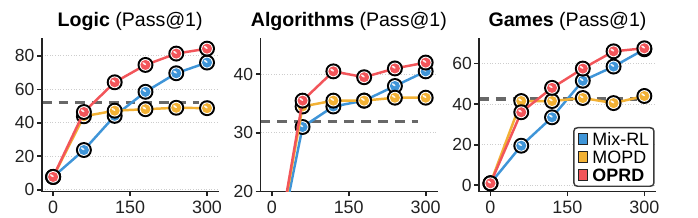}

\captionsetup[subfigure]{
    justification=centering,
    singlelinecheck=false
}
\subcaption{Multi-domain consolidation}
\label{fig:overview_b}
\end{minipage}

\caption{
\textbf{On-policy reverse distillation (OPRD) enables faster and stronger weak-to-strong generalization across two key settings.} \textbf{({a})} For successive model transfer, a checkpoint from a post-trained 4B-scale model serves as the teacher for an 8B-scale student. We average evaluations every 30 training steps: Mean@16 over AIME\textquotesingle24, AIME\textquotesingle25, HMMT\textquotesingle25, and OlympiadBench for math, and Pass@1 over Knights \& Knaves, Quantum Lock, String Manipulation, and Countdown for reasoning tasks. \textbf{({b})} In multi-domain consolidation, four domain-specialized 4B-scale teachers are distilled into one 8B-scale student. Training examples are randomly mixed within each batch, with the corresponding domain teacher activated for each example. We report performance every 60 steps for Logic (averaged over Knights \& Knaves and Quantum Lock), Algorithms (String Manipulation), and Games (Countdown). Gray dashed lines indicate the corresponding weak teachers' performance.
\looseness=-1
}
\label{fig:overview}

\endgroup
\end{figure*}

\clearpage
Knowledge distillation (KD; \citealp{hinton2015distilling}) transfers knowledge from a teacher model to a student. For autoregressive language models, conventional distillation on fixed or teacher-generated sequences can create a mismatch between the prefixes seen during training and those visited by the student at inference time. On-policy distillation (OPD) \citep{gu2024minillm, agarwal2024policy, ko2024distillm} addresses this mismatch by training on student-generated responses and querying the teacher at the prefixes the student visits. Recent work has applied OPD to efficient reasoning post-training \citep{yang2025qwen3, xu2025kdrl, zeng2026glm} and to consolidating capabilities from multiple domain-specific teachers into a single student \citep{xiao2026mimo, yang2026nemotron, xu2026deepseek}. Standard OPD optimizes the student toward the teacher policy, making it well suited when matching that policy is the goal.
\looseness=-1

However, useful supervision need not come from a model that the student should ultimately match. Weak-to-strong generalization has shown that stronger pretrained models can learn from weaker supervisors and even outperform them across language understanding, reward modeling, and reasoning tasks \citep{burns2024weaktostrong, yang2024weak, lang2024theoretical, zhou2025weak}.
This regime is especially promising in two key settings in modern foundation model development.
(i) \emph{Successive model transfer} (\autoref{fig:overview_a}, top):
A post-trained model from one generation can supervise a larger-scale successor, enabling it to inherit prior post-training gains and improve beyond its supervisor.
(ii) \emph{Multi-domain consolidation} (\autoref{fig:overview_b}, top):
Domain-specialized policies can be developed independently at smaller scale, enabling efficient iteration on reward functions, environments, and training recipes. Multi-teacher on-policy distillation (MOPD) \citep{team2026kimi, ma2026mopd, xiao2026mimo} can then consolidate their capabilities into a unified foundation model.
Both settings therefore call for \emph{reverse} distillation that transfers post-training gains from weaker models without limiting the eventual performance of higher-capacity students.
\looseness=-1

Simply applying OPD in the weak-to-strong direction does not resolve this problem. A weak teacher’s final policy combines changes learned during post-training, preferences inherited from its reference policy, and behavior shaped by its limited capacity. Standard OPD matches this entire distribution, transferring all three and retaining the weak policy as the target at each student-visited prefix. Teacher matching can provide useful guidance when the student underperforms the teacher, but can also suppress surprising student behavior when the teacher favors a different solution \citep{akhondzadeh2026reward, ziheng2026less}.
Adding reinforcement learning does not remove this tension if teacher matching remains a separate objective, since the matching loss can compete with reward maximization \citep{xu2025kdrl, zhang2026reinforcement}. Likewise, isolating the teacher’s post-training policy change is insufficient if the student is still trained to match it. This change captures only the improvements realized by the weak teacher, not the full range available to the stronger student, so direct matching can impose the same capacity limitation.
The central question is therefore \textbf{how to leverage weak-model post-training gains to accelerate verifier-driven student optimization without limiting the student to the teacher's capabilities.}
\looseness=-1

We introduce \textbf{On-Policy Reverse Distillation (OPRD)}, which uses the policy change learned during weak-model post-training to accelerate a stronger student’s own optimization. On the student’s on-policy rollouts, OPRD computes the verifier-driven policy gradient and extracts the weak teacher’s policy shift relative to its reference policy. It projects the student gradient onto the direction of this shift and \emph{amplifies the projected component}, leaving the orthogonal component unchanged.
Because this transformation positively rescales only a component already present in the student gradient, it preserves the stationary points of policy optimization in logit space while adding a nonnegative first-order alignment gain. When the teacher shift and student gradient align, OPRD reinforces their shared direction and accelerates convergence; when they oppose, it strengthens surprising student behavior supported by the verifier, allowing the student to improve beyond the teacher.
\looseness=-1

We evaluate OPRD across mathematical reasoning \citep{maa2024to2025aime, dekoninck2026beyond, he2024olympiadbench} and logical reasoning tasks \citep{stojanovski2026reasoning}, as well as code generation \citep{jain2025livecodebench,liu2023your,li2022competition,hendrycks2021measuring} and agentic visual tool use \citep{wu2026reinforcing}, in two main weak-to-strong scenarios. In successive model transfer, OPRD reaches teacher-level performance with 33--67\% fewer updates than GRPO \citep{shao2024deepseekmath} and improves early-checkpoint performance by up to 22.7 percentage points. Unlike OPD, it then moves beyond the teacher rather than saturating after the initial transfer (\autoref{fig:overview_a}, bottom). In the multi-teacher setting, OPRD distills four specialized smaller-scale teachers into a single stronger student, reaching teacher-level performance with 55\% fewer updates than Mix-RL; the resulting student ultimately outperforms all four specialists (\autoref{fig:overview_b}, bottom).
With equal rollouts per update, these gains reflect improved student sample efficiency.
The benefit extends to strong-to-weak distillation, where OPRD moves beyond OPD's plateau through verifier-driven optimization. These results show that weak teachers can accelerate student post-training across generations and domains without constraining students to match their teachers' policies.
\looseness=-1

\paragraph{Contributions.}
In summary, our key contributions in this paper are as follows.
\vspace{-5pt}
\begin{itemize}[leftmargin=*, itemsep=2pt]

    \item \textbf{Weak-to-Strong Generalization.} 
    We study how post-training gains from weaker models can be transferred to stronger students in two practical scenarios: successive model transfer and multi-domain consolidation. We identify the central challenge as exploiting these gains without making either the weak policy or its policy shift a separate optimization target.
    \looseness=-1

    \item \textbf{On-Policy Reverse Distillation.}
    We introduce OPRD, which evaluates a weak teacher's policy shift relative to its reference policy on student rollouts and amplifies the component of the student's verifier-driven policy gradient along that direction. Because OPRD only rescales verifier-supported updates, it accelerates the student's own optimization while preserving its stationary points, allowing the student to move beyond the teacher.
    \looseness=-1
    
    \item \textbf{Empirical Evaluation and Analysis.}
    Across successive-model and multi-teacher settings, OPRD reaches the final performance of competing methods substantially earlier and ultimately outperforms both RL and distillation baselines (\S\ref{subsec:weak_to_strong}, \S\ref{subsec:multi_teacher}). We further confirm that these gains extend to conventional strong-to-weak distillation (\S\ref{subsec:strong_to_weak}). We also compare against recent weak-to-strong methods (\S\ref{subsec:broader_baselines}), analyze the design and dynamics of teacher guidance (\S\ref{subsec:component_analysis}), study student reasoning and response style under teacher guidance (\S\ref{subsec:qualitative_analysis}), and examine practical challenges and mitigations (\S\ref{subsec:challenges_discussion}).
    \looseness=-1
    
\end{itemize}

\section{Method}
\label{sec:method}

\subsection{Preliminary}\label{subsec:preliminary}

\paragraph{Reinforcement Learning with Verifiable Rewards (RLVR).}
RLVR optimizes a language-model policy using rewards computed by programmatic verifiers, such as exact-answer checks or code execution, and has become central to reasoning post-training \citep{shao2024deepseekmath, guo2025deepseek}.
For $x\sim\mathcal D$, the student samples $y\sim\pi_{\theta}(\cdot\mid x)$ and visits prefixes $s_t=(x,y_{<t})$.
Let $A_t$ denote the advantage assigned to token $t$ and $\mathbf z_{t}$ the corresponding next-token logits at prefix $s_t$.
The token-level policy gradient is
\looseness=-1
\begin{equation}
\mathbf g_{t}
:=
A_t
\nabla_{\mathbf z_{t}}
\log\pi_{\theta}(y_t\mid s_t).
\label{eq:rlvr_token_gradient}
\end{equation}
OPRD later rescales this gradient while preserving the RLVR objective, so the student's attainable performance is determined by the verifier objective and its own policy class rather than being bounded by the teacher's capacity.
\looseness=-1

\vspace{-10pt}
\paragraph{On-Policy Distillation (OPD).}
OPD reduces the training--inference distribution mismatch by sampling responses from the student and querying the teacher at each visited prefix, thereby providing dense token-level supervision over the student's inference-time state distribution \citep{gu2024minillm, agarwal2024policy, ko2024distillm}.
A common reverse-KL formulation is
\looseness=-1
\begin{equation}
\mathcal L_{\mathrm{OPD}}(\theta)
:=
\mathop{\mathbb E}\limits_{\substack{
x\sim\mathcal D\\
y\sim\pi_\theta(\cdot\mid x)
}}
\left[
\sum_t
D_{\mathrm{KL}}
\left(
\pi_\theta(\cdot\mid s_t)
\,\middle\|\,
\pi_T(\cdot\mid s_t)
\right)
\right].
\label{eq:opd_objective}
\end{equation}
Equivalently, OPD can be implemented as token-level policy optimization on student-sampled tokens using the teacher-to-student log-probability ratio as the advantage, with negligible empirical differences from direct reverse-KL optimization.
OPD is increasingly used in frontier-model post-training for reasoning and capability integration across domains \citep{ma2026mopd, xiao2026mimo, zeng2026glm, yang2026nemotron}. Recent methods combine teacher matching with reinforcement learning to pair dense teacher supervision with outcome-based optimization \citep{xu2025kdrl,ramos2026recipe}. Even in these hybrid methods, however, teacher matching remains a separate objective, leaving the teacher policy as a direct optimization target.
\looseness=-1

\vspace{-10pt}
\paragraph{Weak-to-Strong Generalization.}
Weak-to-strong generalization studies whether a more capable model can learn from weaker supervisors, such as smaller models or imperfect human feedback, and ultimately outperform them \citep{burns2024weaktostrong}.
Prior work has used weak labels, preferences, and fixed reasoning trajectories to supervise stronger students. Refinement methods help the student exploit its own representations and greater capacity, but often recover only part of the gap to strong supervision \citep{yang2024weak, somerstep2025a, pmlr-v267-dong25g, pmlr-v267-medvedev25a}.
OPD instead provides the full next-token distribution $\bar{\pi}_T(\cdot\mid s_t)$ at each student-visited prefix, where $\bar{\pi}_T$ denotes either the weak teacher or a target policy derived from it.
Under realizability, the resulting KL objective has the pointwise minimizer
\looseness=-1
\begin{equation}
\operatorname*{arg\,min}_{\pi(\cdot\mid s_t)}
D_{\mathrm{KL}}\!\left(
\pi(\cdot\mid s_t)
\,\middle\|\,
\bar{\pi}_T(\cdot\mid s_t)
\right)
=
\bar{\pi}_T(\cdot\mid s_t).
\label{eq:teacher_matching_target}
\end{equation}
Alternative teacher-derived targets only change which policy the student matches, while adding reinforcement learning yields a compromise between teacher matching and reward maximization \citep{xu2025kdrl, ramos2026recipe}. In both cases, the student remains directly optimized toward a policy defined by the weak teacher. OPRD instead extracts the policy change learned during weak-model post-training and uses it only to rescale the stronger student's own policy gradient.
\looseness=-1

\subsection[On-Policy Reverse Distillation]{%
On-Policy Reverse Distillation\protect\footnote{%
Although we refer to our method as distillation, the term carries a subtle paradox. See \appautoref{app:method_discussion} for further discussion.}
}
\label{subsec:method}

\begin{figure*}[t]
    \centering
    \includegraphics[
        width=\textwidth
    ]{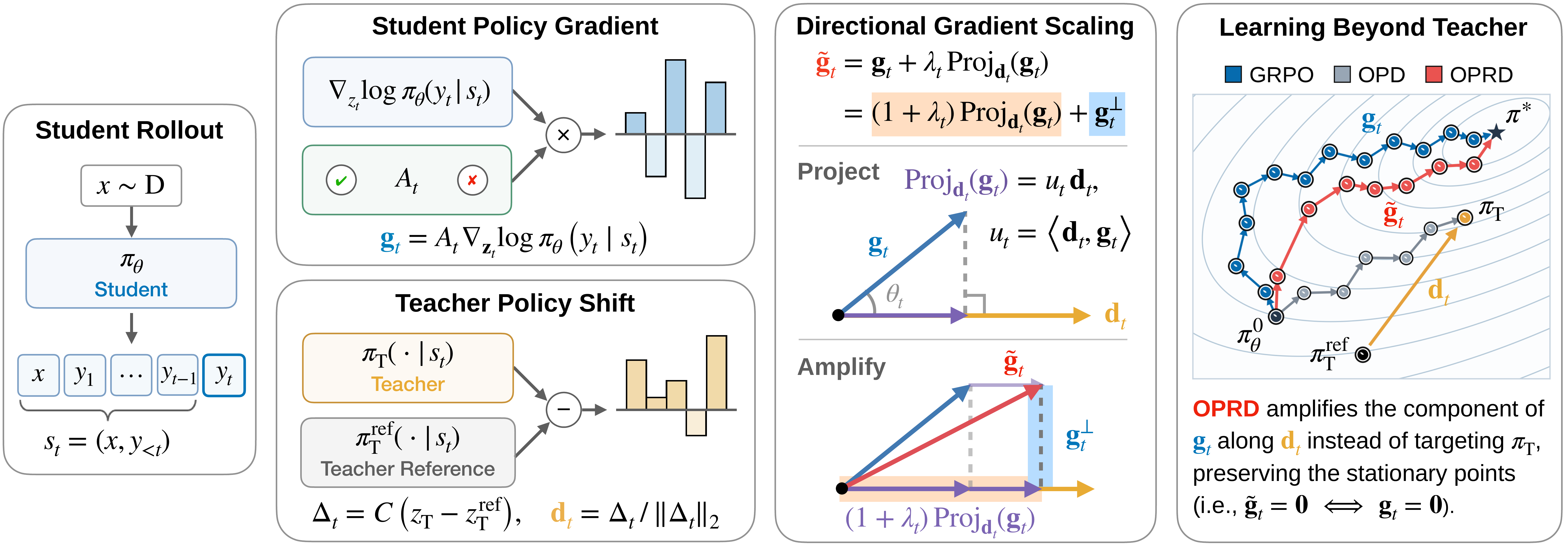}
    \caption{
    \textbf{Conceptual overview of On-Policy Reverse Distillation (OPRD).}
    The figure illustrates OPRD's gradient correction procedure for a single query. Here, $\mathbf z_T$ denotes the logits of the weak teacher after post-training and $\mathbf z_T^{\rm ref}$ those of its reference policy, and $\mathcal{C}$ denotes mean-centering. Their centered difference $\boldsymbol{\Delta}_t$ is the teacher's policy shift at that student-visited prefix, and OPRD keeps only its unit direction $\mathbf d_t$. In practice, we use a simple top-10 truncation under the student policy to focus the correction on its high-probability vocabulary region. The rightmost panel provides a conceptual view of the resulting student trajectory in the optimization landscape, where the student follows the verifier-driven policy gradient $\mathbf g_t$ with its component along $\mathbf d_t$ amplified by $1+\lambda_t$ at each token and its orthogonal component left unchanged.
    \looseness=-1
    }
    \label{fig:oprd_concept}
\end{figure*}

\paragraph{Overview.}
OPRD transfers the policy change learned during teacher post-training rather than matching the teacher's final policy.
At each student-visited prefix, it extracts the local direction of this change relative to the teacher's reference policy and uses its alignment with the verifier-driven student gradient to rescale only the gradient component along that direction.
Because the teacher signal only rescales the student's own gradient, it can accelerate verifier-supported optimization without defining an independent optimization target.
Positive-alignment scaling is active from the outset to amplify updates supported by both the verifier and the teacher, whereas negative-alignment scaling is gradually increased to reinforce verifier-supported departures beyond the weak teacher.
\looseness=-1

\vspace{-10pt}
\paragraph{Teacher Policy Shift.}
The teacher's final policy reflects the change acquired during RL post-training, preferences inherited from its reference policy, and behavior constrained by the weak model's limited capacity. Directly matching it would therefore make all of these part of the student's distillation target.
Let $\pi_T$ denote the frozen RL-trained teacher and $\pi_T^{\rm ref}$ its frozen pre-RL reference policy, and let $\mathbf z_T(s_t)$ and $\mathbf z_T^{\rm ref}(s_t)$ denote their next-token logit vectors at a student-visited prefix $s_t$.
To extract the RL-induced policy delta, we mean-center the difference between the teacher and reference logits, removing a common offset that does not affect relative token preferences.
With $\mathcal C(\mathbf v):=\mathbf v-\frac{1}{|\mathcal V|}(\mathbf 1^\top\mathbf v)\mathbf 1$, we define
\looseness=-1
\begin{oprdmethodbox}
\begin{equation}
\boldsymbol{\Delta}_t:=\mathcal C\!\left(\mathbf z_T(s_t)-\mathbf z_T^{\rm ref}(s_t)\right)=\mathcal C\!\left(\log\pi_T(\cdot\mid s_t)-\log\pi_T^{\rm ref}(\cdot\mid s_t)\right).
\label{eq:teacher_policy_shift}
\end{equation}
\end{oprdmethodbox}
\noindent Intuitively, $\boldsymbol{\Delta}_t$ captures the change in the teacher's relative next-token preferences induced by post-training, and the corresponding uncentered log-policy ratio admits an implicit-reward interpretation under KL-regularized policy optimization.
However, because this shift is learned within the weak teacher's policy class, it need not improve verifier reward for the stronger student.
OPRD therefore uses only its unit direction $\mathbf d_t:=\boldsymbol{\Delta}_t/\lVert\boldsymbol{\Delta}_t\rVert_2$ for gradient scaling rather than optimizing toward the shift itself. The student gradient determines whether the resulting correction follows or opposes this direction, independently of its raw magnitude.
\looseness=-1

\vspace{-10pt}
\paragraph{Gradient Scaling along the Teacher Direction.}
At token $t$, OPRD decomposes the student's policy gradient $\mathbf g_t$ (\eqautoref{eq:rlvr_token_gradient}) relative to the teacher direction $\mathbf d_t$.
Let $u_t:=\mathbf d_t^\top\mathbf g_t$ denote their alignment coefficient, and define the projected and orthogonal components as $\operatorname{Proj}_{\mathbf d_t}(\mathbf g_t):=u_t\mathbf d_t$ and $\mathbf g_t^\perp:=\mathbf g_t-\operatorname{Proj}_{\mathbf d_t}(\mathbf g_t)$, respectively.
At optimization step $k$, OPRD uses a nonnegative scale $\lambda_t$:
\looseness=-1
\begin{oprdmethodbox2}
\begin{equation}
\begin{aligned}
\widetilde{\mathbf g}_t
:&= \mathbf g_t
+\lambda_t\operatorname{Proj}_{\mathbf d_t}(\mathbf g_t) \\
&=
\begin{gathered}[t]
\colorbox{ampbg}{%
  $\displaystyle
  (1+\lambda_t)\operatorname{Proj}_{\mathbf d_t}(\mathbf g_t)
  $%
}\\[-2pt]
\textcolor{ampfg}{\scriptsize\textit{amplified}}
\end{gathered}
\;+\;
\begin{gathered}[t]
\colorbox{unchbg}{%
  $\displaystyle
  \mathbf g_t^\perp
  $%
}\\[2pt]
\textcolor{unchfg}{\scriptsize\textit{unchanged}}
\end{gathered}
\, .
\end{aligned}
\label{eq:oprd_gradient_scaling}
\end{equation}
\end{oprdmethodbox2}
\noindent Essentially, OPRD decomposes the student's policy gradient into its projection onto the teacher informed direction and an orthogonal component, amplifying only the projected component by $1+\lambda_t$, while leaving the orthogonal component unchanged.
We backpropagate $\widetilde{\mathbf g}_t$ in place of $\mathbf g_t$ and use the resulting parameter gradients to update the student.
\looseness=-1

\vspace{-10pt}
\paragraph{Learning Beyond the Weak Teacher.}
Direct teacher matching makes the weak teacher's policy a target of student optimization, even when moving beyond the teacher would yield higher verifier reward.
OPRD instead uses the weak teacher only to rescale the student's own policy gradient.
At token $t$, this scaling can be written as the linear map $\widetilde{\mathbf g}_t=(\mathbf I+\lambda_t\mathbf d_t\mathbf d_t^\top)\mathbf g_t$.
For $\lambda_t\geq0$, the map is invertible and satisfies:
\looseness=-1
\begin{oprdmethodbox}
\begin{flalign}
&\hspace*{4em}
\makebox[13em][r]{\small (Stationarity)}
\hspace{2.5em}
\widetilde{\mathbf g}_t=\mathbf 0
\quad\text{if and only if}\quad
\mathbf g_t=\mathbf 0.
&&
\label{eq:oprd_stationarity}
\\[0.5em]
&\hspace*{4em}
\makebox[13em][r]{\small (Alignment Gain)}
\hspace{2.5em}
\left\langle\mathbf g_t,\widetilde{\mathbf g}_t\right\rangle
=\lVert\mathbf g_t\rVert_2^2+\lambda_t u_t^2
\geq\lVert\mathbf g_t\rVert_2^2.
&&
\label{eq:oprd_alignment_gain}
\end{flalign}
\end{oprdmethodbox}
\noindent Since \eqautoref{eq:oprd_stationarity} holds at every token, the scaling preserves the stationary points of the verifier objective for a fixed response.
\eqautoref{eq:oprd_alignment_gain} shows that the transformed gradient $\widetilde{\mathbf g}_t$ retains the first-order progress of $\mathbf g_t$ and adds the nonnegative alignment gain $\lambda_t u_t^2$, so greater alignment magnitude $|u_t|$ yields greater first-order progress. For a fixed response, these token-level gains sum into a nonnegative term in the guaranteed one-step ascent (see \appautoref{app:oprd-convergence}). OPRD can therefore accelerate the student's optimization without introducing a teacher-defined target.
\looseness=-1

\vspace{-10pt}
\paragraph{Asymmetric Alignment Scaling.}
The sign of $u_t$ indicates whether $\mathbf g_t$ aligns with or opposes $\mathbf d_t$, so scaling reinforces teacher-following updates when $u_t\geq0$ and verifier-supported departures when $u_t<0$. 
However, both signals may include reward-irrelevant bias (e.g., $\mathbf g_t=\mathbf g_t^\star+\boldsymbol{\epsilon}_t$, with $\mathbf g_t^\star$ denoting the reward-improving signal and $\boldsymbol{\epsilon}_t$ aggregating structured bias components). Scaling only one sign can then systematically magnify this bias term, causing it to accumulate over training (see \appautoref{app:alignment_amplification} and \ref{app:isolating_pos_neg_align}). 
Because negative-alignment updates are less reliable on initially weak student rollouts, we activate the positive branch immediately and gradually ramp up the negative branch.
At optimization step $k$, we set the token-wise scaling coefficient as
\looseness=-1
\begin{oprdmethodbox}
\begin{equation}
\lambda_{t}
:=
\begin{cases}
\lambda,
& u_t\geq 0, \\[2pt]
\lambda
\min\!\left\{
\frac{k}{K_{\mathrm{warm}}},
1
\right\},
& u_t<0,
\end{cases}
\label{eq:oprd_asymmetric_scaling}
\end{equation}
\end{oprdmethodbox}
\noindent where $\lambda$ is the scaling strength and $K_{\mathrm{warm}}$ is the warm-up horizon.
Since $\lambda_{t}\geq 0$, \eqsautoref{eq:oprd_stationarity}{eq:oprd_alignment_gain} continue to hold under this schedule. 
The negative-branch warm-up gradually mitigates the initial one-sided amplification caused by positive-only scaling. This preserves immediate teacher-aligned transfer while progressively strengthening verifier-supported departures from the weak teacher. We further analyze this design alongside alternative branch-scaling strategies in \appautoref{app:alignment_schedules}.
\looseness=-1
\section{Experiments}
We evaluate OPRD on mathematical and logical reasoning tasks, code generation, and agentic visual tool use in two primary settings: weak-to-strong transfer across successive model transfer and multi-domain consolidation with multiple specialized teachers. We additionally evaluate conventional strong-to-weak distillation to verify that OPRD does not depend on a particular teacher--student size ordering.
\looseness=-1

\clearpage
\subsection{Experimental Setup}
\label{subsec:experimental_setup}

\paragraph{Tasks.}
For mathematics, we train on DAPO-Math-17K \citep{yu2025dapo} and evaluate on AIME\textquotesingle24, AIME\textquotesingle25 \citep{maa2024to2025aime}, HMMT\textquotesingle25 \citep{dekoninck2026beyond}, and OlympiadBench \citep{he2024olympiadbench}. Reasoning Gym~\citep{stojanovski2026reasoning} comprises Knights \& Knaves (K\&K), Quantum Lock, String Manipulation, and Countdown, each with a fixed split of 19,800 training and 200 evaluation examples. For code generation, we train on DeepCoder Preview \citep{deepcoder2025} and evaluate on LCBv5, HumanEval+, CodeContests, and APPS Competition \citep{jain2025livecodebench,liu2023your,li2022competition,hendrycks2021measuring}. For agentic visual tool use, we use ViLaSR RL \citep{wu2026reinforcing}, spanning Maze, GPT4Scene, SR-91k, and Spatial VQA. We balance 19,200 training questions across the four families and evaluate on 10,464 held-out questions, excluding RL training questions and overlapping images or scenes. Detailed splits appear in \appautoref{app:exp_details}.
\looseness=-1

\vspace{-10pt}
\paragraph{Models.}
Math, Reasoning Gym, and Code use Qwen3-family models \citep{yang2025qwen3}, while Agent uses Qwen2.5-VL models \citep{Qwen2.5-VL} initialized by cold-start SFT. Model pairs are specified below. Each teacher is post-trained with GRPO \citep{shao2024deepseekmath} on the corresponding training data and held fixed during student training.
\looseness=-1

\vspace{-10pt}
\paragraph{Baselines.}
For the single-teacher experiments, we compare OPRD with GRPO \citep{shao2024deepseekmath}, OPD \citep{agarwal2024policy}, and KDRL \citep{xu2025kdrl}. GRPO performs verifier-only policy optimization, OPD matches the frozen teacher on student-generated prefixes, and KDRL serves as a representative hybrid baseline that combines verifier-based policy optimization with on-policy distillation. For the multi-teacher setting, we similarly compare OPRD with Mix-RL, MOPD \citep{ma2026mopd}, and KDRL, which serve as the corresponding verifier-only, distillation-only, and hybrid baselines, respectively.
\looseness=-1

\vspace{-10pt}
\paragraph{Training and Evaluation.}
For each experiment, OPRD and all baselines start from the same student checkpoint and use the same task-specific training prompts, batch size, rollout budget, and number of policy updates. Teacher-based methods also use the same frozen teacher checkpoint for each task.
We evaluate mathematics with Mean@16 and Reasoning Gym with Pass@1. Code uses Mean@4, while Agent uses a single-trajectory mean rule score with numeric MRA partial credit. While teacher and initial-student entries report fixed-checkpoint performance, trained-policy entries in most tables are averaged over five checkpoints to capture both learning speed and performance throughout training: at 30-update intervals for single-teacher Math, Reasoning Gym, and Code, 45-update intervals for Agent, and at 60-update intervals for multi-teacher distillation.
Full training and evaluation details are provided in \autoref{app:exp_details}.
\looseness=-1

\subsection{Weak-to-Strong Distillation for a Successor Model}
\label{subsec:weak_to_strong}

\paragraph{Settings.}
To study successive model transfer in a controlled setting, we perform weak-to-strong distillation across scales within the same model family.
Under the default configurations in \secautoref{subsec:experimental_setup}, we pair a post-trained Qwen3-4B teacher with a Qwen3-8B student for math, and Qwen3-4B-Base teachers with separately trained Qwen3-8B-Base students for the four Reasoning Gym tasks.
Code pairs a GRPO-trained Qwen3-8B-Base teacher with a Qwen3-14B-Base student. Agent pairs a Qwen2.5-VL-3B teacher trained with cold-start SFT and GRPO with a Qwen2.5-VL-7B public ViLaSR cold-start SFT student, using the teacher\textquotesingle s own SFT checkpoint as the OPRD reference model. Additional Qwen3-Base results for mathematics and three other reasoning tasks are provided in \appautoref{app:additional_successive_model}.
\looseness=-1

\begin{table*}[!t]
\caption{
\textbf{Main experimental results for successive model transfer on mathematics, logical reasoning, code generation, and agentic visual tool use.} We report Mean@16 for mathematics, Pass@1 for Reasoning Gym, Mean@4 for Code, and the single-trajectory mean rule score for Agent. Teacher and initial-student rows use fixed checkpoints. Trained rows average steps 30, 60, 90, 120, and 150 for Math, Reasoning Gym, and Code, and steps 45, 90, 135, 180, and 225 for Agent. Scores are multiplied by 100, and Avg.\ is the unweighted mean of four benchmark or family scores. Agent numeric answers retain MRA partial credit. HE+, CC, G4S, and SVQA denote HumanEval+, CodeContests, GPT4Scene, and Spatial VQA, respectively. Bold marks the best trained-policy result per column. Learning curves appear in \appautoref{app:detailed_curve_successive_model}, while Math and Reasoning Gym response-resampling standard deviations appear in \appautoref{app:standard_deviation}.
\looseness=-1
}
\label{tab:main_results}
\centering

\begingroup
\small
\setlength{\tabcolsep}{2.0pt}
\renewcommand{\arraystretch}{1.05}

\begin{tabularx}{\linewidth}{
@{}
p{0.10\linewidth}
*{5}{>{\centering\arraybackslash}X}
@{}p{6pt}@{}
*{5}{>{\centering\arraybackslash}X}
@{}
}
\toprule


&
\multicolumn{5}{c}{\textbf{Math Reasoning}}
& {}
&
\multicolumn{5}{c}{\textbf{Reasoning Gym}} \\
\cmidrule(lr){2-6}
\cmidrule(lr){8-12}

\textbf{Policy}
& AIME\textquotesingle24
& AIME\textquotesingle25
& HMMT\textquotesingle25
& Olympiad
& \textbf{Avg.}
& {}
& Knights
& Quantum
& String
& Count
& \textbf{Avg.} \\
\midrule


\rowcolor{blue!8}
&
\multicolumn{5}{c}{
\footnotesize
\textbf{Q3-4B (Teacher)}
$\rightarrow$
\textbf{Q3-8B (Student)}
}
& {}
&
\multicolumn{5}{c}{
\footnotesize
\textbf{Q3-4B-Base (Teacher)}
$\rightarrow$
\textbf{Q3-8B-Base (Student)}
} \\
\midrule


Teacher
& 42.50
& 38.75
& 21.25
& 52.15
& 38.66
& {}
& 57.50
& 46.58
& 32.00
& 42.50
& 44.65 \\

\midrule

Student
& 25.63
& 19.58
& 12.50
& 46.22
& 25.98
& {}
& 11.00
& \,\,\,3.14
& \,\,\,3.00
& \,\,\,3.00
& \,\,\,5.04 \\


+ GRPO
& 46.63
& 36.42
& 21.96
& 52.52
& 39.38
& {}
& 54.20
& 36.24
& 35.50
& 41.30
& 41.81 \\

+ OPD
& 46.92
& 38.21
& 21.42
& 51.24
& 39.44
& {}
& 55.00
& 39.62
& 35.10
& 41.60
& 42.83 \\

+ KDRL\footnotemark
& 53.46
& 43.13
& 26.08
& \textbf{53.28}
& 43.99
& {}
& 53.30
& 36.63
& 38.10
& 49.50
& 44.38 \\

\rowcolor{gray!15}
+ \textbf{OPRD}
& \textbf{66.92}
& \textbf{56.04}
& \textbf{31.42}
& 53.26
& \textbf{51.91}
& {}
& \textbf{73.30}
& \textbf{51.11}
& \textbf{42.20}
& \textbf{54.10}
& \textbf{55.18} \\

\bottomrule
\addlinespace[4pt]

& \multicolumn{5}{c}{\textbf{Code Generation}} & {} & \multicolumn{5}{c}{\textbf{Agentic Visual Tool Use}} \\
\cmidrule(lr){2-6}\cmidrule(lr){8-12}
\textbf{Policy} & LCBv5 & HE+ & CC & APPS & \textbf{Avg.} & {} & Maze & G4S & SR-91k & SVQA & \textbf{Avg.} \\
\midrule
\rowcolor{blue!8}
& \multicolumn{5}{c}{\footnotesize\textbf{Q3-8B-Base (Teacher)} $\rightarrow$ \textbf{Q3-14B-Base (Student)}}
& {} & \multicolumn{5}{c}{\footnotesize\textbf{Q2.5-VL-3B (Teacher)} $\rightarrow$ \textbf{Q2.5-VL-7B (Student)}} \\
\midrule
Teacher & 28.67 & 83.74 & 18.16 & 22.01 & 38.15 & {} & 85.94 & 64.67 & 50.38 & 69.99 & 67.74 \\
\midrule
Student & 22.04 & 69.02 & 8.73 & 11.04 & 27.71 & {} & 49.45 & 63.66 & 38.86 & 67.33 & 54.83 \\
+ GRPO & 29.03 & 83.44 & 12.64 & 17.93 & 35.76 & {} & 77.27 & 64.91 & 46.51 & 68.73 & 64.35 \\
+ OPD & 29.84 & 86.29 & 16.79 & 20.92 & 38.46 & {} & 81.48 & 64.74 & 47.02 & 69.04 & 65.57 \\
+ KDRL\textsuperscript{2} & 29.96 & 83.53 & 13.68 & 18.17 & 36.33 & {} & 77.86 & 65.65 & 46.74 & 68.43 & 64.67 \\
\rowcolor{gray!15}
+ \textbf{OPRD} & \textbf{31.61} & \textbf{86.63} & \textbf{18.11} & \textbf{24.16} & \textbf{40.13} & {} & \textbf{82.58} & \textbf{67.48} & \textbf{50.69} & \textbf{69.27} & \textbf{67.50} \\
\bottomrule
\end{tabularx}

\endgroup
\vspace{7pt}
\end{table*}

\footnotetext{
We adopt KDRL \citep{xu2025kdrl} as the representative baseline combining distillation with RLVR (see also GKD \citep{agarwal2024policy} and dGRPO \citep{ramos2026recipe}).
For a fair comparison, the coefficient on the OPD objective is annealed to 0.0.
\looseness=-1
}

\vspace{-10pt}
\paragraph{OPRD Accelerates Learning While Continuing Beyond the Weak Teacher.}
\autoref{fig:overview_a} (bottom) shows that OPD improves rapidly but plateaus near the teacher average, whereas GRPO progresses more gradually. OPRD matches OPD's initial acceleration, quickly surpasses the weak teacher, and reaches GRPO's end-of-training performance substantially earlier. Across five checkpoints, \autoref{tab:main_results} shows gains of 7.92 points on mathematics and 10.80 points on Reasoning Gym over the strongest baseline. Their later divergence suggests that direct policy matching becomes restrictive once the student discovers reward-supported improvements beyond the teacher.
These gains also extend to code generation with execution-based verification and agentic visual tool use through multi-turn interactions, where OPRD outperforms the strongest baseline in both domains. Its final Agent checkpoint also surpasses the teacher on all four task families.
\looseness=-1

\subsection{Multi-Teacher Weak-to-Strong Distillation}
\label{subsec:multi_teacher}

\paragraph{Settings.}
Multi-teacher distillation asks whether capabilities acquired by separately post-trained task specialists can be consolidated into a single policy. We use the same Reasoning Gym configuration and method-specific settings as in \secautoref{subsec:experimental_setup}, but each training batch now mixes the four tasks equally. Teacher-based methods pair each example with its corresponding Qwen3-4B-Base specialist. Because this reduces exposure to each task by roughly a factor of four, we train for 300 policy updates. Despite the longer run, we retain the single-teacher coefficient schedules rather than extending them to 300 updates.
\looseness=-1


\begin{table*}[!t]
\caption{
\textbf{(Left) Experimental results for multi-teacher distillation on Reasoning Gym.}
We consolidate four task-specific Qwen3-4B-Base teachers into a single Qwen3-8B-Base student. The teacher and initial-student rows report fixed-checkpoint performance, while each trained-policy row averages Pass@1 over checkpoints at steps 60, 120, 180, 240, and 300. Detailed learning curves for each task are provided in \appautoref{app:multiteacher_distillation}.
\textbf{(Right) Experimental results for strong-to-weak distillation.}
We evaluate Qwen3-8B $\rightarrow$ Qwen3-1.7B on AIME\textquotesingle24 and Qwen3-8B-Base $\rightarrow$ Qwen3-0.6B on Knights \& Knaves. Each trained-policy row averages Mean@16 and Pass@1, respectively, over five checkpoints. Detailed learning curves are provided in \appautoref{app:strong_to_weak_distillation}. The best result in each column is shown in \textbf{bold}.
\looseness=-1
}
\label{tab:additional_results}
\centering


\begin{minipage}[t]{0.55\linewidth}
\vspace{0pt}
\centering

\begingroup
\small
\setlength{\tabcolsep}{1.8pt}
\renewcommand{\arraystretch}{1.08}

\begin{tabularx}{\linewidth}{
@{}
l
*{5}{>{\centering\arraybackslash}X}
@{}
}
\toprule

\textbf{Policy}
& \textbf{Knights}
& \textbf{Quantum}
& \textbf{String}
& \textbf{Count}
& \textbf{Avg.} \\
\midrule

\rowcolor{blue!8}
\multicolumn{6}{c}{
\footnotesize
\textbf{4 Teachers (Q3-4B-Base)}
$\rightarrow$
\textbf{1 Student (Q3-8B-Base)}
} \\
\midrule

Teachers
& 57.50
& 46.58
& 32.00
& 42.50
& 44.65 \\

\midrule

Student
& 11.00
& \,\,\,3.14
& \,\,\,3.00
& \,\,\,3.00
& \,\,\,5.04 \\

+ Mix-RL
& 64.90
& 43.91
& 35.90
& 46.00
& 47.68 \\

+ MOPD
& 57.10
& 37.70
& 35.50
& 42.10
& 43.10 \\

+ KDRL
& 65.90
& 44.09
& 35.00
& 44.90
& 47.47 \\

\rowcolor{gray!15}
+ \textbf{OPRD}
& \textbf{80.70}
& \textbf{59.68}
& \textbf{39.70}
& \textbf{55.00}
& \textbf{58.77} \\

\bottomrule
\end{tabularx}

\endgroup
\end{minipage}\hfill%
\begin{minipage}[t]{0.43\linewidth}
\vspace{0pt}
\centering

\begingroup
\small
\setlength{\tabcolsep}{1.8pt}
\renewcommand{\arraystretch}{1.08}

\begin{tabularx}{\linewidth}{
@{}
l
>{\hsize=0.90\hsize\linewidth=\hsize
  \centering\arraybackslash}X
>{\hsize=1.35\hsize\linewidth=\hsize
  \centering\arraybackslash}X
>{\hsize=0.75\hsize\linewidth=\hsize
  \centering\arraybackslash}X
@{}
}
\toprule

\textbf{Policy}
& \textbf{AIME\textquotesingle24}
& \mbox{\textbf{Knights Knaves}}
& \textbf{Avg.} \\
\midrule

\rowcolor{blue!8}
&
{\footnotesize\bfseries 8B $\rightarrow$ 1.7B}
&
{\footnotesize\bfseries 8B-Base $\rightarrow$ 0.6B}
& {} \\
\midrule

Teacher
& 52.29
& 64.50
& 58.40 \\

\midrule

Student
& 10.00
& \,\,\,5.00
& \,\,\,7.50 \\

+ GRPO
& 18.92
& 20.60
& 19.76 \\

+ OPD
& 29.79
& 20.20
& 25.00 \\

+ KDRL
& 25.33
& 33.70
& 29.52 \\

\rowcolor{gray!15}
+ \textbf{OPRD}
& \textbf{33.58}
& \textbf{49.40}
& \textbf{41.49} \\

\bottomrule
\end{tabularx}

\endgroup
\end{minipage}

\vspace{3pt}
\end{table*}

\vspace{-10pt}
\paragraph{OPRD Consolidates Heterogeneous Specialists without Cross-Task Tradeoffs.}
\autoref{fig:overview_b} (bottom) shows that MOPD rapidly approaches the specialist average but then plateaus, whereas Mix-RL improves more gradually. OPRD combines this early transfer with continued improvement throughout training. \autoref{tab:additional_results} reports an average of 58.77, exceeding Mix-RL by 11.09 points and the specialist average by 14.12 points. OPRD also surpasses the corresponding specialist on all four tasks despite their distinct structures and objectives, indicating joint improvement rather than a cross-task tradeoff. OPRD may reduce cross-task interference by amplifying only the component of each task's verifier-driven student gradient along its teacher-shift direction, rather than matching the full specialist policy. This projection resembles PCGrad \citep{yu2020gradient}, but is applied between each task's student gradient and teacher direction rather than between conflicting task gradients.
\looseness=-1

\subsection{Strong-to-Weak Distillation}
\label{subsec:strong_to_weak}

\paragraph{Settings.}
To test whether OPRD remains effective when the teacher is larger than the student, we evaluate strong-to-weak distillation on two benchmarks.
We evaluate conventional strong-to-weak distillation under the default configurations in \secautoref{subsec:experimental_setup}. For mathematics, we pair a Qwen3-8B teacher with a Qwen3-1.7B student, and we pair a Qwen3-8B-Base teacher with a Qwen3-0.6B student for Knights \& Knaves. Both teachers are taken from step 105 of task-specific GRPO training. Both settings use prompt batch and mini-batch sizes of 64, and we schedule each method's distillation coefficient over 45 updates.
\looseness=-1

\vspace{-10pt}
\paragraph{OPRD Does Not Depend on Teacher--Student Capacity Ordering.}
\autoref{tab:additional_results} shows that OPRD remains effective in conventional strong-to-weak distillation, outperforming OPD by 3.79 points on AIME\textquotesingle24 and 29.20 points on Knights \& Knaves. Recent studies show that standard OPD can fail when capacity or distributional gaps make teacher supervision difficult to exploit, so a stronger teacher need not yield a better student \citep{li2026rethinking,fu2026revisiting}. Consistent with these findings, OPD improves initially in both settings but quickly saturates well below OPRD. KDRL's lower score further suggests that supplementing policy matching with verifier feedback does not fully resolve this issue. OPRD instead amplifies only the component of the student's verifier-driven gradient aligned with the teacher's policy delta. This allows the student to benefit from teacher guidance along reward-supported directions it can realize, without having to reproduce the stronger policy in full.
\looseness=-1

\section{Analysis}
\label{sec:analysis}


\begin{table*}[!t]
\caption{
\textbf{(Left) Comparison with various weak-to-strong baselines.} 
We evaluate 4B-to-8B transfer using instruction-tuned models for math and base models for reasoning tasks. Teacher and initial-student rows show fixed-checkpoint results, while trained-policy rows show checkpoint averages. The best result in each column is shown in \textbf{bold}. See \appautoref{subsec:baseline_implementation} and \appautoref{subsec:baseline_learning_curve} for implementation details and learning curves, respectively.
\textbf{(Right) Ablation Study on Teacher-Checkpoint Quality.}
For each task, the bars report the performance of Qwen3-4B-Base teacher checkpoints obtained at GRPO training steps 15, 60, 105, and 150, while the lines show the learning curves of Qwen3-8B-Base students trained with OPRD using the corresponding checkpoints. All configurations other than the teacher checkpoint follow the defaults in \secautoref{subsec:experimental_setup}.
\looseness=-1
}
\label{tab:baseline_and_teacher_analysis}
\centering


\begin{minipage}[t]{0.475\linewidth}
\vspace{0pt}
\centering

\begingroup
\small
\setlength{\tabcolsep}{2.0pt}
\renewcommand{\arraystretch}{1.15}

\begin{tabularx}{\linewidth}{
@{}
l
*{4}{>{\centering\arraybackslash}X}
@{}
}
\toprule

\textbf{Policy}
& \textbf{AIME\textquotesingle24}
& \textbf{Knights}
& \textbf{String}
& \textbf{Avg.} \\
\midrule

\rowcolor{blue!8}
&
\multicolumn{3}{c}{
\textbf{4B (-Base) $\rightarrow$ 8B (-Base)}
}
& {} \\
\midrule

Teacher
& 42.50
& 57.50
& 32.00
& 44.00 \\

\midrule

Student
& 25.63
& 11.00
& \,\,\,3.00
& 13.21 \\

+ GRPO
& 46.63
& 54.20
& 35.50
& 45.44 \\

+ OPD
& 46.92
& 55.00
& 35.10
& 45.67 \\

\midrule

+ W2SR-P
& 45.42
& 62.50
& 38.00
& 48.64 \\

+ S2L-PO\footnotemark
& 60.54
& 63.40
& 38.70
& 54.21 \\

+ OPSD\footnotemark
& 15.88
& 24.60
& 24.60
& 21.69 \\

\midrule

+ Direct-OPD
& 35.04
& 49.20
& 29.20
& 37.81 \\

+ W2S-OPD
& 52.96
& 59.50
& 35.20
& 49.22 \\

\midrule

\rowcolor{gray!15}[\dimexpr\tabcolsep+1pt\relax][\dimexpr\tabcolsep+1pt\relax]
+ \textbf{OPRD}
& \textbf{66.92}
& \textbf{73.30}
& \textbf{42.20}
& \textbf{60.81} \\

\bottomrule
\end{tabularx}

\endgroup
\end{minipage}\hfill%
\begin{minipage}[t]{0.511\linewidth}
\vspace{0pt}
\centering

\begingroup
\small

\newcommand{\dummycomposite}{%
\fbox{%
\begin{minipage}[c][2.15cm][c]{0.94\linewidth}
\centering

\begin{minipage}[c][1.70cm][c]{0.29\linewidth}
\centering
\textcolor{gray!70}{\textbf{Teacher bars}}\\[4pt]
\footnotesize
Steps\\
15,\;60,\;105,\;150
\end{minipage}\hfill%
\begin{minipage}[c][1.70cm][c]{0.66\linewidth}
\centering
\textcolor{gray!70}{\textbf{Student curves}}\\[4pt]
\footnotesize
Student GRPO\\
OPRD$_{15}$,\;
OPRD$_{60}$,\;
OPRD$_{105}$,\;
OPRD$_{150}$
\end{minipage}

\end{minipage}%
}%
}


\vspace{-3pt}
\includegraphics[
    width=\linewidth,
    keepaspectratio
]{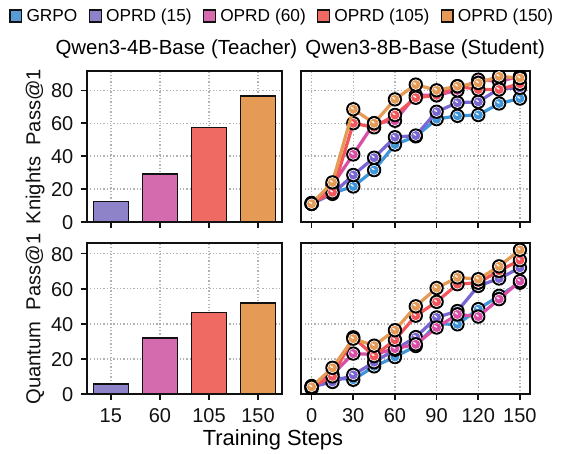}

\endgroup
\end{minipage}

\vspace{4pt}
\end{table*}

\footnotetext[3]{
S2L-PO \citep{ren2026smaller} originally uses a smaller \textit{base} model as the frozen explorer, reflecting the method's central motivation to exploit policy-level diversity. Here, we instead use the post-RL weak-teacher checkpoint as the explorer.
\looseness=-1
}

\footnotetext[4]{
OPSD \citep{zhao2026selfdistilled} and SDPO \citep{hubotter2026reinforcement} use correct self-generated rollouts as privileged information. Here, we instead use correct trajectories generated by the weak teacher, while the KL divergence remains computed against the self-teacher.
\looseness=-1
}

\subsection{Broader Comparison with Weak-to-Strong Methods}
\label{subsec:broader_baselines}

\paragraph{OPRD Outperforms Methods Using Off-Policy Generations from Weak Teacher.}
\autoref{tab:baseline_and_teacher_analysis} (left) compares OPRD with three baselines that use weak-teacher generations differently. W2SR-P \citep{yuan-etal-2026-incentivizing} performs SFT on verified-correct teacher trajectories; S2L-PO \citep{ren2026smaller} mixes off-policy rollouts from a weak explorer with student rollouts in shared GRPO groups before transitioning to fully on-policy RLVR; and our OPSD variant \citep{zhao2026selfdistilled} uses a verified weak-teacher draft as privileged context for self-distillation. W2SR-P and S2L-PO improve over the initial student, indicating that weak-teacher trajectories can provide a useful bootstrap within the same Qwen3 family. However, reliance on off-policy teacher trajectories can create train--inference mismatch \citep{agarwal2024policy} and need not transfer underlying capabilities across model gaps \citep{gudibande2023false}. OPRD instead remains fully on-policy and uses the teacher shift only to rescale the aligned component of the verifier gradient. Empirically, OPRD reaches 60.81, exceeding the strongest alternative, S2L-PO, by 6.60 points and achieving the best score on all three tasks.
\looseness=-1

\vspace{-10pt}
\paragraph{Rescaling the Verifier Gradient Outperforms Direct Optimization of the Weak Policy Delta.}
The lower rows of \autoref{tab:baseline_and_teacher_analysis} (left) compare OPRD with two closely related concurrent works, Direct-OPD \citep{feng2026weak} and W2S-OPD \citep{yu2026weak}, both of which derive the student's objective directly from the weak policy shift. Direct-OPD uses the corresponding log-ratio as a dense reward, whereas W2S-OPD reanchors the shift to the student's base policy and distills the resulting proxy teacher. Given the sensitivity of both methods to the relative strength of the transferred shift, we follow the hyperparameter settings reported in the original papers. However, both methods rely solely on the information encoded in the shift. OPRD instead retains verifier supervision through the orthogonal component $\mathbf g_t^\perp$, allowing the student to pursue reward-supported directions not captured by the weak policy delta. Indeed, OPRD reaches 60.81, outperforming W2S-OPD by 11.59 points and Direct-OPD by 23.00 points on average.
\looseness=-1

\subsection{Design and Dynamics of Teacher Guidance}
\label{subsec:component_analysis}

\paragraph{Better-Trained Weak Teachers Provide More Effective Guidance.}
\autoref{tab:baseline_and_teacher_analysis} (right) shows that later, better-performing GRPO checkpoints of the 4B teacher generally lead to faster learning under OPRD for the 8B student on both reasoning tasks. Because $\boldsymbol{\Delta}_t$ is normalized before scaling, this benefit cannot be attributed to shift magnitude alone; instead, later checkpoints appear to encode a more reward-informative direction, yielding a larger verifier-gradient component for OPRD to amplify. Notably, the step-60 teacher achieves only 29.0\% Pass@1 on Knights \& Knaves, yet the corresponding OPRD student rapidly reaches approximately 88\%, far surpassing both the teacher and GRPO. Thus, while teacher quality affects the strength of OPRD's acceleration, the teacher's absolute performance need not impose a ceiling on the student.
\looseness=-1


\begin{figure*}[!t]
\centering

\begingroup

\captionsetup[subfigure]{
  font=small,
  labelfont=normalfont,
  justification=centering,
  singlelinecheck=true,
  skip=5pt
}


\begin{subfigure}[t]{0.32\linewidth}
\centering

\includegraphics[
    width=\linewidth,
    keepaspectratio
]{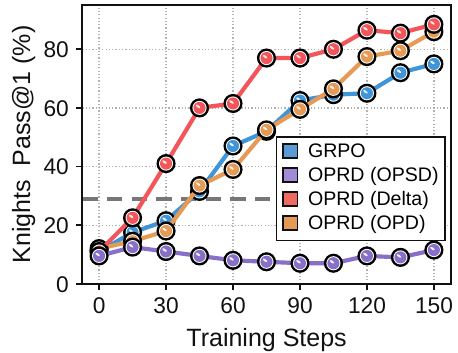}

\caption{Construction of $\mathbf{d}_t$}
\label{fig:analysis_delta_signal}
\end{subfigure}
\hfill
\begin{subfigure}[t]{0.32\linewidth}
\centering

\includegraphics[
    width=\linewidth,
    keepaspectratio
]{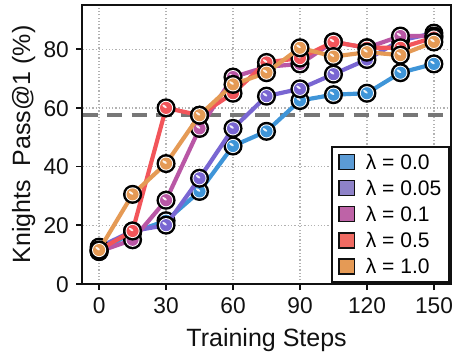}

\caption{Amplification strength $\lambda$}
\label{fig:analysis_lambda_scale}
\end{subfigure}
\hfill
\begin{subfigure}[t]{0.32\linewidth}
\centering

\includegraphics[
    width=\linewidth,
    keepaspectratio
]{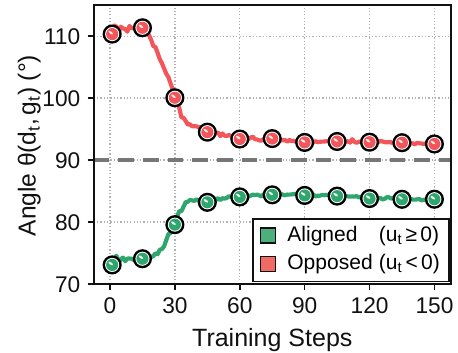}

\caption{Gradient alignment}
\label{fig:analysis_degree_changes}
\end{subfigure}


\caption{
\textbf{(a) Ablations of scaling-direction construction.} All OPRD variants use the step-60 GRPO checkpoint as the weak teacher. For OPSD, a verified draft generated by this teacher is provided as privileged context.
\textbf{(b) Ablation of directional amplification strength.} 
We vary $\lambda$, the coefficient applied to OPRD's directional correction term. $\lambda=0$ corresponds to GRPO. The gray dashed line marks the performance of the teacher checkpoint. All other settings follow the default configurations in \secautoref{subsec:experimental_setup}.
\textbf{(c) Alignment dynamics between $\mathbf d_t$ and $\mathbf g_t$.} On Knights \& Knaves, we track $\theta_t$ between $\mathbf d_t$ and $\mathbf g_t$ during OPRD with a Qwen3-8B-Base student and Qwen3-4B-Base weak teacher. Excluding rollout groups with $\mathbf g_t=\mathbf 0$ (identical rewards within the group), we report token-averaged angles for aligned ($u_t\geq0$) and opposed ($u_t<0$) tokens.
\looseness=-1
}
\label{fig:three_panel_analysis}

\endgroup
\end{figure*}

\vspace{-10pt}
\paragraph{Weak Policy Delta Provides the Most Effective Scaling Direction.}
\autoref{fig:analysis_delta_signal} compares three choices for the guidance direction $\mathbf{d}_t$: the normalized weak policy delta $\boldsymbol{\Delta}_t$, the OPD teacher-matching gradient, and the OPSD self-distillation gradient. The weak policy delta yields the fastest and most sustained gains. Comparing the post-trained teacher with its reference isolates the reward-relevant update, and their log-policy ratio admits an implicit-reward interpretation. OPRD projects the verifier gradient onto this direction and amplifies its aligned component, exploiting the teacher's reward information without inheriting its capacity ceiling. In contrast, OPD captures the full teacher--student mismatch and offers limited acceleration when the teacher is too weak, while OPSD's off-policy supervision can restrict exploration of alternative reasoning paths \citep{kim2026does, kaur2026rethinking}. Although OPD becomes more effective with a better-trained teacher, the weak policy delta is still the fastest and most reliable guidance signal (see \appautoref{app:delta_signal}).
\looseness=-1

\vspace{-10pt}
\paragraph{Sufficient Directional Amplification Enables Early Acceleration.}
\autoref{fig:analysis_lambda_scale} examines $\lambda$, which scales the directional correction and thus controls the strength of teacher guidance. Every $\lambda>0$ improves final Pass@1 over $\lambda=0$ (GRPO). Larger values of $\lambda$ up to $0.5$ also yield faster gains early in training. This systematic relationship between guidance strength and learning speed confirms that OPRD's directional correction indeed drives the observed acceleration. Performance changes little beyond $\lambda=0.5$, so precise tuning is unnecessary once amplification is sufficiently strong. We therefore use $\lambda=0.5$ as the default.
\looseness=-1

\vspace{-10pt}
\paragraph{Teacher Guidance Bootstraps Early Learning but Becomes Less Influential over Time.}
\autoref{fig:analysis_degree_changes} tracks the mean angle $\theta_t$ between the guidance direction $\mathbf{d}_t$ and policy gradient $\mathbf{g}_t$. Early in training, the two directions exhibit substantial alignment for $u_t>0$ and opposition for $u_t<0$. This strong directional coupling allows the weak policy delta to bootstrap student learning. As training proceeds, the mean angle for $u_t>0$ increases toward $90^\circ$, while that for $u_t<0$ decreases toward $90^\circ$. Since $\lVert \operatorname{Proj}_{\mathbf d_t}(\mathbf g_t)\rVert_2/\lVert\mathbf g_t\rVert_2=|\cos\theta_t|$, this convergence toward orthogonality means that the component of $\mathbf g_t$ along $\mathbf d_t$ becomes smaller relative to the full policy gradient. This indicates that the evolving student gradient increasingly follows verifier-supported directions not captured by the teacher shift, so teacher guidance becomes less influential over time.
\looseness=-1

\subsection{Student Behavior under Teacher Guidance}
\label{subsec:qualitative_analysis}


\newlength{\qualitativepanelheight}

\begin{figure*}[!t]
\centering

\setlength{\qualitativepanelheight}{0.32305\linewidth}

\captionsetup[subfigure]{
    justification=centering,
    singlelinecheck=true,
    skip=3pt
}

\begin{subfigure}[b]{0.6395\linewidth}
    \centering

    \begin{minipage}[c][\qualitativepanelheight][c]{\linewidth}
        \centering
        \includegraphics[
            width=\linewidth,
            trim={0 0 7.2bp 0},
            clip
        ]{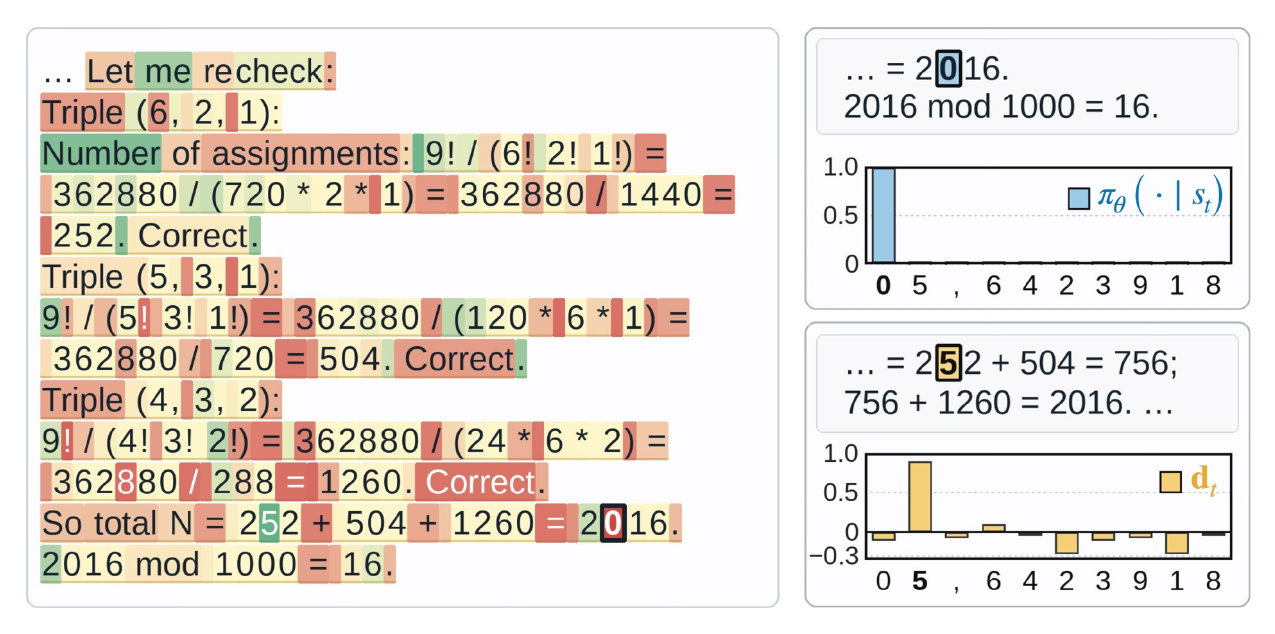}
    \end{minipage}

    \caption{Token alignment and reasoning paths}
    \label{fig:token_correction}
\end{subfigure}%
\hspace{0.0197\linewidth}%
\begin{subfigure}[b]{0.3408\linewidth}
    \centering

    \begin{minipage}[c][\qualitativepanelheight][c]{\linewidth}
        \centering
        \includegraphics[
            width=1.0\linewidth,
            trim={8.1bp 0 0 0},
            clip
        ]{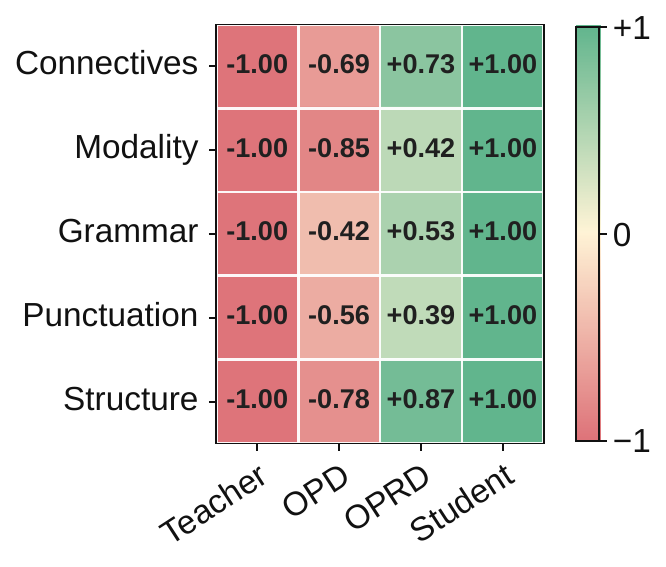}
    \end{minipage}

    \caption{Response style similarity}
    \label{fig:aime24_style_fingerprint}
\end{subfigure}

\caption{
\textbf{(a) Visualizing token alignment and reasoning continuations.} 
An AIME\textquotesingle25 response from the OPRD student at update 150. Green and red indicate positive and negative cosine similarity between $\mathbf d_t$ and the student policy gradient $\mathbf g_t$ with $A_t=1$, respectively (see \appautoref{app:token_visualization}). The token outlined in black, \captionopposedtoken{0}, has the lowest cosine similarity among displayed tokens.
The student's top-1 token \protect\studenttoken{0} completes \texttt{2016} directly. Forcing \protect\shifttoken{5}, the top-1 token under $\mathbf d_t$, leads the same student to this result through an intermediate sum. The plots show the student's top-10 token probabilities above and their teacher-shift values ($\mathbf d_t$) below.
\textbf{(b) Measuring similarity to teacher and student response styles.} 
On AIME\textquotesingle24, we compare response styles using 101 standardized features across five categories.
Normalized distance differences indicate whether each method's average style is closer to the weak teacher (red) or the GRPO-trained student at update 150 (green).
\looseness=-1
}
\label{fig:qualitative_analysis_fingerprint}

\end{figure*}

\paragraph{OPRD Can Move Beyond the Teacher's Reasoning Paths.}
\autoref{fig:token_correction} illustrates how OPRD can exploit an informative teacher shift while allowing the stronger student to follow its own, more direct reasoning path rather than the one favored by the teacher. At the selected AIME\textquotesingle25 prefix, the student's top-1 prediction is \studenttoken{0}, which immediately completes \texttt{2016}. By contrast, the top-1 token under the weak policy shift $\mathbf d_t$ is \shifttoken{5}. Forcing \shifttoken{5} and continuing with the same student produces \texttt{252\,+\,504\,=\,756}, followed by \texttt{756\,+\,1260\,=\,2016}. This detour also reaches the correct result, showing that the teacher shift provides a valid direction that may be useful earlier in training. Here, however, the student can already complete the calculation directly. This is reflected in the highlighted \opposedtoken{0}, which has the most negative alignment with $\mathbf d_t$ among the displayed tokens. OPRD therefore raises the logit of \studenttoken{0} and lowers that of \shifttoken{5} (when $u_t<0$, OPRD amplifies the component of the verifier-driven policy gradient that opposes the teacher shift). The negative-alignment branch thus favors the student's shorter solution over the valid teacher-favored detour, providing a token-level example of how OPRD can move beyond the teacher.
\looseness=-1

\vspace{-10pt}
\paragraph{OPRD Remains Stylistically Closer to the Stronger Student.}
\autoref{fig:aime24_style_fingerprint} examines how teacher guidance affects response style on AIME\textquotesingle24. We summarize each method's average response style using 101 standardized features grouped into five categories: connectives, modality, grammar, punctuation, and sentence and paragraph structure. For each category, a normalized distance difference indicates whether the average style is closer to the teacher (negative) or the GRPO student at update 150 (positive) (see \appautoref{app:style} for details).
At update 150, OPD is closer to the teacher in all five categories, whereas OPRD is closer to the GRPO student. This pattern suggests that the OPRD student can benefit from what the teacher learned without inheriting its response style, consistent with using the teacher shift to rescale the student's own policy gradient rather than matching the teacher policy.
\looseness=-1

\subsection{Discussion of Key Challenges}
\label{subsec:challenges_discussion}

\paragraph{Vanishing Policy Gradients Limit OPRD's Teacher-Guided Correction.}
OPRD requires a nonzero verifier-driven policy gradient. In an additional strong-to-weak experiment pairing a Qwen3-8B-Base teacher with a Qwen3-1.7B-Base student, most Knights \& Knaves responses are invalid, so most rollout groups receive identical rewards (i.e., the resulting group-relative advantages and their contributions to $\mathbf g_t$ therefore vanish). For these groups, the projection onto $\mathbf d_t$ also vanishes, leaving no component for OPRD to amplify and hence no teacher-guided correction. As shown in \autoref{fig:failure_no_signal}, OPRD still accelerates learning relative to GRPO and KDRL, although all three remain below 20\% Pass@1, whereas OPD reaches 40.5\% using dense policy-matching targets that do not depend on verifier rewards. This challenge arises from the student's initial rollout distribution rather than the absence of a useful teacher signal. Such an extreme regime is less likely in our primary weak-to-strong setting, where the student has greater capacity than the teacher, but may still arise on sufficiently difficult tasks. A short task-specific SFT or distillation warm-up could bootstrap valid on-policy behavior before switching to OPRD.
\looseness=-1


\begin{figure*}[!t]
\centering

\begingroup

\captionsetup[subfigure]{
  font=small,
  labelfont=normalfont,
  justification=centering,
  singlelinecheck=true,
  skip=6pt
}


\begin{subfigure}[t]{0.32\linewidth}
\centering

\includegraphics[
    width=\linewidth,
    keepaspectratio
]{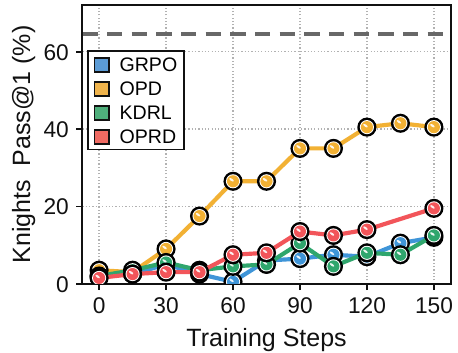}

\caption{Limited gradient signal}
\label{fig:failure_no_signal}
\end{subfigure}%
\hfill
\begin{subfigure}[t]{0.32\linewidth}
\centering

\includegraphics[
    width=\linewidth,
    keepaspectratio
]{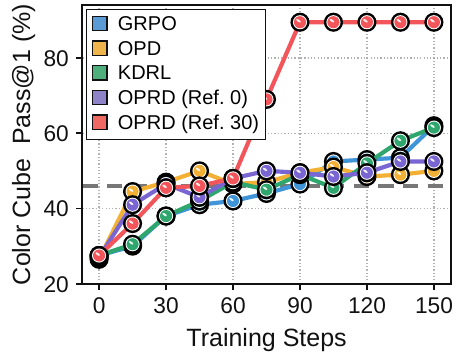}

\caption{Task performance}
\label{fig:failure_performance}
\end{subfigure}%
\hfill
\begin{subfigure}[t]{0.32\linewidth}
\centering

\includegraphics[
    width=\linewidth,
    keepaspectratio
]{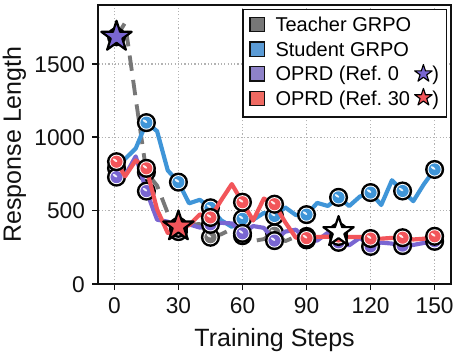}

\caption{Response Length}
\label{fig:failure_response_length}
\end{subfigure}

\caption{
\textbf{(a) Results under limited policy-gradient signal.} On Knights \& Knaves, we transfer a step-105 Qwen3-8B-Base teacher to a Qwen3-1.7B-Base student, with OPRD's negative-branch scale $\lambda_t$ warmed up over the first 75 steps.
\textbf{(b, c) Effect of reference-policy selection on length bias.}
On Color Cube, we transfer a step-105 Qwen3-4B-Base teacher $\pi_T$ to a Qwen3-8B-Base student, using either the step-0 or step-30 checkpoint from the same GRPO run as $\pi_T^{\rm ref}$. OPRD's negative-branch scale $\lambda_t$ is warmed up over the first 75 steps. The gray dashed line marks teacher performance, while the colored stars denote the mean response lengths of the two choices of $\pi_T^{\mathrm ref}$, and the white star marks that of $\pi_T$. All other settings follow \secautoref{subsec:experimental_setup}.
\looseness=-1
}
\label{fig:failure_case_analysis}

\endgroup
\end{figure*}

\vspace{-10pt}
\paragraph{Reference Policy Selection Can Prevent Length Bias from Distorting Teacher Guidance.}
As discussed in \secautoref{subsec:method} and \appautoref{app:asymetric_alignment}, both $\mathbf g_t$ and $\mathbf d_t$ can contain reward-irrelevant components such as $\boldsymbol{\epsilon}_t$, which the projection-and-amplification step can magnify. Response length is one example: when it correlates with verifier reward, both signals can encode a preference for longer or shorter responses, even if changing length does not itself improve reasoning quality. As shown in \figsautoref{fig:failure_performance}{fig:failure_response_length}, the step-0 reference $\pi_T^{\rm ref}$ produces substantially longer responses than the step-105 teacher $\pi_T$ on Color Cube. The resulting shift $\boldsymbol{\Delta}_t$ therefore contains a strong shortening component. With this reference, OPRD rapidly shortens its responses and achieves strong early gains. It nevertheless plateaus at 52.5\% Pass@1, below GRPO and KDRL, suggesting that the teacher-guided correction overemphasizes shortening at the expense of task-relevant reasoning. A simple mitigation is to move the reference to step 30, after the teacher's initial length collapse. This excludes some of the teacher's early gains from $\boldsymbol{\Delta}_t$ but substantially narrows the reference--teacher length gap and weakens the associated bias. OPRD then avoids the plateau and jumps to 89.5\%, discovering a more effective reasoning strategy. \appautoref{app:length_bias} shows the same pattern on Binary Matrix, where this reference policy adjustment is likewise effective.
\looseness=-1

\section{Related Work}
\label{sec:related_work}

\paragraph{Weak-to-Strong Generalization.}
Weak-to-strong generalization has been observed across language understanding, reward modeling, and reasoning, although weak supervision typically recovers only part of the gap to strong supervision \citep{burns2024weaktostrong,yang2024weak}. Analyses attribute the gains to correcting weak pseudo-labels, extending coverage beyond the weak teacher, and differences between teacher and student hypothesis classes or representations \citep{lang2024theoretical,charikar2024quantifying,pmlr-v267-dong25g,xue2025representations,pmlr-v267-medvedev25a}, while naive fine-tuning can instead overfit weak errors \citep{somerstep2025a,yao-etal-2025-revisiting,shi-etal-2025-mitigate}. For reasoning, W2SR-P trains stronger students on verified weak-model trajectories, S2L-PO and related methods use weaker policies to broaden the student's rollouts, and weak critiques can generate and filter improved responses \citep{yuan-etal-2026-incentivizing,ren2026smaller,wang2026takes,jin2026weak}. These methods change the student's training data, exploration, or feedback, whereas OPRD leaves all three unchanged and only rescales the student's own policy gradient.
\looseness=-1

\vspace{-10pt}
\paragraph{On-Policy Distillation.}
Knowledge distillation for language generation has moved from matching teacher distributions on fixed or teacher-generated sequences \citep{hinton2015distilling,kim-rush-2016-sequence} to objectives evaluated on student-generated sequences \citep{gu2024minillm,ko2024distillm}. OPD makes this supervision fully on-policy by querying the teacher along the student's current rollouts, addressing the mismatch between the prefixes seen in training and those the student visits at inference \citep{agarwal2024policy}. It is now a common step in reasoning post-training \citep{yang2025qwen3,zeng2026glm}, and later work uses the same interface to consolidate several specialist teachers into one student \citep{ma2026mopd,team2026kimi,xiao2026mimo}, to exploit privileged information available only during training \citep{zhao2026selfdistilled,ye2026policy}, or to extrapolate the reward implicit in OPD beyond the teacher \citep{yang2026learning}. In all of these, the student is still trained to match a token distribution that the teacher defines, so in the weak-to-strong setting the optimum of the objective is the weak policy itself or a target derived from it.
\looseness=-1

\vspace{-10pt}
\paragraph{Distillation with Reinforcement Learning.}
Methods that combine distillation with verifier-based reinforcement learning differ in how the teacher signal enters optimization. KDRL and later work add a teacher-matching term to the reward objective \citep{xu2025kdrl,ramos2026recipe}. Others modify teacher guidance through policy ratios, reward-based selection, group-level calibration, or token-level interventions \citep{zhang2026reinforcement,akhondzadeh2026reward,zhang2026beyond,ko2026scaling,jia2026asymmetric}, and another uses a privileged self-teacher to control the magnitude of token-level credit \citep{wang2026teach}. However, a teacher-matching loss introduces a second objective that can compete with reward maximization when the teacher favors a solution the verifier does not reward. OPRD adds no such objective and optimizes reward alone.
\looseness=-1

\vspace{-10pt}
\clearpage
\paragraph{Transferring Policy Shifts.}
Several methods transfer the shift between a post-trained policy and its reference rather than the final policy alone. During decoding, this shift can steer a larger frozen model \citep{liu2024tuning,zhou2024weaktostrong}. During training, it has been used as an alignment target for a stronger model \citep{zhu2025weaktostrong}, as a proxy teacher built on the student's base policy in W2S-OPD \citep{yu2026weak}, and as a dense reward on student rollouts in Direct-OPD \citep{heo2026policy,feng2026weak}. In each case the shift itself becomes an optimization target, and it carries only the improvements the weak teacher realized. OPRD instead uses the weak policy delta to rescale the student's policy gradient, so the direction transfers without the shift becoming a target.
\looseness=-1

\vspace{-10pt}
\paragraph{Gradient Manipulation.}
Multi-task optimization combines objectives at the level of gradients rather than losses. Gradient surgery projects one task gradient onto the normal plane of another when the two conflict \citep{yu2020gradient}, a moving average of past gradients makes this projection more stable \citep{pmlr-v235-hsieh24a}, and auxiliary gradients can be gated by their cosine similarity with the main gradient \citep{du2019adapting,zhou2022on}. In these methods, every direction is the gradient of a loss the model itself optimizes, and conflicting components are removed or down-weighted. OPRD is closest to this family in form, but it removes nothing and only amplifies the component of the student's policy gradient that already points along the teacher direction, so the stationary points of the student objective in logit space do not move.
\looseness=-1

\section{Conclusion}
We introduce On-Policy Reverse Distillation (OPRD), which transfers a weak teacher's post-training policy shift by amplifying the aligned component of a stronger student's policy gradient without making the teacher policy an optimization target. By rescaling rather than replacing the student gradient, OPRD accelerates learning while preserving the policy objective's stationary points in logit space. Across successive model transfer and multi-domain consolidation, OPRD reaches teacher-level performance in substantially fewer updates than policy optimization alone and continues improving after on-policy distillation plateaus near the teacher. Its gains extend to strong-to-weak distillation, showing effectiveness under both capacity orderings. Qualitatively, the OPRD student's response style remains closer to the reward-only baseline than to the teacher, consistent with the shift being expressed through the student's own policy rather than imitation. OPRD thus enables efficient transfer from smaller specialists without defining the student's optimization target or limiting its performance.
\looseness=-1

\subsection{Future Works}

\paragraph{Scaling to Larger Models.} 
Our results cover Qwen3 and Qwen2.5-VL models from 0.6B to 14B parameters and both weak-to-strong and strong-to-weak capacity orderings. At much larger scales, OPRD may be especially useful because learning a policy shift with a smaller model could be substantially cheaper than optimizing the larger model directly from verifier feedback. Larger-scale experiments would test how students with greater capacity use the same teacher shift, how far they can improve beyond the teacher, and whether the gains in update efficiency persist as post-training costs increase.
\looseness=-1

\vspace{-10pt}
\paragraph{Systems Considerations at Scale.} 
OPRD adds no-gradient forward passes through the frozen teacher and reference policies and a correction of the student's logit gradient. In weak-to-strong setup, both frozen policies are smaller than the student and require neither generation nor backward propagation, while the correction retains one additional dense logit-gradient tensor. As shown in \appautoref{app:cost_memory}, these additions only increase wall-clock time by 11.9\% and peak GPU memory by 10.2\% relative to GRPO. But at frontier-model scale, keeping this overhead modest will require efficient placement, sharding, and scheduling of the frozen policies within hybrid parallelism, together with communication-efficient correction across model and vocabulary shards.
\looseness=-1

\vspace{-10pt}
\paragraph{Extending OPRD to Pretraining.}
Although our experiments focus on verifier-driven post-training, the core gradient-scaling operation requires only a token-level student gradient and a direction extracted from a weaker model. This suggests that OPRD could also be adapted to pretraining: the policy shift between two checkpoints of a smaller model could guide a larger student's next-token prediction updates while the student continues to optimize its own language-modeling objective rather than matching the smaller model's distribution. Evaluating this extension in both initial and continued pretraining would test whether directional transfer improves training efficiency across model scales and data distributions.
\looseness=-1

\vspace{-10pt}
\paragraph{Toward Recursive Self-Improvement.}
An important direction for future work is to connect weak-to-strong distillation with recursive self-improvement, where each model generation contributes to the development of more capable successors through training supervision, evaluation, and algorithmic improvements. These successors, in turn, use their greater capabilities to improve subsequent model development. For example, earlier models helped supervise GPT-6 Astra's training \citep{openai2026gpt6astra}, while Google reports using agentic loops to recursively evaluate and refine Gemini 3.8 Flash \citep{google2026gemini38flash}. A promising extension is to incorporate reverse distillation into these workflows, allowing earlier generations to contribute not only to training supervision and development but also directly to their successors' policy updates through their post-training policy shifts. Building such pipelines would allow us to test whether reverse distillation can consistently improve sample efficiency and accelerate training as each successor becomes a teacher for the next generation.
\looseness=-1

\subsection{AI Use Statement}
In this work, we used generative AI tools to improve the clarity and readability of the manuscript and to assist with code review and verification. We have reviewed all AI-assisted work, including the suggested revisions and code feedback. We take responsibility for the final content of this work, including text and code produced with the aid of generative AI.

\subsection{Acknowledgements}
We thank Kee-Eung Kim for facilitating access to computational resources through the National AI Research Hub project. We thank Rishabh Agarwal for helpful discussions on related work and algorithm design. We also thank Reza Bayat for feedback on the manuscript.

\clearpage
\bibliographystyle{plainnat}
\bibliography{paper}

\clearpage
\appendix{\setlength{\cftbeforesecskip}{16pt}
\setlength{\cftbeforesubsecskip}{4pt}
\renewcommand\cftsecpagefont{\color{RoyalBlue}}
\renewcommand\cftsubsecpagefont{\color{RoyalBlue}}
\renewcommand\cftsubsubsecpagefont{\color{RoyalBlue}}
\renewcommand{\contentsname}{\large{Contents}}
{
  \hypersetup{linkcolor=}
  \tableofcontents
}

\clearpage

\section{Revisiting the Term Distillation in Our Method}
\label{app:method_discussion}

\noindent Conventional distillation uses a teacher's outputs as targets for student learning. OPRD uses the teacher's learned policy shift to rescale the student's own reward gradient. When the verifier-driven student gradient aligns with the teacher's post-training shift, OPRD amplifies a teacher-following update, consistent with the usual intuition of distillation. In weak-to-strong transfer, the student may discover reward-supported improvements beyond the smaller teacher's capabilities, so OPRD lets the student's reward gradient determine which sign to follow. When the student gradient opposes the teacher's learned direction, OPRD amplifies that opposing component. This differs from conventional distillation, which trains the student to match the teacher's outputs.
\looseness=-1

\noindent OPRD can therefore be interpreted as scaling the student's own update along an axis defined by the teacher's learned policy shift. Using the gradient decomposition in \eqautoref{eq:oprd_gradient_scaling},
\[
\mathbf g_t=u_t\mathbf d_t+\mathbf g_t^\perp,
\qquad
\widetilde{\mathbf g}_t=(1+\lambda_t)u_t\mathbf d_t+\mathbf g_t^\perp,
\qquad
u_t:=\mathbf d_t^\top\mathbf g_t.
\]
The teacher determines which axis is scaled, while the student's reward gradient determines the sign of its component along that axis. Multiplying this component by $1+\lambda_t$ strengthens it whether it follows or opposes the teacher's shift, while leaving $\mathbf g_t^\perp$ unchanged. For a nominal logit step size $\eta$, this is equivalent to using $\eta(1+\lambda_t)$ along the teacher axis and $\eta$ in its orthogonal complement. OPRD thus rescales the student's existing reward gradient without introducing a teacher-matching objective.
\looseness=-1

\noindent Related optimization methods illustrate why selected directions can be useful without being followed as optimization targets. Guided ES~\citep{pmlr-v97-maheswaranathan19a} concentrates random search in a subspace suggested by imperfect surrogate gradients, while using evaluations of the original objective to estimate the update. FOSI~\citep{sivan2024fosi} applies curvature-based updates along selected Hessian eigendirections and a base optimizer in the orthogonal complement. Both methods use information about particular directions to improve optimization of an existing objective. OPRD obtains such information from the policy change already learned by a smaller model and uses it to rescale the larger student's reward gradient.
\looseness=-1

\noindent In traditional policy matching, minimizing KL divergence makes the teacher's output distribution an explicit target. OPRD introduces no such target. Nevertheless, the teacher's learning supplies an axis along which the student's own reward gradient is rescaled. We therefore use the term distillation to describe the transfer of information from teacher learning, while recognizing that the student can learn a different policy and outperform the teacher.
\looseness=-1


\clearpage
\section{Optimization Properties of Teacher-Direction Scaling}
\label{app:oprd-convergence}

\noindent OPRD multiplies the token-level policy gradient by $\mathbf I+\lambda_t\mathbf d_t\mathbf d_t^\top$, which amplifies the component along $\mathbf d_t$ by $1+\lambda_t$ and leaves the orthogonal component unchanged. Over a full response, the transformed token gradients vanish exactly where the original token gradients do, and the corresponding logit update has an ascent bound with an additional nonnegative term.
\looseness=-1

\noindent Fix a response $y$ with $T$ valid tokens and let $\mathbf z=(\mathbf z_1,\ldots,\mathbf z_T)$ collect its next-token logits, with $\pi(\cdot\mid\mathbf z_t)=\operatorname{softmax}(\mathbf z_t)$. The advantages $A_t$ and the directions $\mathbf d_t$ do not depend on $\mathbf z$, and the token gradient in \eqautoref{eq:rlvr_token_gradient} is the gradient of the objective for this response,
\begin{equation}
J(\mathbf z)=\sum_{t=1}^{T}A_t\log\pi(y_t\mid\mathbf z_t),
\qquad
\mathbf g_t=\nabla_{\mathbf z_t}J(\mathbf z).
\label{eq:oprd-response-objective}
\end{equation}
Each block of the Hessian of $J$ is $A_t$ times that of $\log\pi(y_t\mid\mathbf z_t)$, whose eigenvalues lie in $[-\tfrac12,0]$, so $J$ is $L$-smooth with $L=\max_t|A_t|/2$.

\begin{proposition}[Stationarity and One-Step Ascent]
\label{prop:oprd-beyond-teacher}
Let $\mathbf z_k$ be the current logits and write $\mathbf g_t=\nabla_{\mathbf z_t}J(\mathbf z_k)$ and $u_t=\mathbf d_t^\top\mathbf g_t$ for the alignment coefficient, with $\lVert\mathbf d_t\rVert_2=1$ and $\lambda_t\geq0$ fixed for this step. With step size $\eta>0$, set
\begin{equation}
\widetilde{\mathbf g}_t=(\mathbf I+\lambda_t\mathbf d_t\mathbf d_t^\top)\mathbf g_t,
\qquad
\mathbf z_{k+1,t}=\mathbf z_{k,t}+\eta\,\widetilde{\mathbf g}_t.
\label{eq:oprd-response-logit-update}
\end{equation}
Then $\widetilde{\mathbf g}_t=\mathbf 0$ for every $t$ if and only if $\nabla_{\mathbf z}J(\mathbf z_k)=\mathbf 0$. If in addition $\eta L(1+\bar\lambda)\leq1$ with $\bar\lambda=\max_t\lambda_t$,
\begin{equation}
J(\mathbf z_{k+1})-J(\mathbf z_k)
\;\geq\;
\frac{\eta}{2}\lVert\nabla_{\mathbf z}J(\mathbf z_k)\rVert_2^2
+\frac{\eta}{2}\sum_{t}\lambda_tu_t^2 .
\label{eq:oprd-contraction}
\end{equation}
\end{proposition}

\begin{proof}
Let $\mathbf g=\nabla_{\mathbf z}J(\mathbf z_k)$ and let $\mathbf P$ be the block-diagonal matrix with blocks $\mathbf I+\lambda_t\mathbf d_t\mathbf d_t^\top$, so that $\mathbf z_{k+1}=\mathbf z_k+\eta\mathbf P\mathbf g$. Each block has eigenvalue $1+\lambda_t$ along $\mathbf d_t$ and $1$ on the orthogonal complement. Hence $\mathbf P$ is positive definite and therefore invertible, which gives the first claim, and
\[
\mathbf g^\top\mathbf P\mathbf g=\lVert\mathbf g\rVert_2^2+\sum_t\lambda_tu_t^2,
\qquad
\mathbf P^2\preceq(1+\bar\lambda)\mathbf P .
\]
By $L$-smoothness,
\[
\begin{aligned}
J(\mathbf z_{k+1})
&\geq
J(\mathbf z_k)+\eta\,\mathbf g^\top\mathbf P\mathbf g-\frac{L\eta^2}{2}\,\mathbf g^\top\mathbf P^2\mathbf g
\\
&\geq
J(\mathbf z_k)+\eta\Big(1-\frac{L\eta(1+\bar\lambda)}{2}\Big)\mathbf g^\top\mathbf P\mathbf g
\\
&\geq
J(\mathbf z_k)+\frac{\eta}{2}\,\mathbf g^\top\mathbf P\mathbf g
\\
&=
J(\mathbf z_k)+\frac{\eta}{2}\lVert\mathbf g\rVert_2^2+\frac{\eta}{2}\sum_t\lambda_tu_t^2 .
\end{aligned}
\]
\end{proof}

\noindent Setting $\lambda_t=0$ in \eqautoref{eq:oprd-contraction} recovers the bound $\tfrac{\eta}{2}\lVert\nabla_{\mathbf z}J(\mathbf z_k)\rVert_2^2$ of an unscaled step, so the second term is what scaling adds. It grows with the component of the verifier-driven policy gradient along the teacher direction and disappears when the two are orthogonal at every token. Scaling therefore adds to the progress guaranteed at each step without changing where the update vanishes, and the price is the tighter condition $\eta L(1+\bar\lambda)\leq1$ on the step size, since the scaled update is longer. The gain depends on $u_t^2$, so alignments of equal magnitude contribute equally whether the student follows or opposes the teacher. \appautoref{app:alignment_amplification} analyzes what changes when the two branches use different scales.
\looseness=-1

\noindent The stationarity and ascent statements concern the fixed-response objective $J(\mathbf z)$ in token logit space. Note that they do not assert preservation of stationary points or the same ascent bound in neural network parameter space. For a parameterization $\mathbf z_\theta$ with Jacobian $\mathbf K_\theta$, the backward signals are $\mathbf K_\theta^\top\mathbf g$ and $\mathbf K_\theta^\top\mathbf P\mathbf g$. Although $\mathbf P$ is invertible, cancellations after multiplication by $\mathbf K_\theta^\top$ need not be preserved.
\looseness=-1

\clearpage

\section{Analysis of Asymmetric Alignment Scaling}
\label{app:asymetric_alignment}

\subsection{One-Sided Amplification under Positive-Only Scaling}
\label{app:alignment_amplification}

Positive-only amplification is locally well motivated. When $u_t>0$, the component of the sampled student gradient $\mathbf g_t$ along the teacher-derived direction $\mathbf d_t$ follows the teacher's post-training shift. Amplifying this component therefore reinforces an update supported by both the teacher shift and the verifier-driven student gradient. A related positive-gating rule is used by \citet{du2019adapting}, who weight auxiliary updates by the positive part of their gradient cosine similarity.
\looseness=-1

However, applying different scales to the two alignment signs introduces a one-sided effect. Let $\lambda_+$ and $\lambda_-$ denote the scales applied when $u_t\geq0$ and $u_t<0$, respectively. The coefficient multiplying $\mathbf d_t$ in the added correction is
\begin{equation}
\begin{aligned}
c_t
&:=
\lambda_+ u_t \mathbf{1}\{u_t \geq 0\}
+
\lambda_- u_t \mathbf{1}\{u_t < 0\}
\\
&=
\tfrac{\lambda_+ + \lambda_-}{2}\,u_t
+
\tfrac{\lambda_+ - \lambda_-}{2}\,
\lvert u_t\rvert.
\end{aligned}
\end{equation}
This decomposition separates sign-symmetric scaling from the asymmetry introduced by using different scales for the two branches. The first term symmetrically scales $u_t$ by the average of the two branch scales. Because it preserves the sign of $u_t$, positive and negative contributions can cancel across tokens. The second term depends on $\lvert u_t\rvert$ and appears only when the branch scales differ. In particular, when $\lambda_+>\lambda_-$, this term remains nonnegative for either sign of $u_t$. It therefore cannot be canceled by changes in the alignment sign, leaving a one-sided coefficient on the local teacher direction.
\looseness=-1

Under positive-only scaling, $\lambda_-=0$. If positive and negative alignments nearly balance across sampled tokens and rollouts, such that $\mathbb E[u_t]\approx0$, then
\begin{equation}
\begin{aligned}
\mathbb E[c_t]
&=
\tfrac{\lambda_+}{2}
\left(
\mathbb E[u_t]
+
\mathbb E[\lvert u_t\rvert]
\right)
\\
&\approx
\tfrac{\lambda_+}{2}\,
\mathbb E[\lvert u_t\rvert]
>
0.
\end{aligned}
\end{equation}
Thus, even when the signed alignments cancel on average, the scalar coefficient $c_t$ remains positive on average under positive-only scaling. This residual coefficient is governed by the mean alignment magnitude $\mathbb E[\lvert u_t\rvert]$, rather than the small signed mean $\mathbb E[u_t]$.
\looseness=-1

Importantly, this residual amplification need not reflect only reward-relevant teacher progress. The sign of $u_t$ reveals whether $\mathbf d_t$ and $\mathbf g_t$ agree, but not why they agree. At an individual sampled token, we write $\mathbf g_t=\mathbf g_t^\star+\boldsymbol{\epsilon}_t$, where $\mathbf g_t^\star$ denotes the underlying reward-improving signal and $\boldsymbol{\epsilon}_t$ aggregates incidental or misattributed components arising from coarse response-level credit assignment, rollout and mini-batch sampling, and estimator-specific effects that need not correspond to actions responsible for higher reward. Similarly, $\mathbf d_t$ captures all changes induced by teacher post-training, including both reward-relevant progress and incidental behavioral changes. Positive alignment may therefore arise from either useful teacher-acquired progress or an incidental tendency shared by the two vectors. Positive-only scaling cannot distinguish between these cases and amplifies the aligned component regardless of its source.
\looseness=-1

Response length provides one concrete example of such a shared tendency. In reasoning tasks, higher verifier rewards are often associated with longer reasoning traces, so the student gradient $\mathbf g_t$ may favor token-level changes that prolong generation. The teacher direction $\mathbf d_t$ may encode a similar tendency acquired during teacher post-training. This tendency may represent useful additional reasoning, but it may also reflect length-dependent effects in the policy-gradient estimate \citep{liu2025understanding}. When it is shared by both signals, positive-only scaling amplifies it whenever it produces positive alignment.
\looseness=-1

\clearpage
\subsection{Isolating the Positive and Negative Alignment Branches}
\label{app:isolating_pos_neg_align}

To examine the branch-specific effects, we isolate the two alignment branches by activating gradient scaling only when $u_t\geq0$ (\emph{positive-only}) or only when $u_t<0$ (\emph{negative-only}), while holding all other training settings fixed within each task. \autoref{fig:asymmetric_scaling} reports evaluation performance and response length during training on mathematics and Knights \& Knaves tasks. Across both tasks, positive-only scaling produces rapid early gains accompanied by a sharp increase in response length. Performance then begins to decline as responses grow toward the generation limit. Negative-only scaling exhibits the opposite pattern: responses become shorter, while performance quickly falls to near zero.
\looseness=-1


\begin{figure*}[!h]
\vspace{5pt}
\centering

\begingroup

\captionsetup[subfigure]{
  font=small,
  labelfont=normalfont,
  justification=centering,
  singlelinecheck=true,
  skip=3pt
}


\includegraphics[
    width=0.95\textwidth
]{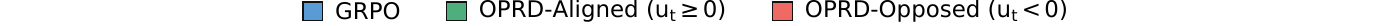}

\par\vspace{3pt}


\begin{subfigure}[t]{0.492\textwidth}
\centering

\includegraphics[
    width=0.49\linewidth
]{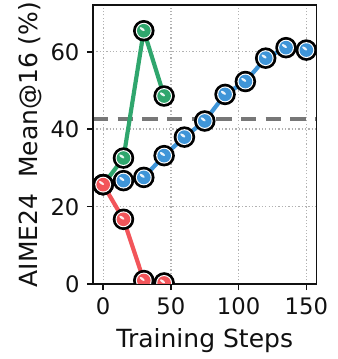}
\hfill
\includegraphics[
    width=0.49\linewidth
]{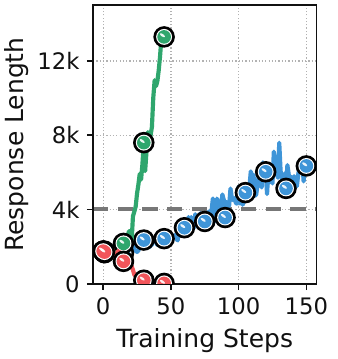}

\vspace{2pt}
\caption{Math}
\label{fig:asymmetric_scaling_math}
\end{subfigure}
\hfill
\begin{subfigure}[t]{0.492\textwidth}
\centering

\includegraphics[
    width=0.49\linewidth
]{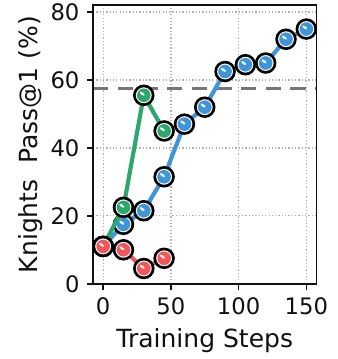}
\hfill
\includegraphics[
    width=0.49\linewidth
]{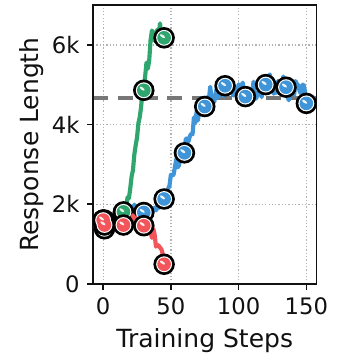}

\vspace{2pt}
\caption{Knights \& Knaves}
\label{fig:asymmetric_scaling_reasoning_gym}
\end{subfigure}


\caption{
\textbf{Evaluation performance and response length over training for OPRD variants with only the positive- or negative-alignment branch active.}
For Math and Knights \& Knaves, we use Qwen3-4B and Qwen3-4B-Base teacher checkpoints obtained after 75 and 105 RL training steps, respectively. The two single-branch OPRD variants are trained for 45 steps with $\lambda=0.5$, while GRPO is shown through 150 steps for reference. The maximum generation lengths are 20K and 8K tokens for the two settings, respectively. The gray dashed lines denote the performance or response length of the corresponding weak teachers. All other training settings follow the dataset-specific default configurations described in \autoref{app:exp_details}.
\looseness=-1
}
\label{fig:asymmetric_scaling}

\endgroup

\vspace{10pt}
\end{figure*}

The rapid gains under positive-only scaling suggest that teacher-aligned components provide effective early transfer of the progress acquired during teacher post-training. By contrast, the collapse under negative-only scaling suggests that teacher-opposing components are less reliable early in training, when the student's rollouts remain weak. Because the verifier provides only response-level feedback, even a rewarded trajectory may contain locally unhelpful token choices whose gradients are negatively aligned with the teacher shift. Applying negative-branch scaling at full strength from the outset can therefore reinforce unreliable token-level updates.
\looseness=-1

The response-length dynamics are also consistent with the shared tendency discussed in \appautoref{app:alignment_amplification}. In both tasks, performance improvements under GRPO are accompanied by longer reasoning traces, suggesting that the student gradient $\mathbf g_t$ favors token-level changes that prolong generation. The teacher develops a similar tendency during RL post-training, which may be encoded in $\mathbf d_t$. Positive-only scaling reinforces this shared tendency and rapidly drives responses toward the generation limit. Negative-only scaling instead amplifies student-gradient components that oppose $\mathbf d_t$, counteracting the length-increasing tendency and producing shorter responses.
\looseness=-1

\clearpage
\subsection{Mitigating One-Sided Amplification through Branch Scheduling}
\label{app:alignment_schedules}

The isolated-branch results suggest that the positive and negative branches play complementary roles over training. The positive branch amplifies components supported by both the teacher shift and the verifier-driven student gradient, thereby providing rapid early transfer. The negative branch instead amplifies verifier-supported departures from the teacher direction, which may help the stronger student move beyond the weak teacher. However, these departures are less reliable early in training, when the student's on-policy rollouts remain weak. This difference motivates controlling the relative strengths of the two branches over training.
\looseness=-1

We compare three strategies for avoiding persistent one-sided amplification. Under the default OPRD schedule, $\lambda_+$ remains fixed at $\lambda$, while $\lambda_-$ gradually increases from $0$ to $\lambda$. Early in training, the larger positive-branch scale prioritizes teacher-aligned components and provides an effective bootstrap. As $\lambda_-$ increases, verifier-supported gradient components whose projections oppose the teacher direction receive progressively greater amplification. Once $\lambda_-=\lambda_+=\lambda$, the asymmetric term proportional to $\lvert u_t\rvert$ vanishes and the correction coefficient reduces to $c_t=\lambda u_t$. The schedule thus preserves rapid teacher-aligned transfer early in training while gradually introducing stronger departures from the weak teacher.
\looseness=-1

Alternatively, we keep $\lambda_-=0$ and gradually decrease $\lambda_+$ from $\lambda$ to $0$. This schedule likewise uses positive-branch scaling as an early bootstrap but progressively removes the added teacher-direction correction. Once $\lambda_+=0$, both branch scales are zero, so the transformed gradient reduces to the original verifier-driven policy gradient and training returns to GRPO. As a schedule-free alternative, we also consider fixed symmetric scaling, which sets $\lambda_+=\lambda_-=\lambda$ throughout training. This removes one-sided amplification from the outset but activates the distinct effects of both branches simultaneously.
\looseness=-1


\begin{figure*}[!h]
\centering

\begingroup

\captionsetup[subfigure]{
  font=small,
  labelfont=normalfont,
  justification=centering,
  singlelinecheck=true,
  skip=7pt
}


\includegraphics[
    width=0.8\textwidth
]{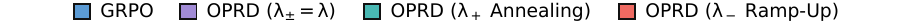}

\par\vspace{3pt}


\begin{subfigure}[t]{0.32\textwidth}
\centering
\includegraphics[
    width=\linewidth
]{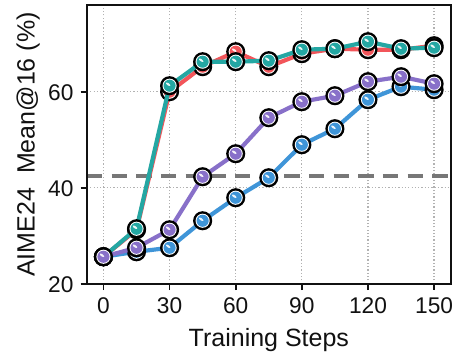}

\caption{Math}
\label{fig:branch_scheduling_math}
\end{subfigure}
\hspace{0.025\textwidth}
\begin{subfigure}[t]{0.32\textwidth}
\centering
\includegraphics[
    width=\linewidth
]{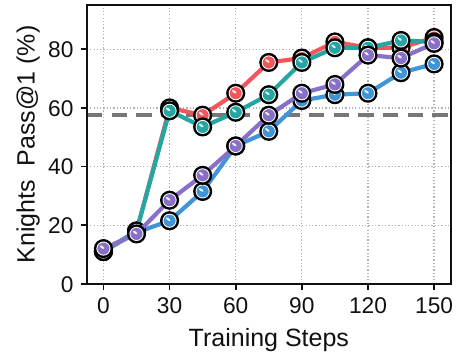}

\caption{Knights \& Knaves}
\label{fig:branch_scheduling_knights}
\end{subfigure}

\caption{
\textbf{Comparison of three branch-scheduling strategies.}
We compare the default $\lambda_-$ ramp-up ($\lambda_+=0.5$, $\lambda_-:0\rightarrow0.5$), $\lambda_+$ annealing ($\lambda_+:0.5 \rightarrow 0$, $\lambda_-=0$), and fixed symmetric scaling ($\lambda_+=\lambda_-=1.0$) against GRPO. The two scheduled variants use horizons of 30 updates for Math and 75 updates for Knights \& Knaves.
We use Qwen3-4B and Qwen3-4B-Base teacher checkpoints obtained after 75 and 105 RL training steps, respectively. Gray dashed lines denote teacher performance. All other training configurations follow \appautoref{app:exp_details}.
\looseness=-1
}
\label{fig:branch_scheduling}

\endgroup

\vspace{6pt}
\end{figure*}

As shown in \autoref{fig:branch_scheduling}, the two scheduled variants begin with positive-only amplification and achieve rapid early gains, whereas fixed symmetric scaling improves much more slowly despite using $\lambda=1.0$: it only gradually breaks through on Math and yields limited early gains on Knights \& Knaves. Because response-level feedback can reward trajectories containing locally incorrect or incidental steps, the resulting $u_t<0$ components are less reliable on weak early rollouts and can dampen the positive-branch bootstrap when amplified from the outset. Activating only $\lambda_+$ is therefore the more reliable default for early acceleration.
\looseness=-1

The later acceleration of fixed symmetric scaling on Knights \& Knaves suggests that $\lambda_-$ becomes useful once the student reaches a stronger regime and produces more informative on-policy rollouts. At this stage, it can amplify meaningful verifier-supported departures discovered through the student's own rollouts, helping it move beyond the weak teacher. Although only $\lambda_+$ annealing shows that returning to verifier-only optimization after the initial bootstrap is also viable, it forgoes explicit amplification of these student-discovered departures. We therefore adopt $\lambda_-$ ramp-up as the default: it preserves the early acceleration from $\lambda_+$ while introducing $\lambda_-$ later to remove persistent one-sided amplification and support progress beyond the weak teacher.
\looseness=-1


\clearpage
\section{Training and Evaluation Details}
\label{app:exp_details}

\autoref{tab:shared_training_details} and \oprdaddition{\autoref{tab:code_agent_training_details} give the shared training settings for Math and Reasoning Gym, and for Code and Agent, respectively.} \autoref{tab:method_training_details} summarizes the method-specific training settings used with these shared configurations. \oprdaddition{Our default training setup uses four NVIDIA B200 GPUs with Fully Sharded Data Parallel (FSDP) \citep{zhao2023pytorch}.} Scenario-specific settings are provided in their respective Appendix sections.
\looseness=-1

\begingroup
\makeatletter
\setlength{\@fptop}{0pt}
\setlength{\@dblfptop}{0pt}
\makeatother

\begin{table*}[!ht]
\caption{
\textbf{Default training settings for Math and Reasoning Gym.}
These configurations are shared across all methods. Method-specific settings are provided in \autoref{tab:method_training_details}.
}
\label{tab:shared_training_details}
\centering

\begingroup
\small
\setlength{\tabcolsep}{5pt}
\renewcommand{\arraystretch}{1.4}

\begin{tabularx}{\linewidth}{
@{}
>{\raggedright\arraybackslash}p{0.22\linewidth}
>{\raggedright\arraybackslash}X
>{\raggedright\arraybackslash}X
@{}
}
\toprule

\textbf{Settings}
& \multicolumn{1}{c}{\textbf{Math}}
& \multicolumn{1}{c}{\textbf{Reasoning Gym}} \\
\midrule

\rowcolor{blue!8}
\multicolumn{3}{@{}l}{\textbf{Data and Models}} \\

Training data
& DAPO-Math-17K
& Knights \& Knaves, Quantum Lock, String Manipulation, and Countdown
\newline
(19,800 examples per task) \\

Prompt format
& Chat template with a system prompt
& Chat template without a system prompt \\

Student policy
& Qwen3-8B (non-thinking)
& Qwen3-8B-Base (non-thinking) \\

Teacher policy
& Qwen3-4B (step 75, non-thinking)
& Qwen3-4B-Base (step 75 for String task,\newline step 105 for the other tasks, non-thinking) \\

\midrule

\rowcolor{blue!8}
\multicolumn{3}{@{}l}{\textbf{Optimization}} \\

Training horizon
& 150 policy updates
& 150 policy updates \\

Prompt batch / mini-batch
& 64 / 64
& 64 / 32 \\

Rollouts per prompt
& 8
& 8 \\

Optimizer
& AdamW, $\beta=(0.9,0.999)$, weight decay $0.01$, gradient clipping $1.0$
& AdamW, $\beta=(0.9,0.999)$, weight decay $0.01$, gradient clipping $1.0$ \\

Learning rate
& $1\times10^{-6}$ (constant schedule with \newline10 warm-up updates)
& $1\times10^{-6}$ (constant schedule with \newline10 warm-up updates) \\

Policy optimization
& PPO clipping range $[0.20,0.28]$, \newline
no standard-deviation normalization, \newline
no KL or entropy regularization
& PPO clipping range $[0.20,0.28]$, \newline
no standard-deviation normalization, \newline
no KL or entropy regularization \\

Task reward
& Verified final answer:\newline $1$ if correct, $0$ otherwise
& K\&K, String Manipulation, and Countdown:\newline
$1$ if correct, $0$ otherwise\newline
Quantum Lock:\newline
$1$ for a reference-length valid path,\newline
$0.5$ for any other valid path, $0$ otherwise \\

Length-based reward
& No penalty up to 16,384 tokens, then\newline linear penalty
reaching $-1$ at 20,480 tokens
& -- \\

\midrule

\rowcolor{blue!8}
\multicolumn{3}{@{}l}{\textbf{Generation}} \\

Training-time decoding
& Temperature $1.0$, top-$p$ $1.0$, no top-$k$
& Temperature $1.0$, top-$p$ $1.0$, no top-$k$ \\

Maximum prompt length
& 2,048 tokens
& 2,048 tokens \\

Maximum response length
& 20,480 tokens
& 8,192 tokens \\

\bottomrule
\end{tabularx}

\endgroup
\vspace{10pt}
\end{table*}

\begin{table*}[!t]
\caption{\oprdaddition{\textbf{Default training settings for Code and Agent.} These configurations are shared across all methods. Method-specific settings are provided in \autoref{tab:method_training_details}. ViLaSR RL comprises four task families: Maze, GPT4Scene, SR-91k, and Spatial VQA. For numeric Agent answers, mean relative accuracy (MRA) averages ten binary correctness scores at relative-error tolerances $\tau\in\{0.05,0.10,\ldots,0.50\}$. For prediction $\hat{y}$ and target $y$, each score is $1$ when $|\hat{y}-y|/|y|\leq\tau$ and $0$ otherwise, giving an MRA reward from $0$ to $1$ in increments of $0.1$.}}
\label{tab:code_agent_training_details}
\centering
\begingroup
\color{revisiongreen}
\small
\setlength{\tabcolsep}{3pt}
\renewcommand{\arraystretch}{1.4}
\begin{tabularx}{\linewidth}{@{}>{\raggedright\arraybackslash}p{0.25\linewidth}>{\raggedright\arraybackslash}X>{\raggedright\arraybackslash}X@{}}
\toprule
\textbf{Settings} & \multicolumn{1}{c}{\textbf{Code}} & \multicolumn{1}{c}{\textbf{Agent}} \\
\midrule
\rowcolor{blue!8}
\multicolumn{3}{@{}l}{\textbf{Data and Models}} \\
Training data & DeepCoder Preview\newline (24,287 problems) & ViLaSR RL\newline (19,200 examples, 4,800 per family) \\
Prompt format & Chat template without a system prompt & Chat template with a system prompt \\
Student policy & Qwen3-14B-Base (non-thinking) & Qwen2.5-VL-7B (ViLaSR SFT) \\
Teacher policy & Qwen3-8B-Base (step 435, non-thinking) & Qwen2.5-VL-3B-Instruct (step 450,\newline initialized with 3-epoch SFT) \\
\midrule
\rowcolor{blue!8}
\multicolumn{3}{@{}l}{\textbf{Optimization}} \\
Training horizon & 150 policy updates & 225 policy updates \\
Prompt batch / mini-batch & 64 / 32 & 64 / 64 (16 per family) \\
Rollouts per prompt & 8 & 8 \\
Optimizer & AdamW, $\beta=(0.9,0.999)$, weight decay $0.01$, gradient clipping $1.0$ & AdamW, $\beta=(0.9,0.999)$, weight decay $0.01$, gradient clipping $1.0$ \\
Learning rate & $1\times10^{-6}$ (constant schedule with \newline10 warm-up updates) & $1\times10^{-6}$ (constant schedule with \newline10 warm-up updates) \\
Policy optimization & PPO clipping range $[0.20,0.28]$, \newline no standard-deviation normalization, \newline no KL or entropy regularization & PPO clipping range $[0.20,0.28]$, \newline no standard-deviation normalization, \newline no KL or entropy regularization \\
Task reward & Execution test results:\newline $1$ if all tests pass, $0$ otherwise & Multiple-choice answers:\newline $1$ if correct, $0$ otherwise\newline Numeric answers: MRA\newline Invalid or incomplete trajectories score $0$ \\
Length-based reward & No penalty up to 12,288 tokens, then\newline linear penalty reaching $-1$ at 16,384 tokens & -- \\
\midrule
\rowcolor{blue!8}
\multicolumn{3}{@{}l}{\textbf{Generation}} \\
Training-time decoding & Temperature $1.0$, top-$p$ $1.0$, no top-$k$ & Temperature $1.0$, top-$p$ $1.0$, no top-$k$ \\
Maximum prompt length & 2,048 tokens & 8,192 tokens \\
Maximum response length & 16,384 tokens & 16,384 tokens\newline (including tool observations) \\
Maximum action length & -- & 4,096 tokens per model turn \\
Maximum observation length & -- & 8,192 tokens per tool observation \\
Available tools & -- & Object Mapper and Path Tracer \\
Interaction limits & -- & Up to 12 tool calls and 42 images\newline per trajectory (including input images) \\
\bottomrule
\end{tabularx}
\endgroup
\end{table*}

\clearpage
\endgroup

\begin{table*}[!h]
\caption{
\textbf{Method-specific training settings.}
All distillation-based methods use the same task-specific frozen teacher
checkpoint specified in \autoref{tab:shared_training_details} or \autoref{tab:code_agent_training_details}.
}
\label{tab:method_training_details}
\centering

\begingroup
\small
\setlength{\tabcolsep}{-2pt}
\renewcommand{\arraystretch}{1.25}

\begin{tabularx}{\linewidth}{
@{}
>{\raggedright\arraybackslash}p{0.170\linewidth}
>{\centering\arraybackslash}p{0.120\linewidth}
>{\raggedright\arraybackslash}X
>{\raggedright\arraybackslash}p{0.28\linewidth}
>{\raggedright\arraybackslash}p{0.28\linewidth}
@{}
}
\toprule

\textbf{Settings}
& \textbf{GRPO}
& \multicolumn{1}{c}{\textbf{OPD}}
& \multicolumn{1}{c}{\textbf{KDRL}}
& \multicolumn{1}{c}{\textbf{OPRD}} \\
\midrule

\rowcolor{blue!8}
\multicolumn{5}{@{}l}{\textbf{Optimization}} \\

Frozen teacher
& --
& Task-specific
& Task-specific
& Task-specific \\

Teacher reference
& --
& --
& --
& \oprdaddition{Each teacher's pre-RL checkpoint} \\

Teacher signal
& --
& Teacher--student\newline log-probability ratio
& K2 signal
& Teacher-shift direction $\mathbf d_t$ \\

Teacher temperature
& --
& $1.0$
& $1.0$
& $1.0$ \\

Token support
& --
& Sampled tokens
& Sampled tokens
& Sampled $\cup$ student top-10 \\

Coefficient schedule
& --
& Fixed at $1.0$
& $\beta_k: 0.005 \rightarrow 0$ over
\newline Math: 30 updates
\newline Reasoning Gym: 75 updates
\newline \oprdaddition{Code and Agent: 45 updates}
& $\lambda_+=0.5$ \oprdaddition{($1.0$ for Agent)}
\newline $\lambda_-: 0 \rightarrow \lambda_+$ over
\newline Math: 30 updates
\newline Reasoning Gym: 75 updates
\newline \oprdaddition{Code and Agent: 45 updates} \\

\bottomrule
\end{tabularx}

\endgroup
\end{table*}

\autoref{tab:evaluation_details} and \autoref{tab:code_agent_evaluation_details} give the shared evaluation settings for Math and Reasoning Gym, and for Code and Agent, respectively. All methods use the same task-specific settings. We report Mean@16 for Math, Pass@1 for Reasoning Gym, Mean@4 for Code, and Mean@1 rule scores for Agent. For $N$ questions with $m$ sampled responses each, Mean@$m$ averages all $Nm$ per-response benchmark scores. For binary scores, this estimates single-sample success probability (Pass@1) rather than the probability that any of $m$ attempts succeeds. All reported scores are multiplied by 100.
\looseness=-1

\begin{table*}[!h]
\caption{
\textbf{Default evaluation settings for Math and Reasoning Gym.}
Teacher and initial-student checkpoints use the same decoding and scoring protocols as trained student checkpoints.
}
\label{tab:evaluation_details}
\centering

\begingroup
\small
\setlength{\tabcolsep}{2pt}
\renewcommand{\arraystretch}{1.4}

\begin{tabularx}{\linewidth}{
@{}
>{\raggedright\arraybackslash}p{0.220\linewidth}
>{\raggedright\arraybackslash}p{0.340\linewidth}
>{\raggedright\arraybackslash}X
@{}
}
\toprule

\textbf{Settings}
& \multicolumn{1}{c}{\textbf{Math}}
& \multicolumn{1}{c}{\textbf{Reasoning Gym}} \\
\midrule

\rowcolor{blue!8}
\multicolumn{3}{@{}l}{\textbf{Benchmarks and Metrics}} \\

Reported benchmarks
& AIME\textquotesingle24, AIME\textquotesingle25,
HMMT\textquotesingle25, OlympiadBench
& Knights \& Knaves, Quantum Lock,\newline String Manipulation, Countdown
\newline (200 examples per task) 
\\

Evaluation metric
& Mean@16
& Pass@1 \\

Rollouts per problem
& 16
& 1 \\

\midrule

\rowcolor{blue!8}
\multicolumn{3}{@{}l}{\textbf{Decoding}} \\

Decoding parameters
& Temperature $0.7$, top-$p$ $0.8$, top-$k$ $20$
& Temperature $0.6$, top-$p$ $0.95$, top-$k$ $20$ \\

Maximum prompt length
& 2,048 tokens
& 2,048 tokens \\

Maximum response length
& 38,912 tokens
& 8,192 tokens \\

\midrule

\rowcolor{blue!8}
\multicolumn{3}{@{}l}{\textbf{Scoring and Reporting}} \\

Scoring
& Exact match after answer normalization
& Nonempty boxed answer required
\newline K\&K: exact match after normalization
\newline Quantum Lock:
\newline $1$ for a reference-length valid path,
\newline $0.5$ for any other valid path, $0$ otherwise
\newline String Manipulation: case-sensitive exact match
\newline Countdown: valid expression using each given number exactly once and reaching the target
\\

Checkpoint averaging
& 5-checkpoint mean (30-update intervals)
& 5-checkpoint mean (30-update intervals) \\

\bottomrule
\end{tabularx}

\endgroup
\end{table*}

\begingroup
\makeatletter
\setlength{\@fptop}{0pt}
\setlength{\@dblfptop}{0pt}
\makeatother

\begin{table*}[!t]
\caption{\oprdaddition{\textbf{Default evaluation settings for Code and Agent.} Teacher and initial-student checkpoints use the same decoding and scoring protocols as trained student checkpoints.}}
\label{tab:code_agent_evaluation_details}
\centering
\begingroup
\color{revisiongreen}
\small
\setlength{\tabcolsep}{2pt}
\renewcommand{\arraystretch}{1.4}
\begin{tabularx}{\linewidth}{
@{}
>{\raggedright\arraybackslash}p{0.220\linewidth}
>{\raggedright\arraybackslash}p{0.340\linewidth}
>{\raggedright\arraybackslash}X
@{}
}
\toprule
\textbf{Settings} & \multicolumn{1}{c}{\textbf{Code}} & \multicolumn{1}{c}{\textbf{Agent}} \\
\midrule
\rowcolor{blue!8}
\multicolumn{3}{@{}l}{\textbf{Benchmarks and Metrics}} \\
Reported benchmarks
& LCBv5, HumanEval+, CodeContests,\newline APPS Competition
& Maze, GPT4Scene, SR-91k, and Spatial VQA\newline
(holdout disjoint from RL training questions, images, and scenes; 4,111, 673, 818, and 4,862 questions, respectively) \\

Evaluation metric
& Mean@4
& Mean@1 (mean rule score) \\

Rollouts per problem
& 4
& 1 \\

\midrule
\rowcolor{blue!8}
\multicolumn{3}{@{}l}{\textbf{Decoding}} \\
Decoding parameters
& Temperature $0.7$, top-$p$ $0.8$, top-$k$ $20$
& Temperature $0.6$, top-$p$ $0.95$, top-$k$ $20$ \\

Maximum prompt length
& 2,048 tokens
& 8,192 tokens \\

Maximum response length
& 30,720 tokens
& 16,384 tokens per trajectory\newline (including tool observations) \\

Tool and image limits
& --
& Up to 12 tool calls and 42 images per trajectory (including input images)\newline
Action length: 4,096 tokens per model turn\newline
Observation length: 8,192 tokens per tool observation\newline
Image resolution: 3,136\textasciitilde{}200,704 pixels per image \\
\midrule
\rowcolor{blue!8}
\multicolumn{3}{@{}l}{\textbf{Scoring and Reporting}} \\
Scoring
& Python solution extracted from a single code block\newline
$1$ if all execution tests pass, $0$ otherwise\newline
HumanEval+: all original and additional tests must pass\newline
Failed or invalid solutions score $0$
& Exact answer-letter accuracy for multiple-choice answers\newline
Mean relative accuracy (MRA) for numeric answers\newline
Invalid or incomplete trajectories score~$0$ \\

Checkpoint averaging
& 5-checkpoint mean (30-update intervals)
& 5-checkpoint mean (45-update intervals) \\
\bottomrule
\end{tabularx}
\endgroup
\end{table*}
\clearpage
\endgroup

\clearpage
\section{Detailed Results for Weak-to-Strong Model Transfer}
\label{app:weak_to_strong_distill}

\subsection{Detailed Learning Curves}
\label{app:detailed_curve_successive_model}

We study controlled within-family transfer using Qwen3 4B-to-8B teacher--student pairs for mathematics and Reasoning Gym \citep{yang2025qwen3}\oprdaddition{, Qwen3-Base 8B-to-14B pairs for code generation, and Qwen2.5-VL 3B-to-7B pairs for agentic visual tool use \citep{Qwen2.5-VL}}. Teachers are post-trained with GRPO; for mathematics and Reasoning Gym, we select intermediate checkpoints whose performance exceeds the initial student's but remains below the student's final GRPO performance. \oprdaddition{Agent models start from cold-start SFT, and the teacher's pre-RL SFT checkpoint is the OPRD reference.}
\looseness=-1

We compare OPRD with GRPO, OPD, and KDRL on four mathematics benchmarks and four Reasoning Gym tasks, reporting Mean@16 and Pass@1, respectively. \autoref{fig:overview} shows aggregate performance, and \autoref{tab:main_results} reports five-checkpoint averages. \autoref{fig:successor_math_learning_curves} and \autoref{fig:successor_reasoning_learning_curves} give benchmark-level curves. OPRD retains OPD's rapid early improvement but continues beyond the teacher, where OPD plateaus, and reaches GRPO's final performance substantially earlier.
\looseness=-1


\begin{figure*}[!h]
\vspace{3pt}
\centering

\begingroup

\captionsetup[subfigure]{
  font=small,
  labelfont=normalfont,
  justification=centering,
  singlelinecheck=true,
  skip=0pt
}


\includegraphics[
    width=0.41\textwidth
]{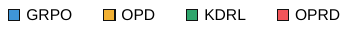}

\par\vspace{3pt}


\begin{subfigure}[t]{0.261\textwidth}
\centering
\includegraphics[
    width=\linewidth
]{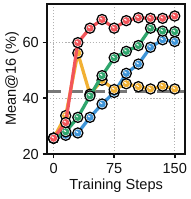}
\vspace{-10pt}
\caption{AIME\textquotesingle24}
\label{fig:successor_curve_aime24}
\end{subfigure}
\hfill
\begin{subfigure}[t]{0.235\textwidth}
\centering
\includegraphics[
    width=\linewidth
]{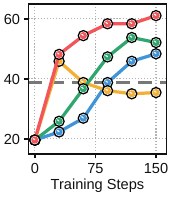}
\vspace{-10pt}
\caption{AIME\textquotesingle25}
\label{fig:successor_curve_aime25}
\end{subfigure}
\hfill
\begin{subfigure}[t]{0.235\textwidth}
\centering
\includegraphics[
    width=\linewidth
]{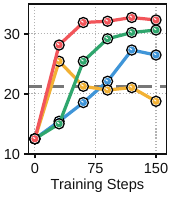}
\vspace{-10pt}
\caption{HMMT\textquotesingle25}
\label{fig:successor_curve_hmmt25}
\end{subfigure}
\hfill
\begin{subfigure}[t]{0.235\textwidth}
\centering
\includegraphics[
    width=\linewidth
]{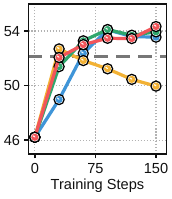}
\vspace{-10pt}
\caption{OlympiadBench}
\label{fig:successor_curve_olympiadbench}
\end{subfigure}

\caption{
\textbf{Learning curves for successive model transfer on individual math benchmarks.}
We use Qwen3-4B as the teacher and Qwen3-8B as the student. The gray dashed line denotes the performance of the weak teacher. All training settings follow the default Math configuration described in \autoref{app:exp_details}.
\looseness=-1
}
\label{fig:successor_math_learning_curves}

\endgroup
\vspace{2pt}
\end{figure*}

\begin{figure*}[!h]
\centering

\begingroup

\captionsetup[subfigure]{
  font=small,
  labelfont=normalfont,
  justification=centering,
  singlelinecheck=true,
  skip=0pt
}


\includegraphics[
    width=0.41\textwidth
]{assets/appendix/successor_learning_curve/legend.pdf}

\par\vspace{3pt}


\begin{subfigure}[t]{0.261\textwidth}
\centering
\includegraphics[
    width=\linewidth
]{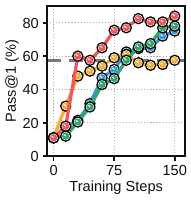}
\vspace{-10pt}
\caption{Knights \& Knaves}
\label{fig:successor_curve_kk}
\end{subfigure}
\hfill
\begin{subfigure}[t]{0.235\textwidth}
\centering
\includegraphics[
    width=\linewidth
]{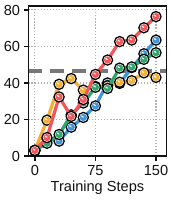}
\vspace{-10pt}
\caption{Quantum Lock}
\label{fig:successor_curve_quantum}
\end{subfigure}
\hfill
\begin{subfigure}[t]{0.235\textwidth}
\centering
\includegraphics[
    width=\linewidth
]{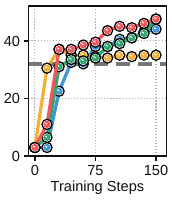}
\vspace{-10pt}
\caption{String Manipulation}
\label{fig:successor_curve_string}
\end{subfigure}
\hfill
\begin{subfigure}[t]{0.235\textwidth}
\centering
\includegraphics[
    width=\linewidth
]{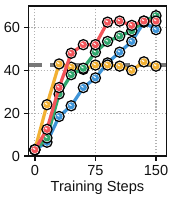}
\vspace{-10pt}
\caption{Countdown}
\label{fig:successor_curve_countdown}
\end{subfigure}

\caption{
\textbf{Learning curves for successive model transfer on individual reasoning benchmarks.}
We use Qwen3-4B-Base as the teacher and Qwen3-8B-Base as the student. The gray dashed line denotes the performance of the weak teacher. All training settings follow the default Reasoning Gym configuration described in \autoref{app:exp_details}.
}
\label{fig:successor_reasoning_learning_curves}

\endgroup
\end{figure*}

\clearpage

We also considered Qwen2.5-to-Qwen3 math transfer, but our preliminary Qwen2.5-3B and 7B teachers \citep{qwen2025qwen25technicalreport} scored only around 14\% on AIME\textquotesingle24, below our target range. Training a suitable Qwen2.5 teacher would take too long, so we use Qwen3 for mathematics. \oprdaddition{For Code, we choose Qwen3-14B-Base, the largest publicly released dense Qwen3 base model \citep{yang2025qwen3}. In preliminary runs with instruction-tuned Qwen3, we observed mixed thinking and non-thinking behavior with abrupt score jumps under OPRD. We therefore use Base models for Code to better isolate distillation gains.}
\looseness=-1

\oprdaddition{\autoref{fig:code_results} shows the four Code learning curves from updates 0 to 150. OPRD's checkpoint average is 40.13, exceeding OPD's 38.46; by update 90 it leads all trained baselines on every code benchmark. \autoref{fig:agent_results} shows the four Agent learning curves from updates 0 to 225 on a fixed evaluation set disjoint from RL training data. Its checkpoint average is 67.50, compared with 65.57 for OPD. At update 225, OPRD reaches 87.81 on Maze, 69.44 on GPT4Scene, 52.07 on SR-91k, and 70.80 on Spatial VQA, exceeding the teacher on every family. The four-family average is 70.03, compared with 67.74 for the teacher.}
\looseness=-1

\begin{figure*}[!h]
\vspace{1pt}
\centering

\begingroup
\captionsetup[subfigure]{
  font=small,
  labelfont=normalfont,
  justification=centering,
  singlelinecheck=true,
  skip=0pt
}

\includegraphics[width=0.41\textwidth]{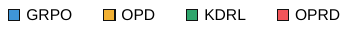}
\par\vspace{3pt}

\begin{subfigure}[t]{0.261\textwidth}
\centering
\includegraphics[width=\linewidth]{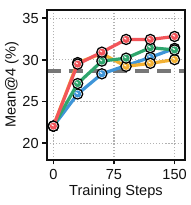}
\vspace{-10pt}
\caption{\oprdaddition{LCBv5}}
\label{fig:code_results_lcbv5}
\end{subfigure}
\hfill
\begin{subfigure}[t]{0.235\textwidth}
\centering
\includegraphics[width=\linewidth]{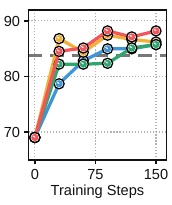}
\vspace{-10pt}
\caption{\oprdaddition{HumanEval+}}
\label{fig:code_results_humanevalplus}
\end{subfigure}
\hfill
\begin{subfigure}[t]{0.235\textwidth}
\centering
\includegraphics[width=\linewidth]{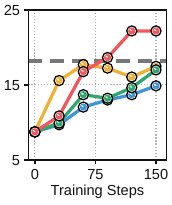}
\vspace{-10pt}
\caption{\oprdaddition{CodeContests}}
\label{fig:code_results_codecontests}
\end{subfigure}
\hfill
\begin{subfigure}[t]{0.235\textwidth}
\centering
\includegraphics[width=\linewidth]{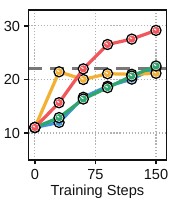}
\vspace{-10pt}
\caption{\oprdaddition{APPS}}
\label{fig:code_results_apps}
\end{subfigure}

\endgroup

\caption{\oprdaddition{\textbf{Learning curves on individual code-generation benchmarks.} We use a Qwen3-8B-Base teacher from GRPO step 435 and a Qwen3-14B-Base student. The gray dashed line marks the fixed teacher's Mean@4, and all methods start from the same initial student. Other settings follow the Code configurations in \appautoref{app:exp_details}.}}
\label{fig:code_results}
\vspace{8pt}
\end{figure*}

\begin{figure*}[!h]
\vspace{1pt}
\centering

\begingroup
\captionsetup[subfigure]{
  font=small,
  labelfont=normalfont,
  justification=centering,
  singlelinecheck=true,
  skip=0pt
}

\includegraphics[width=0.41\textwidth]{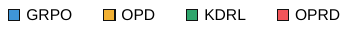}
\par\vspace{3pt}

\begin{subfigure}[t]{0.261\textwidth}
\centering
\includegraphics[width=\linewidth]{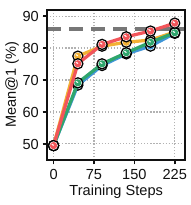}
\vspace{-10pt}
\caption{\oprdaddition{Maze}}
\label{fig:agent_results_maze}
\end{subfigure}
\hfill
\begin{subfigure}[t]{0.235\textwidth}
\centering
\includegraphics[width=\linewidth]{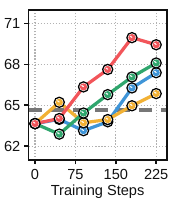}
\vspace{-10pt}
\caption{\oprdaddition{GPT4Scene}}
\label{fig:agent_results_gpt4scene}
\end{subfigure}
\hfill
\begin{subfigure}[t]{0.235\textwidth}
\centering
\includegraphics[width=\linewidth]{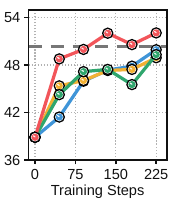}
\vspace{-10pt}
\caption{\oprdaddition{SR-91k}}
\label{fig:agent_results_sr91k}
\end{subfigure}
\hfill
\begin{subfigure}[t]{0.235\textwidth}
\centering
\includegraphics[width=\linewidth]{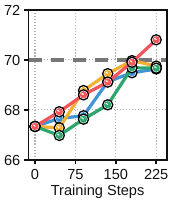}
\vspace{-10pt}
\caption{\oprdaddition{Spatial VQA}}
\label{fig:agent_results_spatial_vqa}
\end{subfigure}

\endgroup

\caption{\oprdaddition{\textbf{Learning curves on individual agentic visual tool use benchmarks.} We use a Qwen2.5-VL-3B SFT teacher from GRPO step 450 and a Qwen2.5-VL-7B SFT student. The gray dashed line marks the fixed teacher's Mean@1, and all methods start from the same initial student. Other settings follow the Agent configurations in \appautoref{app:exp_details}.}}
\label{fig:agent_results}
\vspace{8pt}
\end{figure*}

\clearpage
\subsection{Additional Results Across Model Variants and Tasks}
\label{app:additional_successive_model}

The instruction-tuned Qwen3 results reported in \secautoref{app:detailed_curve_successive_model} exhibit a potential response-length confound. Although thinking mode is disabled, longer responses may implicitly elicit some of the reasoning behavior associated with that mode, leading to abrupt, transient score gains. In \autoref{fig:successor_math_learning_curves}, for example, OPD briefly surpasses the teacher on both AIME\textquotesingle24 and AIME\textquotesingle25 at step 30 before returning toward a teacher-level plateau. Such behavior can confound comparisons of early learning speed. We therefore evaluate Qwen3-Base models, for which this effect is less pronounced, in \autoref{fig:successor_variant_math_learning_curves_math}. OPRD again substantially accelerates weak-to-strong generalization, achieving high performance much earlier than the baselines on all four benchmarks.
\looseness=-1


\begin{figure*}[!h]
\vspace{1pt}
\centering

\begingroup

\captionsetup[subfigure]{
  font=small,
  labelfont=normalfont,
  justification=centering,
  singlelinecheck=true,
  skip=0pt
}

\includegraphics[
    width=0.41\textwidth
]{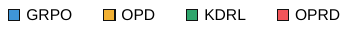}

\par\vspace{3pt}

\begin{subfigure}[t]{0.261\textwidth}
\centering
\includegraphics[
    width=\linewidth
]{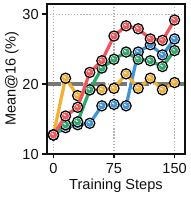}
\vspace{-10pt}
\caption{AIME\textquotesingle24}
\label{fig:successor_variant_aime24}
\end{subfigure}
\hfill
\begin{subfigure}[t]{0.235\textwidth}
\centering
\includegraphics[
    width=\linewidth
]{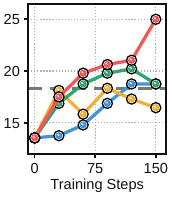}
\vspace{-10pt}
\caption{AIME\textquotesingle25}
\label{fig:successor_variant_aime25}
\end{subfigure}
\hfill
\begin{subfigure}[t]{0.235\textwidth}
\centering
\includegraphics[
    width=\linewidth
]{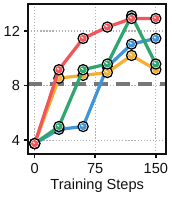}
\vspace{-10pt}
\caption{HMMT\textquotesingle25}
\label{fig:successor_variant_hmmt25}
\end{subfigure}
\hfill
\begin{subfigure}[t]{0.235\textwidth}
\centering
\includegraphics[
    width=\linewidth
]{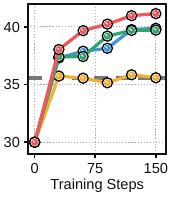}
\vspace{-10pt}
\caption{OlympiadBench}
\label{fig:successor_variant_olympiadbench}
\end{subfigure}

\endgroup

\caption{
\textbf{Learning curves on individual math benchmarks using Qwen3-Base models.}
We use Qwen3-4B-Base as the teacher and Qwen3-8B-Base as the student. The gray dashed line denotes the performance of the weak teacher.
All other settings follow the default Math configuration in \autoref{app:exp_details}, but we omit the system prompt and reduce the mini-batch size to 32, yielding two optimizer steps per training batch.
\looseness=-1
}
\label{fig:successor_variant_math_learning_curves_math}
\vspace{8pt}
\end{figure*}

Reward gains often conincide with longer responses. For instruction-tuned Qwen3, this makes a potential confound: distillation gains may simply reflect longer responses eliciting latent thinking behavior.
To test whether OPRD depends on this effect, we evaluate three more reasoning tasks in \autoref{fig:successor_variant_reasoning_learning_curves_reasoning}, where, as in String Manipulation, post-training shortens responses by a factor of three to four relative to the raw checkpoints. OPRD still improves substantially faster than the baselines, quickly reaching GRPO's eventual plateau while reducing response length. This opposite trend shows that its gains are not tied to response-length growth. OPRD's gradient scaling can nevertheless magnify length bias in the teacher-shift signal, as discussed in \secautoref{subsec:challenges_discussion}, \appautoref{app:asymetric_alignment}, and \appautoref{app:length_bias}. We resolve this by using a later teacher checkpoint, rather than the raw model, as the reference policy.
\looseness=-1

\begin{figure*}[!h]
\centering

\begingroup

\captionsetup[subfigure]{
  font=small,
  labelfont=normalfont,
  justification=centering,
  singlelinecheck=true,
  skip=0pt
}

\includegraphics[
    width=0.41\textwidth
]{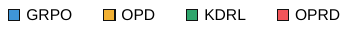}

\par\vspace{3pt}

\begin{subfigure}[t]{0.325\textwidth}
\centering
\includegraphics[
    width=\linewidth
]{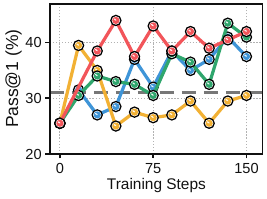}
\vspace{-10pt}
\caption{Zebra Puzzles}
\label{fig:successor_variant_zebra}
\end{subfigure}
\hfill
\begin{subfigure}[t]{0.30\textwidth}
\centering
\includegraphics[
    width=\linewidth
]{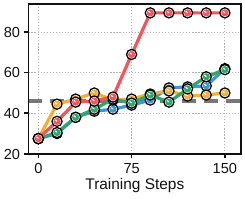}
\vspace{-10pt}
\caption{Color Cube Rotation}
\label{fig:successor_variant_color_cube}
\end{subfigure}
\hfill
\begin{subfigure}[t]{0.30\textwidth}
\centering
\includegraphics[
    width=\linewidth
]{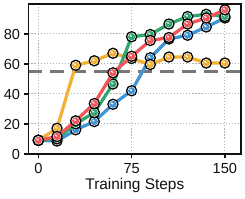}
\vspace{-10pt}
\caption{Binary Matrix}
\label{fig:successor_variant_binary_matrix}
\end{subfigure}

\endgroup

\caption{
\textbf{Learning curves on additional three Reasoning Gym tasks.}
We use Qwen3-4B-Base as the teacher and Qwen3-8B-Base as the student. The gray dashed line denotes the performance of the weak teacher.
For Color Cube Rotation and Binary Matrix, we use the teacher checkpoints from steps 30 and 45, respectively, as the reference policies instead of the raw step-0 models to mitigate length bias (see \autoref{app:length_bias} for details).
\looseness=-1
}
\label{fig:successor_variant_reasoning_learning_curves_reasoning}

\end{figure*}

\clearpage

To complement the detailed learning curves, \autoref{tab:app_main_results} reports checkpoint-averaged results, providing a numerical summary of how quickly each method reaches high performance. OPRD again achieves the strongest results, confirming that it accelerates weak-to-strong generalization across these additional settings.


\begin{table*}[!h]
\vspace{3pt}
\caption{
\textbf{Additional results for successive model transfer with Qwen3-Base models on mathematics and three additional Reasoning Gym tasks.}
We report Mean@16 for Math and Pass@1 for Reasoning Gym. We use intermediate GRPO checkpoints of Qwen3-4B-Base as teachers for Qwen3-8B-Base students.
The teacher and initial-student rows report fixed-checkpoint performance, whereas each trained-policy row averages evaluations at steps 30, 60, 90, 120, and 150. The corresponding learning curves are shown in \autoref{fig:successor_variant_math_learning_curves_math} and \autoref{fig:successor_variant_reasoning_learning_curves_reasoning}. The best trained-policy result in each column is shown in \textbf{bold}.
}
\label{tab:app_main_results}
\centering

\begingroup
\small
\setlength{\tabcolsep}{1.8pt}
\renewcommand{\arraystretch}{1.05}

\begin{tabularx}{\linewidth}{
@{}
l
*{5}{
  >{\hsize=0.92\hsize
    \linewidth=\hsize
    \centering\arraybackslash}X
}
*{4}{
  >{\hsize=1.10\hsize
    \linewidth=\hsize
    \centering\arraybackslash}X
}
@{}
}
\toprule


&
\multicolumn{5}{c}{\textbf{Math Reasoning}}
&
\multicolumn{4}{c}{\textbf{Reasoning Gym}} \\
\cmidrule(lr){2-6}
\cmidrule(lr){7-10}

\textbf{Policy}
& AIME\textquotesingle24
& AIME\textquotesingle25
& HMMT\textquotesingle25
& Olympiad
& \textbf{Avg.}
& Zebra
& Color
& Binary
& \textbf{Avg.} \\
\midrule


\rowcolor{blue!8}
&
\multicolumn{5}{c}{
\makebox[0pt][c]{%
\footnotesize
\textbf{Qwen3-4B-Base (Teacher)}
$\rightarrow$
\textbf{Qwen3-8B-Base (Student)}%
}
}
&
\multicolumn{4}{c}{
\makebox[0pt][c]{%
\footnotesize
\textbf{Qwen3-4B-Base (Teacher)}
$\rightarrow$
\textbf{Qwen3-8B-Base (Student)}%
}
} \\
\midrule


Teacher
& 20.00
& 18.33
& \,\,\,8.13
& 35.55
& 20.50
& 31.00
& 46.00
& 55.00
& 44.00 \\

\midrule

Student
& 12.71
& 13.54
& \,\,\,3.75
& 30.01
& 15.00
& 25.50
& 27.50
& \,\,\,9.00
& 20.67 \\


+ GRPO
& 20.00
& 16.58
& \,\,\,8.33
& 38.59
& 20.88
& 35.40
& 48.30
& 56.50
& 46.73 \\

+ OPD
& 20.00
& 17.21
& \,\,\,9.13
& 35.58
& 20.48
& 29.10
& 48.30
& 62.10
& 46.50 \\

+ KDRL
& 21.92
& 18.88
& \,\,\,9.29
& 38.69
& 22.19
& 35.60
& 49.50
& 61.40
& 48.83 \\

\rowcolor{gray!15}
+ \textbf{OPRD}
& \textbf{24.79}
& \textbf{20.79}
& \textbf{11.75}
& \textbf{40.01}
& \textbf{24.34}
& \textbf{39.10}
& \textbf{72.40}
& \textbf{66.80}
& \textbf{59.43} \\

\bottomrule
\end{tabularx}

\endgroup
\vspace{3pt}
\end{table*}

\subsection{Evaluation Results with Standard Deviations}
\label{app:standard_deviation}

To assess the evaluation-time robustness of the \oprdaddition{Math and Reasoning Gym} comparisons in \autoref{tab:main_results}, we report response-resampling variability in \autoref{tab:std_weak_to_strong}. Because multi-seed post-training is prohibitively expensive, we hold the benchmark problems and trained checkpoints fixed and resample only their responses. Each of 1{,}000 replicates draws 16 responses with replacement from a pool of 32 per Math problem and one from a pool of eight per Reasoning Gym problem. Trained-policy results are averaged over five checkpoints within each replicate, and we report the resulting mean and sample standard deviation. OPRD still surpasses the strongest baseline by approximately 8.0 points on Math and 9.9 points on Reasoning Gym, margins far exceeding the observed response-resampling variability.
\looseness=-1


\begin{table*}[!h]
\caption{
\textbf{Evaluation results with standard deviations for successive model transfer on mathematics and reasoning tasks.}
We use intermediate GRPO checkpoints of 4B-scale Qwen3 models as teachers and report Mean@16 for mathematics and Pass@1 for Reasoning Gym as the bootstrap mean $\pm$ sample standard deviation over 1{,}000 response-resampled evaluations, with benchmark items held fixed. In each bootstrap replicate, we sample 16 responses with replacement from a pool of 32 for each mathematics problem and one response from a pool of eight for each Reasoning Gym problem.
\looseness=-1
}
\label{tab:std_weak_to_strong}
\centering

\begingroup
\small

\setlength{\tabcolsep}{2.0pt}
\renewcommand{\arraystretch}{1.25}
\renewcommand{\tabularxcolumn}[1]{m{#1}}

\definecolor{teacherheader}{RGB}{235,235,250}
\definecolor{stdred}{RGB}{185,45,65}

\newcommand{\stdresult}[2]{%
  \shortstack[c]{%
    #1\\[0pt]
    {\scriptsize\textcolor{stdred}{$\pm$\,#2}}%
  }%
}

\begin{tabularx}{\linewidth}{
@{}
>{\raggedright\arraybackslash}m{0.078\linewidth}
*{10}{>{\centering\arraybackslash}X}
@{}
}

\toprule

&
\multicolumn{5}{c}{\textbf{Math Reasoning}}
&
\multicolumn{5}{c}{\textbf{Reasoning Gym}}
\\

\cmidrule(lr){2-6}
\cmidrule(lr){7-11}

\textbf{Policy}
&
AIME\textquotesingle24
&
AIME\textquotesingle25
&
HMMT\textquotesingle25
&
Olympiad
&
\textbf{Avg.}
&
Knights
&
Quantum
&
String
&
Count
&
\textbf{Avg.}
\\

\midrule

\rowcolor{teacherheader}
\multicolumn{6}{c}{
  {\footnotesize\textbf{
    Qwen3-4B (Teacher) $\rightarrow$ Qwen3-8B (Student)
  }}
}
&
\multicolumn{5}{c}{
  {\footnotesize\textbf{
    Qwen3-4B-Base (Teacher) $\rightarrow$ Qwen3-8B-Base (Student)
  }}
}
\\

\midrule

Teacher
&
\stdresult{41.77}{1.48}
&
\stdresult{36.77}{1.38}
&
\stdresult{21.44}{1.15}
&
\stdresult{52.02}{0.23}
&
\stdresult{38.00}{0.58}
&
\stdresult{54.92}{2.93}
&
\stdresult{41.64}{2.73}
&
\stdresult{33.78}{1.17}
&
\stdresult{42.13}{1.57}
&
\stdresult{43.12}{1.11}
\\

\midrule

Student
&
\stdresult{24.25}{1.09}
&
\stdresult{19.84}{1.09}
&
\stdresult{13.04}{0.95}
&
\stdresult{46.21}{0.24}
&
\stdresult{25.83}{0.45}
&
\stdresult{11.71}{1.96}
&
\stdresult{5.00}{1.35}
&
\stdresult{3.61}{1.12}
&
\stdresult{2.83}{1.05}
&
\stdresult{5.79}{0.70}
\\

\midrule

+ GRPO
&
\stdresult{46.88}{0.63}
&
\stdresult{36.61}{0.55}
&
\stdresult{22.47}{0.51}
&
\stdresult{52.48}{0.10}
&
\stdresult{39.61}{0.25}
&
\stdresult{53.17}{1.21}
&
\stdresult{34.70}{0.93}
&
\stdresult{35.51}{0.59}
&
\stdresult{41.95}{0.64}
&
\stdresult{41.33}{0.45}
\\

+ OPD
&
\stdresult{46.62}{0.74}
&
\stdresult{37.91}{0.59}
&
\stdresult{21.40}{0.47}
&
\stdresult{51.30}{0.10}
&
\stdresult{39.31}{0.27}
&
\stdresult{55.78}{1.31}
&
\stdresult{39.91}{1.19}
&
\stdresult{34.55}{0.73}
&
\stdresult{42.11}{0.78}
&
\stdresult{43.09}{0.51}
\\

+ KDRL
&
\stdresult{54.20}{0.63}
&
\stdresult{42.80}{0.59}
&
\stdresult{25.66}{0.55}
&
\stdresult{\textbf{53.39}}{0.10}
&
\stdresult{44.01}{0.26}
&
\stdresult{53.91}{1.14}
&
\stdresult{37.83}{1.12}
&
\stdresult{38.10}{0.58}
&
\stdresult{49.54}{0.75}
&
\stdresult{44.84}{0.48}
\\

\rowcolor{gray!15}
\textbf{+ OPRD}
&
\stdresult{\textbf{67.40}}{0.63}
&
\stdresult{\textbf{55.69}}{0.64}
&
\stdresult{\textbf{31.63}}{0.58}
&
\stdresult{\textbf{53.28}}{0.09}
&
\stdresult{\textbf{52.00}}{0.27}
&
\stdresult{\textbf{72.02}}{0.94}
&
\stdresult{\textbf{49.70}}{1.02}
&
\stdresult{\textbf{42.11}}{0.57}
&
\stdresult{\textbf{55.30}}{0.75}
&
\stdresult{\textbf{54.78}}{0.42}
\\

\bottomrule

\end{tabularx}

\endgroup
\end{table*}

\clearpage
\section{Detailed Results for Multi-Teacher Weak-to-Strong Distillation}
\label{app:multiteacher_distillation}

We follow the single-teacher Reasoning Gym setting but jointly train one student on domain-mixed batches, pairing each example with its task-specific teacher. \autoref{fig:multiteacher_learning_curves} shows the per-domain learning curves. Despite heterogeneous task structures and response-length trends—String Manipulation responses shorten as reward improves, whereas those for the other tasks generally lengthen—OPRD accelerates learning and attains the highest Pass@1 across all four domains. MOPD shows signs of cross-task interference, most notably on Quantum Lock, where it falls below the corresponding specialist, while OPRD rapidly transfers the specialist capabilities and continues improving without comparable degradation.
\looseness=-1


\begin{figure*}[!h]
\vspace{-2pt}
\centering

\begingroup

\captionsetup[subfigure]{
  font=small,
  labelfont=normalfont,
  justification=centering,
  singlelinecheck=true,
  skip=0pt
}

\includegraphics[
    width=0.41\textwidth
]{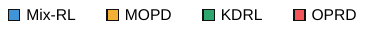}

\par\vspace{3pt}

\begin{subfigure}[t]{0.261\textwidth}
\centering
\includegraphics[
    width=\linewidth
]{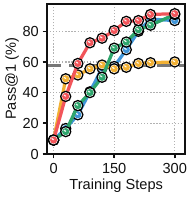}
\vspace{-10pt}
\caption{Knights \& Knaves}
\label{fig:multiteacher_curve_kk}
\end{subfigure}
\hfill
\begin{subfigure}[t]{0.235\textwidth}
\centering
\includegraphics[
    width=\linewidth
]{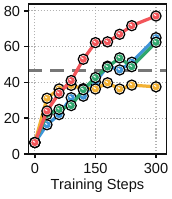}
\vspace{-10pt}
\caption{Quantum Lock}
\label{fig:multiteacher_curve_quantum}
\end{subfigure}
\hfill
\begin{subfigure}[t]{0.235\textwidth}
\centering
\includegraphics[
    width=\linewidth
]{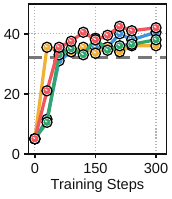}
\vspace{-10pt}
\caption{String Manipulation}
\label{fig:multiteacher_curve_string}
\end{subfigure}
\hfill
\begin{subfigure}[t]{0.235\textwidth}
\centering
\includegraphics[
    width=\linewidth
]{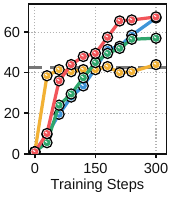}
\vspace{-10pt}
\caption{Countdown}
\label{fig:multiteacher_curve_countdown}
\end{subfigure}

\endgroup

\caption{
\textbf{Learning curves on individual Reasoning Gym tasks under multi-teacher distillation.}
We use four task-specific Qwen3-4B-Base models as teachers and jointly train a Qwen3-8B-Base student. The gray dashed line denotes the performance of the corresponding specialist teacher. All other settings follow the configurations described in \secautoref{app:exp_details}.
\looseness=-1
}
\label{fig:multiteacher_learning_curves}
\vspace{-2pt}
\end{figure*}

\section{Detailed Results for Strong-to-Weak Distillation}
\label{app:strong_to_weak_distillation}

\autoref{fig:strong_to_weak_app} presents detailed learning curves for strong-to-weak settings. On AIME\textquotesingle24, OPRD raises the Qwen3-1.7B initial student's Mean@16 from 10.0 to above 41 within 150 updates, whereas GRPO reaches only about 25 at the same point and 35 even after 240 updates. On Knights \& Knaves, OPD improves initially but collapses midway through training and remains below the teacher after recovering. In contrast, OPRD rapidly improves the Qwen3-0.6B student and ultimately surpasses the Qwen3-8B-Base teacher's Pass@1 of 64.5. 
\looseness=-1


\begin{figure*}[!h]
\centering

\begingroup

\captionsetup[subfigure]{
  font=small,
  labelfont=normalfont,
  justification=centering,
  singlelinecheck=true,
  skip=3pt
}


\includegraphics[
    width=0.4\textwidth,
    keepaspectratio
]{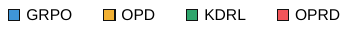}

\par\vspace{3pt}


\begin{subfigure}[t]{0.34\linewidth}
\centering

\includegraphics[
    width=\linewidth,
    keepaspectratio
]{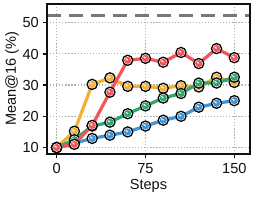}

\caption{AIME\textquotesingle24}
\label{fig:strong_to_weak_aime24}
\end{subfigure}
\hspace{0.04\linewidth}
%
\begin{subfigure}[t]{0.34\linewidth}
\centering

\includegraphics[
    width=\linewidth,
    keepaspectratio
]{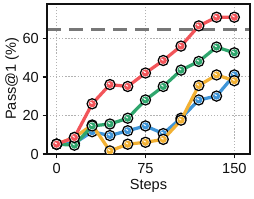}
\caption{Knights \& Knaves}
\label{fig:strong_to_weak_knights_knaves}
\end{subfigure}

\caption{
\textbf{Learning curves for strong-to-weak distillation on two tasks.}
We evaluate Qwen3-8B $\rightarrow$ Qwen3-1.7B on AIME\textquotesingle24 and Qwen3-8B-Base $\rightarrow$ Qwen3-0.6B on Knights \& Knaves, using teachers from step 105 of task-specific GRPO. Gray dashed lines mark teacher performance. All other settings follow \autoref{app:exp_details}, with method-specific distillation coefficients scheduled over the first 45 updates.
\looseness=-1
}
\label{fig:strong_to_weak_app}

\endgroup
\end{figure*}

\clearpage
\section{Detailed Results for Weak-to-Strong Method Comparisons}

\subsection{Baseline Implementation Details}
\label{subsec:baseline_implementation}

Under the default training and evaluation configurations in \appautoref{app:exp_details}, all baselines use the same model pairs, teacher checkpoints, prompt formats, and evaluation protocols as OPRD unless otherwise noted. We describe only their method-specific settings below.

\vspace{-5pt}
\begin{itemize}[leftmargin=*, itemsep=2pt]
    \item \textbf{W2SR-P} \citep{yuan-etal-2026-incentivizing}. 
    We reproduce the seeded prompt stream used by the 150-update RL runs, yielding $150\times64=9{,}600$ prompt occurrences. For each occurrence, we sample eight responses from the weak teacher and select one verifier-correct, format-valid, non-truncated response, discarding occurrences with no valid candidate. We then fully fine-tune the initial student checkpoint for three epochs using next-token prediction with a global batch size of 64 and a learning rate of $2\times10^{-5}$.
    \looseness=-1

    \vspace{3pt}
    \item \textbf{S2L-PO} \citep{ren2026smaller}. S2L-PO linearly anneals the fraction of weak-model rollouts over the first half of GRPO training. Although the original method advocates using a smaller \textit{base} model as the weak explorer to exploit its policy-level diversity, we use the same post-RL weak teacher as the other baselines for a controlled comparison. While the original implementation uses 16 rollouts per prompt, we retain its 16-phase schedule with the default group size of eight. Over 150 updates, the weak/student composition transitions from $8/0$ to $0/8$ during the first eight phases (updates 1--75) and remains at $0/8$ during the remaining eight phases (updates 76--150). For each trajectory, we compute the importance ratio using its generating policy as $\pi_{\mathrm{rollout}}$, namely $\pi_T$ for weak-teacher rollouts and $\pi_{\theta_{\mathrm{old}}}$ for student rollouts. We also retain the original KL regularization toward the initial student with a coefficient of $10^{-3}$.
    \looseness=-1

    \vspace{3pt}
    \item \textbf{OPSD} \citep{zhao2026selfdistilled}. OPSD is originally a self-distillation method that uses a correct self-generated rollout as privileged information. To adapt it to our weak-to-strong setting, we instead use a verifier-correct weak-teacher rollout as privileged information, falling back to a correct student rollout when the weak teacher produces none. An EMA copy of the student serves as the self-teacher, conditioning on the privileged rollout to provide distillation targets for the original student trajectories and being updated after each step with a rate of $0.05$. Whenever valid privileged information is available, we apply the distillation loss to all student trajectories in the group, regardless of whether they are correct or incorrect. We use generalized JSD with $\alpha=0.5$ over the top-100 student tokens and an additional tail bucket.
    \looseness=-1

    \vspace{3pt}
    \item \textbf{Direct-OPD} \citep{feng2026weak}. Developed concurrently with OPRD, Direct-OPD optimizes the weak policy shift as a dense reward on student-generated trajectories:
    \[
    \mathcal{J}_{\mathrm{Direct\text{-}OPD}}=\mathbb{E}_{x,\,y\sim\pi_\theta}\!\left[\sum_t\!\left(\log\pi_T(y_t\mid s_t)-\log\pi_T^{\mathrm{ref}}(y_t\mid s_t)\right)\right]-\alpha D_{\mathrm{KL}}\!\left(\pi_\theta\,\|\,\pi_{S,0}\right).
    \]
    Following the original implementation, we evaluate the dense reward over the top-16 tokens of the old student policy at each visited state and use the reported hyperparameters. The policy-shift scale is fixed at $1$, while $\alpha$, the coefficient of the KL anchor toward $\pi_{S,0}$, is initialized at $2.5$. Before each actor update, $\alpha$ is multiplied by $1.01$ or $0.99$ depending on whether the batch-mean dense reward is positive or negative, respectively, and clipped to $[0.5,2.5]$. This KL anchor is computed separately on the sampled response tokens using the low-variance k3 estimator.
    \looseness=-1

    \vspace{3pt}
    \item \textbf{W2S-OPD} \citep{yu2026weak}. 
    W2S-OPD reanchors the weak policy shift to the initial student by defining the proxy teacher as
    \[
    \pi_{\mathrm{proxy}}(v\mid s_t)\propto \pi_{S,0}(v\mid s_t)\left(\frac{\pi_T(v\mid s_t)}{\pi_T^{\mathrm{ref}}(v\mid s_t)}\right)^\gamma.
    \]
    Following the original implementation, we set $\gamma=1$ and compute the proxy scores over the full vocabulary before selecting the proxy's top-32 tokens. We normalize both the proxy and current-student distributions over this proxy-selected support and minimize the reverse KL from the current student to the proxy. This restricted-support reverse KL serves as the sole actor objective, with no additional KL anchor or adaptive coefficient.
    \looseness=-1
    
\end{itemize}

\clearpage
\subsection{Detailed Learning Curve}
\label{subsec:baseline_learning_curve}

In \autoref{fig:weak_to_strong_offpolicy_comparison}, we compare OPRD with three methods that leverage off-policy generations from the weak teacher. OPRD exhibits the strongest and most consistent gains overall. Consistent with \citet{yuan-etal-2026-incentivizing}, W2SR-P shows that SFT on verifier-correct teacher rollouts can move the student slightly beyond weak-teacher performance. S2L-PO \citep{ren2026smaller} remains competitive on the two Reasoning Gym tasks, although its AIME\textquotesingle24 performance deteriorates after weak-teacher rollouts are fully annealed out at update 75 and its checkpoint-averaged performance remains below OPRD. Our OPSD variant \citep{zhao2026selfdistilled} performs poorly whether the privileged trace is self-generated or supplied by the weak teacher. This behavior is consistent with recent findings that privileged self-distillation can impair thinking models by shortening or suppressing deliberative reasoning \citep{kim2026does, kaur2026rethinking}. Accordingly, OPSD provides a modest benefit only on String Manipulation, where higher rewards coincide with shorter reasoning traces, and fails to deliver competitive gains on the other tasks.
\looseness=-1


\begin{figure*}[!h]
\centering

\begingroup

\captionsetup[subfigure]{
  font=small,
  labelfont=normalfont,
  justification=centering,
  singlelinecheck=true,
  skip=6pt
}


\includegraphics[
    width=0.65\textwidth,
    keepaspectratio
]{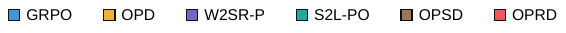}

\par\vspace{3pt}


\begin{subfigure}[t]{0.32\linewidth}
\centering

\includegraphics[
    width=\linewidth,
    keepaspectratio
]{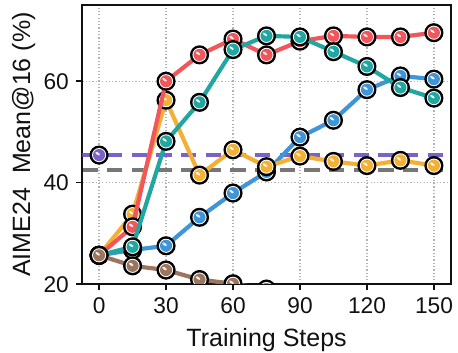}

\caption{AIME\textquotesingle24}
\label{fig:weak_to_strong_offpolicy_aime24}
\end{subfigure}%
\hfill
\begin{subfigure}[t]{0.32\linewidth}
\centering

\includegraphics[
    width=\linewidth,
    keepaspectratio
]{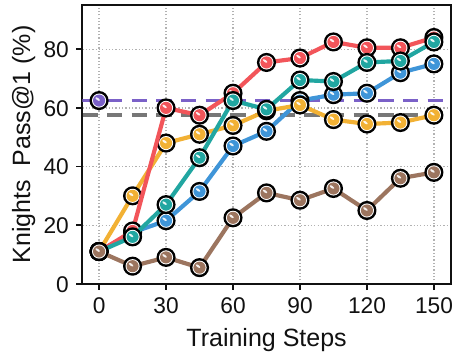}

\caption{Knights \& Knaves}
\label{fig:weak_to_strong_offpolicy_knights}
\end{subfigure}%
\hfill
\begin{subfigure}[t]{0.32\linewidth}
\centering

\includegraphics[
    width=\linewidth,
    keepaspectratio
]{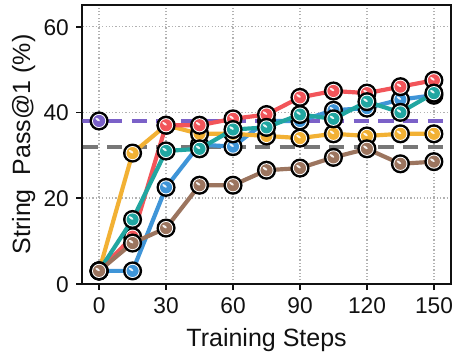}

\caption{String Manipulation}
\label{fig:weak_to_strong_offpolicy_string}
\end{subfigure}

\caption{
\textbf{Training dynamics of weak-to-strong methods using off-policy generations from the weak teacher.}
Because W2SR-P performs SFT without subsequent RL, its final performance is shown as a horizontal dashed line. The gray dashed lines indicate weak-teacher performance. See \appautoref{subsec:baseline_implementation} for baseline implementation details.
\looseness=-1
}
\label{fig:weak_to_strong_offpolicy_comparison}

\endgroup
\end{figure*}

In \autoref{fig:weak_to_strong_delta_comparison}, we further compare OPRD with Direct-OPD \citep{feng2026weak} and W2S-OPD \citep{yu2026weak}, two concurrent methods that likewise exploit the weak policy delta. Although these methods use the same transferred signal, their objectives are defined directly by the delta and therefore receive no independent verifier-driven update direction. W2S-OPD can surpass the weak teacher, but ultimately plateaus near teacher-level performance because its optimization target remains restricted to the policy changes encoded by the weak teacher. OPRD instead uses the delta only to identify and rescale the component of the verifier gradient aligned with the weak shift, while preserving the orthogonal component $\mathbf g_t^\perp$. Consequently, the delta guides rather than replaces verifier-driven optimization, allowing OPRD to improve beyond teacher-level saturation and achieve the strongest final performance across all three tasks.
\looseness=-1


\begin{figure*}[!h]
\centering

\begingroup

\captionsetup[subfigure]{
  font=small,
  labelfont=normalfont,
  justification=centering,
  singlelinecheck=true,
  skip=6pt
}


\includegraphics[
    width=0.58\textwidth,
    keepaspectratio
]{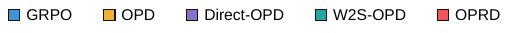}

\par\vspace{3pt}


\begin{subfigure}[t]{0.32\linewidth}
\centering

\includegraphics[
    width=\linewidth,
    keepaspectratio
]{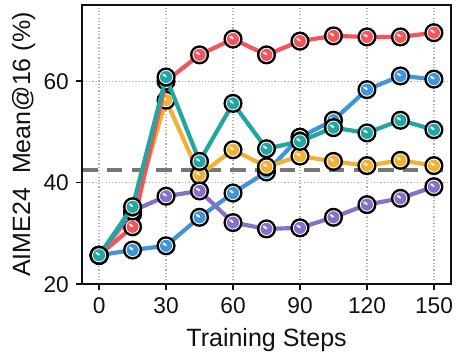}

\caption{AIME\textquotesingle24}
\label{fig:weak_to_strong_delta_aime24}
\end{subfigure}%
\hfill
\begin{subfigure}[t]{0.32\linewidth}
\centering

\includegraphics[
    width=\linewidth,
    keepaspectratio
]{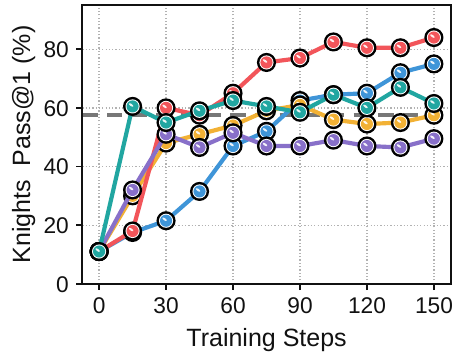}

\caption{Knights \& Knaves}
\label{fig:weak_to_strong_delta_knights}
\end{subfigure}%
\hfill
\begin{subfigure}[t]{0.32\linewidth}
\centering

\includegraphics[
    width=\linewidth,
    keepaspectratio
]{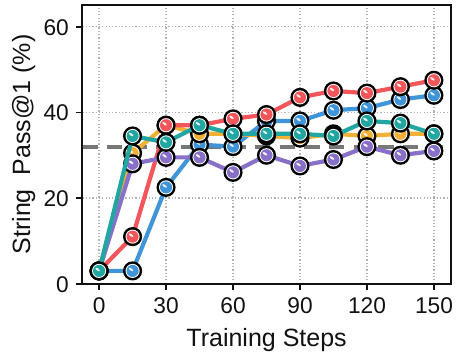}

\caption{String Manipulation}
\label{fig:weak_to_strong_delta_string}
\end{subfigure}

\caption{
\textbf{Training dynamics of weak-to-strong methods using the weak policy delta.}
The horizontal dashed lines indicate weak-teacher performance. See \appautoref{subsec:baseline_implementation} for baseline implementation details.
}
\label{fig:weak_to_strong_delta_comparison}

\endgroup
\end{figure*}

\clearpage
\section{Additional Results on Guidance-Direction Construction}
\label{app:delta_signal}

OPRD requires a guidance direction that captures the reward-relevant change acquired by the teacher. \secautoref{subsec:component_analysis} compares three constructions: the weak policy delta $\boldsymbol{\Delta}_t$ contrasts the post-trained teacher with its reference policy and isolates the change acquired during post-training; OPD contrasts the post-trained teacher with the current student, so its direction conflates the teacher's post-training update with the broader mismatch between the teacher's reference policy and the current student (i.e., $\mathbf{z}_T-\mathbf{z}_S=(\mathbf{z}_T-\mathbf{z}_T^{\mathrm{ref}})+(\mathbf{z}_T^{\mathrm{ref}}-\mathbf{z}_S)$); and OPSD derives its direction from the discrepancy induced by a privileged teacher draft. The comparison uses the step-60 checkpoint from the weak teacher's GRPO run. Under this setting, the weak policy delta outperforms both alternatives by a wide margin. 
\looseness=-1

However, OPD follows the gradient of a teacher-matching objective, its usefulness as a scaling direction should depend on teacher performance. We test this using the stronger teacher checkpoints adopted in our main experiments while keeping all other settings fixed (\autoref{fig:delta_signal_app}). For mathematics, we use the step-75 teacher, which is already relatively strong. For Knights \& Knaves, we use the step-105 teacher, which achieves 57.5\% Pass@1 compared with 29.0\% at step 60. With these teachers, the OPD direction performs well on both tasks, although it remains slightly behind the weak policy delta overall. OPSD is less consistent: it finishes above GRPO on AIME\textquotesingle24 but barely improves on Knights \& Knaves.
\looseness=-1


\begin{figure*}[!h]
\centering

\begingroup

\captionsetup[subfigure]{
  font=small,
  labelfont=normalfont,
  justification=centering,
  singlelinecheck=true,
  skip=7pt
}


\includegraphics[
    width=0.55\textwidth,
    keepaspectratio
]{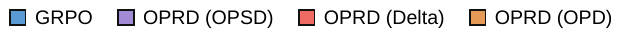}

\par\vspace{3pt}


\begin{subfigure}[t]{0.34\linewidth}
\centering

\includegraphics[
    width=\linewidth,
    keepaspectratio
]{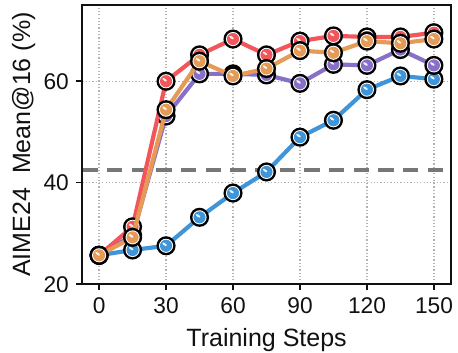}

\caption{Math}
\label{fig:delta_signal_math_app}
\end{subfigure}
\hspace{0.04\linewidth}
%
\begin{subfigure}[t]{0.34\linewidth}
\centering

\includegraphics[
    width=\linewidth,
    keepaspectratio
]{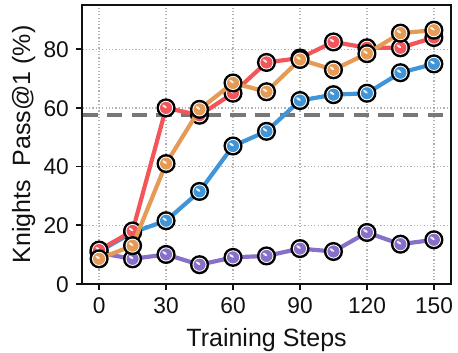}

\caption{Knights \& Knaves}
\label{fig:delta_signal_knights_knaves_app}
\end{subfigure}

\caption{
\textbf{Comparison of guidance-direction constructions.}
The variants construct $\mathbf d_t$ from the teacher--reference policy shift (Delta), the teacher--student mismatch (OPD), or the privileged-context discrepancy (OPSD). Using stronger teacher checkpoints than those used in \autoref{fig:analysis_delta_signal}, we pair the step-75 Qwen3-4B teacher with a Qwen3-8B student for Math and the step-105 Qwen3-4B-Base teacher with a Qwen3-8B-Base student for Knights \& Knaves. Gray dashed lines mark teacher performance. All other training configurations follow \autoref{app:exp_details}.
\looseness=-1
}
\label{fig:delta_signal_app}

\endgroup
\end{figure*}

These results suggest that the OPD gradient can provide a useful guidance direction when the weak teacher is sufficiently capable. In practice, however, the eventual performance gap between the weak teacher and the larger student cannot be known without fully training the student, making OPD difficult to adopt as a reliable default. OPSD is also less reliable because a privileged draft can constrain the student to a prescribed reasoning path \citep{kim2026does, kaur2026rethinking}. Using its self-distillation gradient as $\mathbf d_t$ can then amplify verifier-gradient components aligned with this restrictive signal and hinder learning.
The weak policy delta avoids both limitations because it compares the post-trained teacher only with its own reference policy. This isolates the change acquired during post-training without relying on either the evolving student or a privileged draft. We therefore retain the weak policy delta as our default guidance direction due to its stronger empirical performance and greater reliability in practice.
\looseness=-1

\clearpage
\section{Detailed Analysis of Student Behavior}
\label{app:style_qualitative}

\subsection{Token Alignment Analysis}
\label{app:token_visualization}

We analyze the correct AIME\textquotesingle25 response shown in \autoref{fig:token_correction} using the OPRD-trained Qwen3-8B student at update 150 and the Qwen3-4B teacher at GRPO update 75. We compute the policy gradient $\mathbf g_t$ assuming a single correct rollout with $A_t=1$, since the magnitude of a positive advantage does not affect cosine similarity.
Each token is colored by $\cos(\mathbf d_t,\mathbf g_t)$, with green indicating positive alignment and red indicating negative alignment. We compare the original continuation with an alternative generated by the same student checkpoint, keeping the selected prefix fixed and forcing the next token to be \shifttoken{5}, the top-ranked token under $\mathbf d_t$.
\looseness=-1

\subsection{Response Style Analysis}
\label{app:style}

We evaluate Qwen3-8B students trained with OPD and OPRD on DAPO-Math-17K at updates 30, 60, 90, 120, and 150. For each checkpoint, we generate 16 responses to each of the 30 AIME\textquotesingle24 problems, yielding 480 responses. We use two fixed references: the Qwen3-4B teacher at GRPO update 75 and a separately GRPO-trained Qwen3-8B student at update 150. Each response is represented by 101 style features across five categories: connectives (15), modality (10), grammar (64), punctuation (8), and sentence and paragraph structure (4). The first three categories measure relative frequencies of function words, including connectives, modal and negation words, and grammatical words such as pronouns and articles. Punctuation features count occurrences per 1,000 words, while structure features capture the mean and standard deviation of words per sentence and sentences per paragraph. We standardize the features at every checkpoint of each method using a shared mean and standard deviation for each feature, computed from the 960 responses of the two reference models.
\looseness=-1


\begin{figure*}[!h]
\centering

\begingroup

\captionsetup[subfigure]{
  font=small,
  labelfont=normalfont,
  justification=centering,
  singlelinecheck=true,
  skip=5pt
}


\begin{subfigure}[t]{0.29\linewidth}
\centering

\includegraphics[
    width=\linewidth,
    keepaspectratio
]{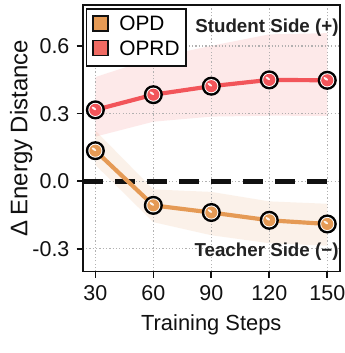}

\caption{Overall style similarity}
\label{fig:style_trajectory_aime24}
\end{subfigure}
\hfill
%
\begin{subfigure}[t]{0.685\linewidth}
\centering

\includegraphics[
    width=\linewidth,
    keepaspectratio
]{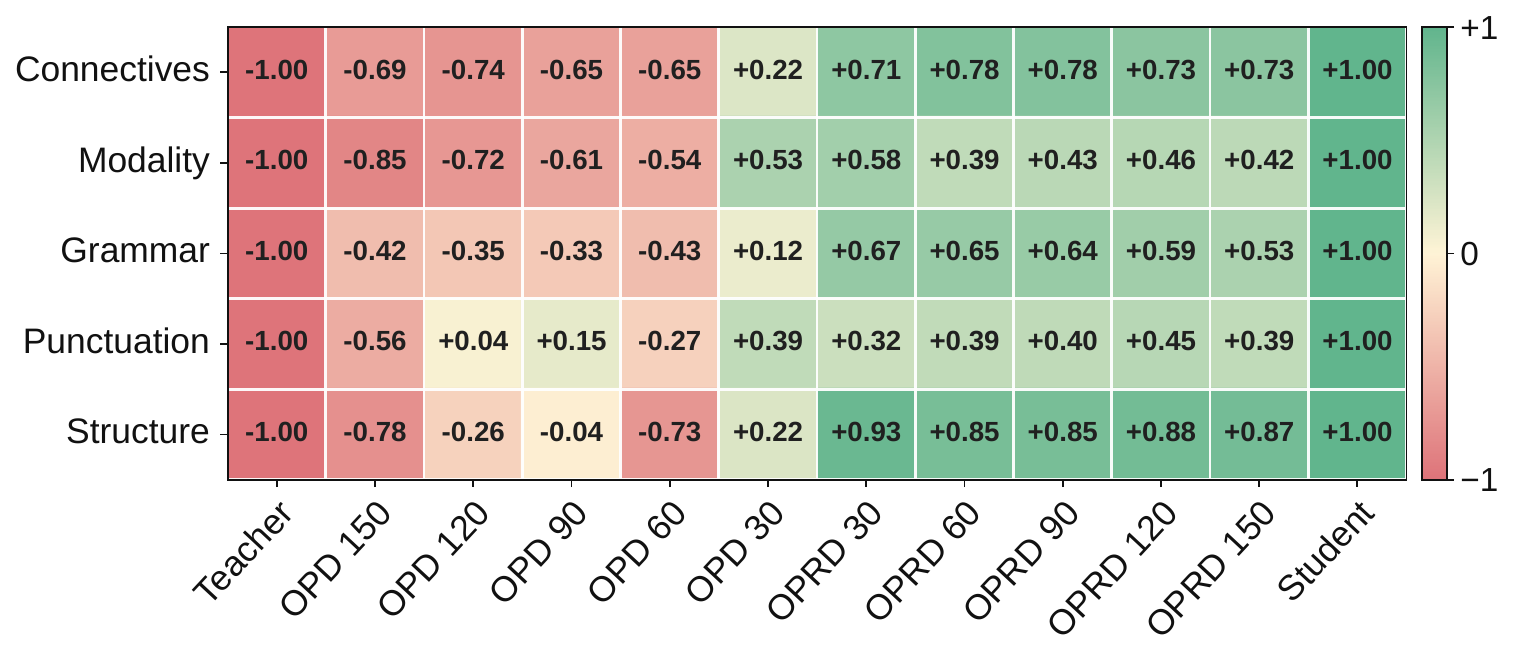}

\caption{Style similarity by category}
\label{fig:style_category_fingerprint_aime24}
\end{subfigure}

\par\vspace{2pt}

\caption{
\textbf{(a) Comparing response style distributions over training.}
Energy-distance differences are averaged across 30 AIME\textquotesingle24 problems, with positive values indicating greater similarity to the GRPO student and negative values to the teacher. Shading shows 95\% confidence intervals from 2,000 bootstrap resamples of the problems.
\textbf{(b) Measuring similarity to teacher and student response styles.}
Normalized distance differences compare average styles in five categories, with red indicating greater similarity to the teacher and green to the GRPO student.
Both panels use fixed references: the Qwen3-4B teacher at GRPO update 75 and the Qwen3-8B GRPO student at update 150.
\looseness=-1
}
\label{fig:style_qualitative_analysis_aime24}

\endgroup
\vspace{10pt}
\end{figure*}

\autoref{fig:style_trajectory_aime24} compares response style distributions using all 101 standardized features. For this comparison, we combine connectives, modality, and grammar into a group of 89 function-word features and scale the function-word, punctuation, and structure coordinates by $1/\sqrt{89}$, $1/\sqrt{8}$, and $1/\sqrt{4}$, respectively, to balance the three groups' contributions. For each problem, let $X_k$, $X_T$, and $X_S$ denote the sets of 16 response feature vectors from the evaluated checkpoint at update $k$, the teacher, and the GRPO student, respectively. We compute the energy-distance difference
\[
\Delta_{\mathrm{ED}}
=
\mathrm{ED}(X_k,X_T)-\mathrm{ED}(X_k,X_S).
\]
Energy distance measures differences between distributions by accounting for both between-set distances and within-set variation. We average $\Delta_{\mathrm{ED}}$ across the 30 problems, with positive values indicating greater similarity to the GRPO student and negative values to the teacher. Shading shows 95\% confidence intervals from 2,000 bootstrap resamples of the problems. OPD shifts toward the teacher over training, with its mean difference decreasing from $+0.135$ at update 30 to $-0.190$ at update 150. OPRD remains closer to the GRPO student at every evaluated checkpoint, with its mean difference increasing from $+0.316$ to $+0.447$ over the same period.
\looseness=-1

\autoref{fig:style_category_fingerprint_aime24} compares the same checkpoints separately across the five style categories. For each category $c$, we average the standardized features across all 480 responses without additional feature-group scaling. Let $\boldsymbol{\mu}_{k,c}$, $\boldsymbol{\mu}_{T,c}$, and $\boldsymbol{\mu}_{S,c}$ denote these average vectors for the evaluated checkpoint at update $k$, the teacher, and the GRPO student, respectively. We compute
\[
s_c(k)
=
\frac{
\|\boldsymbol{\mu}_{k,c}-\boldsymbol{\mu}_{T,c}\|_2
-
\|\boldsymbol{\mu}_{k,c}-\boldsymbol{\mu}_{S,c}\|_2
}{
\|\boldsymbol{\mu}_{T,c}-\boldsymbol{\mu}_{S,c}\|_2
}.
\]
The score measures the difference in distances to the two reference averages, normalized by their separation. Scores range from $-1$ to $+1$, with negative values indicating greater proximity to the teacher, positive values to the GRPO student, and zero indicating equal distance. OPD is closer to the GRPO student in all five categories at update 30 but closer to the teacher in all five by update 150. OPRD remains closer to the GRPO student in every category at all evaluated checkpoints, consistent with the overall distribution comparison.
\looseness=-1

\clearpage
\section{Additional Results on Length Bias in Teacher Policy Shift}
\label{app:length_bias}

The teacher policy shift $\boldsymbol{\Delta}_t$ (and hence the guidance direction $\mathbf d_t$) may contain reward-irrelevant components such as $\boldsymbol{\epsilon}_t$ alongside task-relevant progress. As discussed in \secautoref{subsec:challenges_discussion} and \appautoref{app:asymetric_alignment}, OPRD amplifies the projection of $\mathbf g_t$ onto $\mathbf d_t$, and this can also magnify reward-irrelevant components encoded in the guidance direction. Response length provides one observable example: when it correlates with verifier reward, both $\mathbf g_t$ and $\mathbf d_t$ may favor shorter responses even when shortening itself does not improve reasoning. Because $\mathbf d_t$ is derived from $\boldsymbol{\Delta}_t$, the choice of $\pi_T^{\rm ref}$ determines how much of the teacher's length change enters the guidance. A step-0 reference uses the base model and therefore includes the full post-training shift, whereas a later reference can exclude a sharp early length collapse.
\looseness=-1

Binary Matrix provides another instance of this behavior. The teacher's mean response length falls sharply between steps 30 and 45 and then stabilizes. We therefore compare step-0 and step-45 choices of $\pi_T^{\rm ref}$: the former includes a large shortening component in $\boldsymbol{\Delta}_t$, whereas the latter excludes most of it. \autoref{fig:length_bias_binary_matrix} compares both OPRD variants with GRPO, OPD, and KDRL. Both initially improve faster than GRPO but diverge after step 90. With the step-0 reference, the student's responses continue to shorten and Pass@1 plateaus at 81.5\%, below GRPO and KDRL, consistent with the correction overemphasizing length reduction. With the step-45 reference, response length does not exhibit the same continued decline and Pass@1 reaches 96.0\%. As on Color Cube, placing the reference after the sharp length transition mitigates this bias while retaining the teacher's later task progress. However, changing the reference modifies $\boldsymbol{\Delta}_t$ as a whole rather than isolating its length-related component. Disentangling structured bias from task-relevant guidance therefore remains an open question.
\looseness=-1


\begin{figure*}[!h]
\centering

\begingroup

\captionsetup[subfigure]{
  font=small,
  labelfont=normalfont,
  justification=centering,
  singlelinecheck=true,
  skip=7pt
}


\begin{subfigure}[t]{0.34\linewidth}
\centering

\includegraphics[
    width=\linewidth,
    keepaspectratio
]{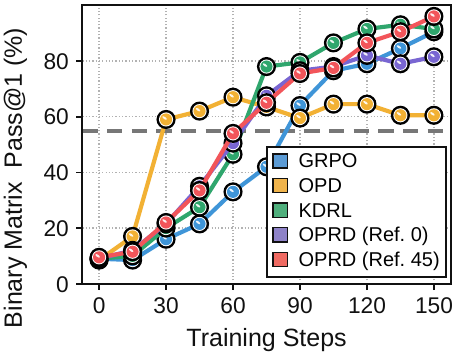}

\caption{Task performance}
\label{fig:length_bias_binary_matrix_performance}
\end{subfigure}
\hspace{0.04\linewidth}
%
\begin{subfigure}[t]{0.34\linewidth}
\centering

\includegraphics[
    width=\linewidth,
    keepaspectratio
]{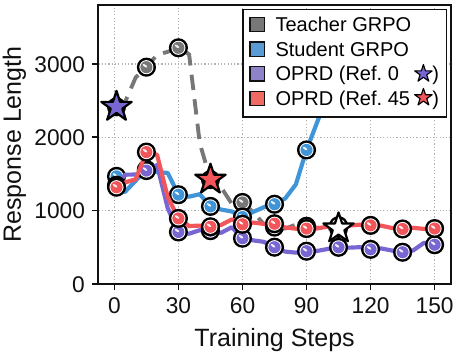}
\caption{Response Length}
\label{fig:length_bias_binary_matrix_response_length}
\end{subfigure}

\caption{
\textbf{Effect of reference-policy selection on length bias in Binary Matrix.}
We transfer a step-105 Qwen3-4B-Base teacher $\pi_T$ to a Qwen3-8B-Base student, using either the step-0 or step-45 checkpoint from the same GRPO run as $\pi_T^{\rm ref}$. OPRD's negative-branch scale $\lambda_t$ is warmed up over the first 75 steps. The gray dashed line marks teacher performance, while the colored stars denote the mean response lengths of the two choices of $\pi_T^{\rm ref}$, and the white star marks that of $\pi_T$. All other settings follow \appautoref{app:exp_details}.
\looseness=-1
}
\label{fig:length_bias_binary_matrix}

\endgroup
\end{figure*}

\clearpage
\section{Computational Cost and Memory Usage}
\label{app:cost_memory}

\paragraph{Benchmark Setup.} 
To isolate method-specific training overhead from response-length differences, we force every generated response to contain exactly 16{,}384 tokens by ignoring EOS. Both benchmarks use the mathematics setting of \secautoref{subsec:weak_to_strong}, with a Qwen3-8B student and the Qwen3-4B teacher checkpoint at update 75. OPRD additionally loads the Qwen3-4B base policy as its reference. Within each benchmark, all methods receive prompts in the same order and use the same random seed. Wall-clock timing uses 64 prompts, whereas memory profiling uses 8 prompts, with 8 rollouts per prompt in both cases. All runs execute on four NVIDIA B200 GPUs with DP4, rollout TP1, BF16, Flash Attention 2, padding removal, and gradient checkpointing. Optimization uses a global minibatch of 64 responses and one PPO epoch per training step. Dynamic token batching caps each GPU at 36{,}864 tokens, which yields two complete samples per microbatch under the fixed-length setting. The rollout engine uses \texttt{gpu\_memory\_utilization=0.60} and \texttt{max\_num\_seqs=128}. Sampling uses temperature 1.0, top-$p$ 1.0, and no top-$k$ truncation. To keep reward-side computation identical, we replace task-specific reward evaluation with deterministic alternating binary rewards within each prompt group. Validation, periodic model saving, external logging, and all non-training diagnostics are disabled. 
\looseness=-1

\vspace{-10pt}
\paragraph{Wall-Clock Time.} 
We run each method in a fresh process, discard one complete warm-up step, and report the mean and sample standard deviation over the following four steps. As shown in \autoref{tab:cost}, rollout generation is the largest component of each training step, taking roughly 602 seconds and accounting for 61.6\% of the total GRPO time, with small differences across methods attributable to run-to-run variation. OPD and KDRL, each of which evaluates one frozen teacher, incur total overheads of 7.2\% and 7.9\% over GRPO, respectively. OPRD evaluates the teacher and its reference sequentially, increasing frozen-forward time from 60.46 seconds for OPD to 108.80 seconds. Because rollout generation dominates the step, this additional reference evaluation increases total time by only 4.4\% over OPD, resulting in an overall overhead of 11.9\% relative to GRPO. Student-forward time is effectively unchanged, while the update containing the teacher-direction projection and scaling increases by just 2.79 seconds over GRPO, equivalent to 0.25\% of the full OPRD step. Beyond the teacher evaluation already required by OPD and KDRL, nearly all of OPRD's additional runtime therefore comes from evaluating the reference policy. 
\looseness=-1

\begin{table}[!ht]
\caption{
\textbf{Wall-clock time per training step.} 
All methods process the same 64 prompts with 8 rollouts per prompt, with every response fixed at 16{,}384 tokens to equalize the number of generated tokens. Total time is reported as the mean $\pm$ sample standard deviation, while the component columns report their means. Rollout includes the complete generation call and the actor-to-rollout mode transition. Student forward is a no-gradient pass that recomputes the old log probabilities of the sampled tokens. The frozen-model column reports no-gradient teacher evaluation for OPD and KDRL and sequential teacher and reference evaluations for OPRD, while GRPO requires neither. Update includes a separate gradient-enabled student forward pass, backward propagation, and the optimizer step, excluding the separately timed frozen-model evaluations. Etc.\ includes reward construction, advantage computation, batch assembly and balancing, orchestration, and residual boundary costs. 
\looseness=-1
}
\label{tab:cost}
\centering

\begingroup
\small
\setlength{\tabcolsep}{3.0pt}
\renewcommand{\arraystretch}{1.15}

\begin{tabularx}{\linewidth}{
@{}
l
>{\hsize=1.35\hsize\linewidth=\hsize\centering\arraybackslash}X
>{\hsize=0.90\hsize\linewidth=\hsize\centering\arraybackslash}X
@{}p{6pt}@{}
>{\hsize=0.95\hsize\linewidth=\hsize\centering\arraybackslash}X
@{}p{6pt}@{}
>{\hsize=0.85\hsize\linewidth=\hsize\centering\arraybackslash}X
>{\hsize=1.20\hsize\linewidth=\hsize\centering\arraybackslash}X
@{}p{6pt}@{}
>{\hsize=0.95\hsize\linewidth=\hsize\centering\arraybackslash}X
>{\hsize=0.80\hsize\linewidth=\hsize\centering\arraybackslash}X
@{}
}
\toprule

&
\multicolumn{2}{c}{\textbf{Total}}
& {} &
\multicolumn{1}{c}{\textbf{Rollout}}
& {} &
\multicolumn{2}{c}{\textbf{Model Forward}}
& {} &
\multicolumn{2}{c}{\textbf{Optimization}} \\
\cmidrule(lr){2-3}
\cmidrule(lr){5-5}
\cmidrule(lr){7-8}
\cmidrule(lr){10-11}

\textbf{Method}
& Time (s/step)
& Overhead
& {}
& Student
& {}
& Student
& Teacher $+$ Ref
& {}
& Update
& Etc. \\
\midrule

\rowcolor{blue!8}
&
\multicolumn{10}{c}{
\textbf{Qwen3-4B (Teacher)}
$\rightarrow$
\textbf{Qwen3-8B (Student)}
} \\
\midrule

GRPO
& $\,\,\,977.72 \pm 16.83$
& --
& {}
& $602.20$
& {}
& $70.49$
& $\,\,\,\,\,\,0.00$
& {}
& $304.00$
& $1.03$ \\

OPD
& $1048.00 \pm 17.49$
& $\,\,\,+7.2\%$
& {}
& $610.74$
& {}
& $70.38$
& $\,\,\,60.46$
& {}
& $305.38$
& $1.04$ \\

KDRL
& $1055.41 \pm 15.94$
& $\,\,\,+7.9\%$
& {}
& $618.20$
& {}
& $70.08$
& $\,\,\,60.17$
& {}
& $305.92$
& $1.05$ \\

\rowcolor{gray!15}
OPRD
& $1094.44 \pm 16.17$
& $+11.9\%$
& {}
& $607.39$
& {}
& $70.40$
& $108.80$
& {}
& $306.79$
& $1.06$ \\

\bottomrule
\end{tabularx}

\endgroup
\vspace{10pt}
\end{table}

In our primary weak-to-strong setting, the teacher is smaller than the student, keeping the cost of maintaining and evaluating its reference policy relatively low. In strong-to-weak settings, however, retaining both a large teacher and its equally sized reference can impose substantial memory and computational costs. Since the teacher is already more capable than the student, one potential alternative is to use the OPD gradient directly as the guidance direction, as explored in \autoref{fig:delta_signal_app}, eliminating the need for a separate reference model. Another option is to replace the teacher's reference with the student's smaller frozen reference policy. Although these alternatives no longer isolate the teacher's RL-induced policy shift or retain its corresponding implicit-reward interpretation, they could still exploit the strong teacher's token preferences to guide verifier-driven student optimization. We leave their evaluation to future work.
\looseness=-1

\clearpage
\paragraph{Peak GPU Memory.} 
We measure peak GPU memory during rollout generation and the actor update, while separately recording the frozen-policy evaluation performed within the update. We also report the overall maximum observed during the complete training step. 
As shown in \autoref{tab:memory_cost}, OPRD carries a nearly constant additional footprint throughout training: approximately 14\,GiB per GPU relative to GRPO and 13\,GiB relative to OPD and KDRL. 
The two largest identifiable memory requirements within OPRD are the 3.75\,GiB frozen teacher and reference parameter shards and a transient 9.43\,GiB dense corrected-gradient allocation within the correction hook. Of the 3.75\,GiB in frozen-model parameters, 1.87\,GiB is additional relative to OPD and KDRL, which already retain the teacher. The 9.43\,GiB hook allocation is also specific to OPRD. Because the absolute NVML peaks additionally include shared model and optimization state, allocator caches, and CUDA and distributed runtime state, these quantities identify the main OPRD-specific allocations but do not provide an exact additive decomposition of the observed peak difference.
\looseness=-1

The overall maximum occurs during rollout for every method. OPRD reaches 151.27\,GiB, exceeding GRPO by 13.98\,GiB (10.2\%) and OPD and KDRL by 13.21\,GiB. Rollout itself increases memory by approximately 51--52\,GiB for all four methods. The difference is already present before generation, where OPRD begins the measured step at 99.97\,GiB, 14.93\,GiB above GRPO and 13.25\,GiB above OPD and KDRL. OPRD's higher rollout peak therefore results from adding essentially the same generation-time allocation to a higher starting footprint, rather than from rollout requiring more memory.
\looseness=-1

Frozen-policy evaluation is performed within the broader actor-update interval, and their maximum values coincide in our measurements. OPRD reaches 115.19\,GiB during both frozen-policy evaluation and the full update, exceeding OPD and KDRL by 13.15\,GiB. During the update, it also exceeds GRPO by 14.84\,GiB. These differences closely match those observed before and during rollout, indicating that neither frozen-policy evaluation nor the correction introduces a separate phase-specific increase in the device-memory peak. Within the correction hook, PyTorch-allocated memory grows by 9.43\,GiB, matching the largest dense BF16 corrected-gradient tensor. By comparison, the sparse support formed by the sampled action and the student's top-10 tokens occupies at most 1.38\,MiB, and direct gather and sparse scatter avoid an additional 9.27\,GiB response-by-vocabulary copy. The hook allocation is already contained within the 115.19\,GiB update peak, which remains well below the overall maximum during rollout.
\looseness=-1

\begin{table}[!ht]
\caption{
\textbf{Peak GPU memory per training step.} 
We profile 8 prompts with 8 rollouts per prompt and fix every response at 16{,}384 tokens. Each method is evaluated in four independent trials, each launched in a fresh process on four NVIDIA B200 GPUs with DP4/TP1 and BF16. Each trial discards one complete warm-up step and measures the following step. Whole-device NVML memory is sampled every 100\,ms, and each Peak entry reports the mean across trials of the maximum usage over time and across the four GPUs within the indicated interval. All values are in GiB per GPU. Because each Peak entry represents the worst-GPU peak, multiplying it by four provides only a rough upper bound on aggregate device memory.
Overall is the maximum over the complete step. Under Rollout, Peak $-$ Start is the increase from the step-start baseline to the rollout peak. Teacher $+$ Ref.\ Peak reports the maximum during frozen-model evaluation, covering one teacher evaluation for OPD and KDRL and sequential sequential teacher and reference evaluations for OPRD. This evaluation is a subinterval of Actor Update. Params.\ reports the calculated lower bound for the total BF16 parameter shards of the resident frozen models. Actor Update Peak includes resident models, gradients, optimizer states, activations, and runtime buffers, while Hook Growth reports the additional PyTorch allocation during the OPRD correction.
\looseness=-1
}
\label{tab:memory_cost}
\centering

\begingroup
\small
\setlength{\tabcolsep}{2.5pt}
\renewcommand{\arraystretch}{1.15}

\begin{tabularx}{\linewidth}{
@{}
l
>{\hsize=0.95\hsize\linewidth=\hsize\centering\arraybackslash}X
>{\hsize=1.05\hsize\linewidth=\hsize\centering\arraybackslash}X
@{\hspace{5pt}}
>{\hsize=0.95\hsize\linewidth=\hsize\centering\arraybackslash}X
>{\hsize=0.90\hsize\linewidth=\hsize\centering\arraybackslash}X
@{\hspace{5pt}}
>{\hsize=0.95\hsize\linewidth=\hsize\centering\arraybackslash}X
>{\hsize=0.95\hsize\linewidth=\hsize\centering\arraybackslash}X
@{\hspace{5pt}}
>{\hsize=0.95\hsize\linewidth=\hsize\centering\arraybackslash}X
>{\hsize=1.30\hsize\linewidth=\hsize\centering\arraybackslash}X
@{}
}
\toprule

&
\multicolumn{2}{c}{\textbf{Overall}}
&
\multicolumn{2}{c}{\textbf{Rollout}}
&
\multicolumn{2}{c}{\textbf{Teacher $+$ Ref}}
&
\multicolumn{2}{c}{\textbf{Actor Update}} \\
\cmidrule(lr){2-3}
\cmidrule(lr){4-5}
\cmidrule(lr){6-7}
\cmidrule(lr){8-9}

\textbf{Method}
& \mbox{Peak}
& \mbox{Overhead}
& \mbox{Peak}
& \mbox{Peak $-$ Start}
& \mbox{Peak}
& \mbox{Params}
& \mbox{Peak}
& \mbox{Hook Growth} \\
\midrule

\rowcolor{blue!8}
&
\multicolumn{8}{c}{
\textbf{Qwen3-4B (Teacher)}
$\rightarrow$
\textbf{Qwen3-8B (Student)}
} \\
\midrule

GRPO
& $137.29$
& --
& $137.29$
& $52.25$
& --
& --
& $100.35$
& -- \\

OPD
& $138.06$
& $\,\,\,+0.6\%$
& $138.06$
& $51.34$
& $102.04$
& $1.87$
& $102.04$
& -- \\

KDRL
& $138.06$
& $\,\,\,+0.6\%$
& $138.06$
& $51.34$
& $102.04$
& $1.87$
& $102.04$
& -- \\

\rowcolor{gray!15}
OPRD
& $151.27$
& $+10.2\%$
& $151.27$
& $51.30$
& $115.19$
& $3.75$
& $115.19$
& $9.43$ \\

\bottomrule
\end{tabularx}

\endgroup
\end{table}
}

\end{document}